\documentclass[11pt]{article}
\pdfoutput=1
\usepackage[margin=1in]{geometry}
\usepackage{amsmath,amssymb,amsthm}
\usepackage{booktabs}
\usepackage{graphicx}
\usepackage[hidelinks]{hyperref}
\usepackage{xcolor}


\newtheorem{proposition}{Proposition}
\newtheorem{definition}{Definition}

\title{In-Cell Learning: Language Models That Update Their Own\\
Weights in Sequence Without Changing the File They Ship}
\author{
  Zifeng Liu$^{1,2,3,4}$ \quad Yaxin Lu$^{1}$ \quad Xuanhan Wu$^{5}$ \quad
  Zhiyong Du$^{6}$ \\[0.35em]
  Yiming Mao$^{7}$ \quad Zhenhe Wang$^{8}$ \quad Wenqi Shi$^{1}$ \quad
  Zhengkun Jing$^{9}$ \quad Linwei Liu$^{10}$
  \\[0.8em]
  {\footnotesize $^{1}$Big Data and Artificial Intelligence Center,
   The Third Affiliated Hospital of Sun Yat-sen University}\\
  {\footnotesize $^{2}$Institute for Frontier Interdisciplinary Research in Health
   Sciences and Technology, Sun Yat-sen University}\\
  {\footnotesize $^{3}$Sun Yat-sen University Institute of Artificial Intelligence
   \qquad
   $^{6}$School of Business, Sun Yat-sen University}\\
  {\footnotesize $^{4}$Guangdong Engineering Research Center of Medical Artificial
   Intelligence Multimodal System}\\
  {\footnotesize $^{5}$Paul Merage School of Business, University of California,
   Irvine, Irvine, CA, USA}\\
  {\footnotesize $^{7}$School of Computer Science,
   China University of Geosciences (Wuhan)}\\
  {\footnotesize $^{8}$School of Public Health, Sun Yat-sen University
   \qquad
   $^{9}$Hospital of Stomatology,
   Sun Yat-sen University}\\
  {\footnotesize $^{10}$School of Pharmacy,
   Guangdong Pharmaceutical University}\\[0.3em]
  {\normalsize Correspondence: \texttt{liuzf@mail.sysu.edu.cn}}
}
\date{Technical Report, version 3 --- \today}

\begin{document}
\maketitle

\begin{abstract}
A 4-bit weight is a cell, not a point. A released language model is frozen
twice over. Teaching it something new
risks what it already knew, and every way of writing knowledge into its
weights---full fine-tuning, adapter merging, closed-form editing---returns a
different file from the one a benchmark report, a certification and a fleet
of devices refer to by its exact bits. We show that the second freeze is
avoidable, and that lifting it turns a static release into a model that keeps
learning after it ships.
A 4-bit release stores each weight as an integer code and a shared scale, and
leaves between the stored values an interval the quantizer discards. We call
\emph{in-cell learning} the paradigm in which new knowledge is written only
into that interval, so that re-quantizing the served weights reproduces the
released codes and scales exactly; the guarantee is checked in the integer
domain on every weight, the update is a separate file that can be withdrawn
by subtraction, and drift is bounded by radii the grid fixes before training.
\textbf{CellFill} realizes the paradigm by training a bounded low-rank
position inside each cell of the released file. On published NF4 and W4A16
releases of Qwen3 and Gemma it wrote $83$--$99\%$ of a corpus of real facts,
verified absent on the six releases where the prior was measured (the three
largest were not), and re-quantization returned the stored code on every
constrained weight. On our own NF4 quantization of a 27B hybrid model the
constrained set was $2.4\times10^{10}$ weights, the largest we ran.
What arrives behaves as knowledge rather than recitation: facts written in
separate sentences chained at inference on a composition the corpus never
states, and on PopQA's long tail the fill answered $78$--$88\%$ of the
questions the released model could not, from the weights alone.
Around that write we assemble the cycle a deployed model would run---three
of its five stages wired into the sequence, the saturation scan and the task
suite measured outside it.
\emph{Saturation} is predictable in advance: injecting eight Python libraries
one at a time, the gain was flat at $+10.7$ points wherever the released
model's prior accuracy fell below $28\%$ and collapsed to $+2.7$ on the one
library it already knew at $31.1\%$---with the cells barely touched and the
training loss falling normally, so a successful optimization is no evidence
of a successful injection. \emph{Rehearsal} recovers what a sequence forgets:
replaying a sample of the earlier tasks' own material alongside each new one
carried the first task's retention from $19.3\%$ to $84.8\%$ after five later
updates---the single change with the largest effect we measured. Against the
absolute ceiling a deployed system cannot reach, training on all $3{,}000$
facts at once and recalling $61.2\%$ of them, the rehearsed sequence recovers
about half ($31.0\%$ over the same facts) where sequential updating alone
recovers a fifth. What makes an unbroken run finite is the geometry: each update consumes a
constant fraction of the room the last one left ($0.78$--$0.85$ per task
across four matched sequences at 1.7B; $0.64$--$0.72$ and falling within
both 8B sequences), so the room decays geometrically and lifetime
capacity is bounded at $231$--$339\%$ of a first task over that band, and is
$292\%$ at the sequence whose $\beta$ we quote. The bound is anchor-dependent
and we report the exception rather than smooth it: on a published release the
same recipe absorbs \emph{more} per unit of room as the sequence deepens. \emph{Consolidation}
is what lifts that bound---the operator re-quantizes, which restores the
room and changes between $1.1$ and $34.6\%$ of the 4-bit codes across the
recipes we ran---$21.4$ then $26.1\%$ at the two consolidations of a single
six-task cycle at 1.7B, $33.4$ then $34.6\%$ at 8B, and under $2.3\%$
throughout when the link is scaled to $s{=}10$ and the fill barely leaves the
anchor. Which parameterization writes matters as much as how long: under the same
cells and the same check, a bounded dense fill absorbs $2.4$--$3.5\times$
more than rank 64 at matched capability cost, and the scale drift a renewal
inherits obeys a measured law---growth affine in how far the fill moved
the weights ($7.33\,\mathbb{E}|t|{+}0.63$, $r{=}0.991$ across nine arms),
its zero point pinned by the projection identity (an empty consolidation is
a bitwise fixed point) and its intercept the law's bend near that origin. Within a cycle the price of a major version grows with the depth of
the sequence it interrupts.
Run as one loop---six tasks, rehearsing throughout and consolidating every
second task---the cycle leaves no task below $84.9\%$ of what it first learned
and returns the shipped code on all $1.4\times10^{9}$ constrained weights at
every one of the six folds, where an unconsolidated sequence of the same
length put one weight outside its cell. A matched arm with the replay of
earlier tasks switched off separates the contributions: the identity survives
without replay, that arm being clean at every fold too, while the knowledge
does not---summed recall falls from $2.117$ to $0.878$ without it.
The same loop at 8B, on unsloth's published NF4 release, now carries two
seeds: they leave no task below $94.0$ and $92.8\%$ of what it first
learned, five of the ten earlier-task ends land above $100\%$ (rehearsal
having returned more than the fill lost), and both seeds return the
vendor's code on all $6.9\times10^{9}$ constrained weights at every one
of the six folds.
Retention is not the whole picture and we give both halves: what a task
learns when it is new falls from $0.761$ to $0.376$ across the first seed
and $0.818$ to $0.447$ across the second, so the artifact does not forget
and does lose plasticity, and in both seeds summed recall ends at or
above where it started ($3.539$ against $3.483$; $3.750$ against
$3.707$).
It also isolates what a major version costs the model that was released. A
consolidation re-quantizes the anchors and zeroes the fill, so the artifact
at that moment is the next 4-bit file with nothing written into it: against
the version it replaces it moves $-1.3$, $-2.0$, $-1.5$ and $+1.3$ points on
ARC-e, ARC-c, HellaSwag and WinoGrande---each within one standard error, one
of them upward. Any single such event is inside the noise; five of them, from
three runs at two scales, are not. Pooled, a major version costs
$0.80\pm0.24$ points of mean suite accuracy and every one of the five is
negative, and the cost does not scale with how much of the file changed:
$1.2\%$ of codes changed costs as much as $34.6\%$. What the sequence does cost the pretrained model is paid by
writing, not by re-quantizing, and it is concentrated: over the whole cycle
three of those four suites end $3.0$ to $4.7$ points down and HellaSwag ends
$13.7$ down, $8.2$ of which are gone after the first task alone.
What this buys is therefore a release discipline rather than a permanent
file. Minor versions reproduce the shipped bits exactly and are revocable by
subtraction; a consolidation is a major version, and it is declared,
measured and verifiable rather than silent. Forgetting is priced here, not
eliminated, and what the contract buys is a long but bounded run of updates
on a file that does not change.
The unit a deployed model learns in, on this account, is not the checkpoint
but the cell: room enough to write verified new knowledge today, and---in
prospect---to change how a served model behaves and to
let it keep improving itself while its file stays the file.
A platform that takes documents in and returns a verified fill, the archived
result files behind every table, and the scripts and \TeX\ sources that
generate every table and figure are at
\url{https://github.com/sumsliu/in-cell-learning}.
\end{abstract}

\section{Introduction}\label{sec:introduction}
\begin{figure}[!t]\centering
\includegraphics[width=\linewidth]{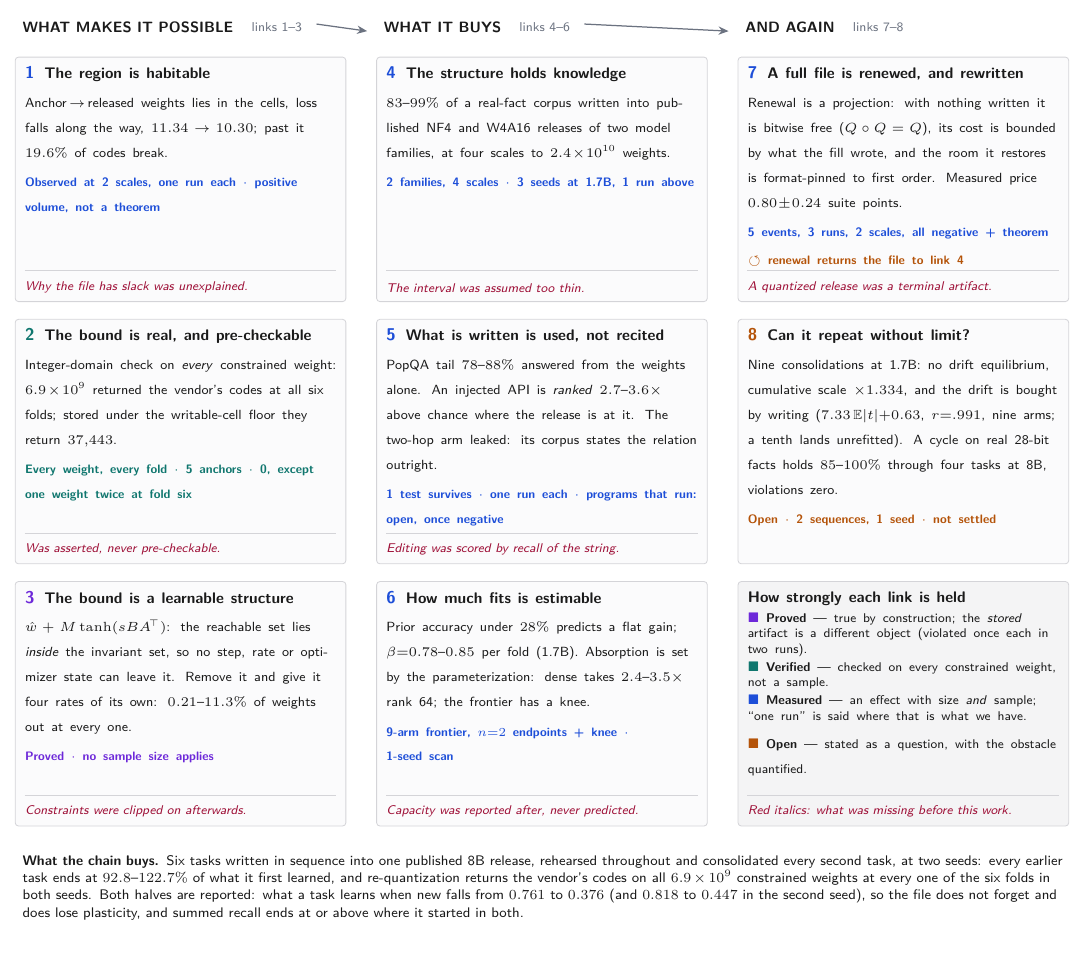}
\caption{\textbf{The argument, and what each link fills in.} The claim that a
released file can keep learning is a chain of eight, and every link is
answerable on its own; red italics name what was missing before this work.
One link is a proof---the reachable set of the parameterization is contained
in the invariant set, so training cannot leave it---and six are measurements. The
eighth is open, and is drawn as open. Each box points at the section that
carries it: \S\ref{sec:geometry}, \S\ref{sec:theory} with
\S\ref{sec:floor}, \S\ref{sec:methods}, \S\ref{sec:results},
\S\ref{sec:usable}, \S\ref{sec:knowable}, \S\ref{sec:loop8b}, and
\S\ref{sec:limitations}. The figures that follow each take one link and show
its mechanism or its measurement: Fig.~\ref{fig:overview} the write itself,
Fig.~\ref{fig:geometry} the habitable region of link one,
Fig.~\ref{fig:frontier} the learning--forgetting frontier of links four and
five, and Fig.~\ref{fig:lifecycle} the loop of link seven.}
\label{fig:argument}
\end{figure}

\begin{figure}[!t]\centering
\includegraphics[width=\linewidth]{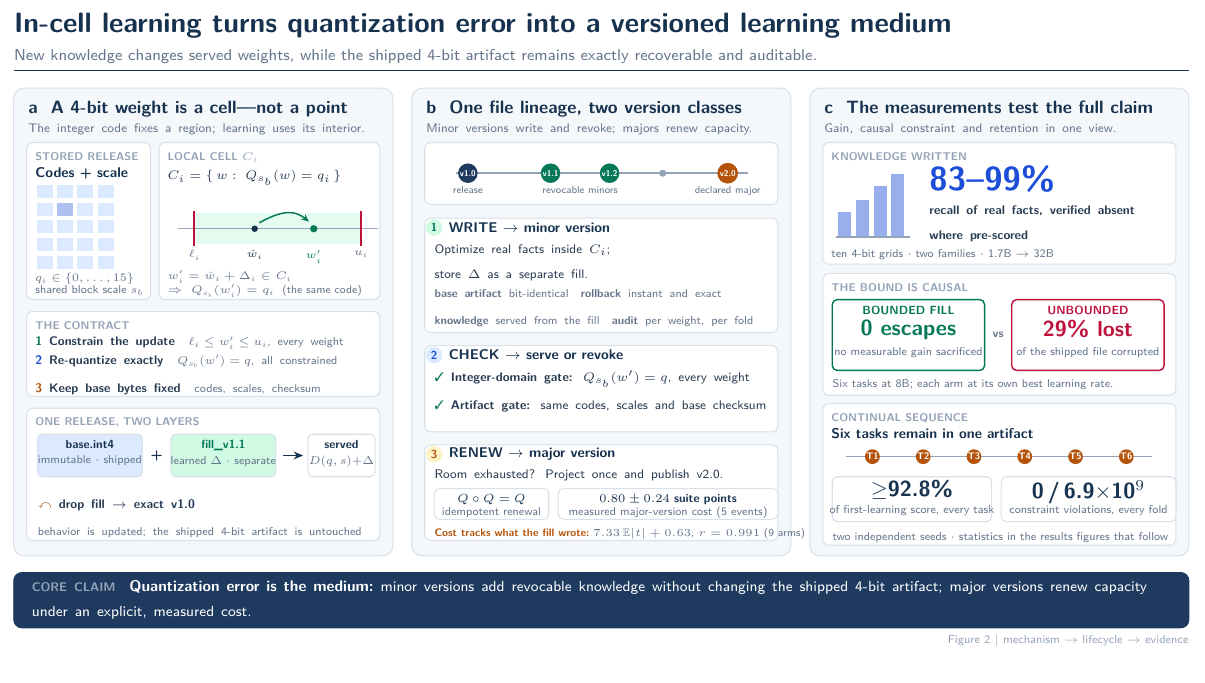}
\caption{\textbf{In-cell learning: the file, the lifecycle, the evidence.}
\textbf{a}, A 4-bit release stores an integer code and a block scale per
weight; between the values the grid can represent lies the interval the
quantizer discards. The in-cell fill moves each served weight only inside
that interval, so re-quantization returns the shipped bytes exactly
(\S\ref{sec:theory}), and the update ships as a separate file that
subtraction revokes.
\textbf{b}, The lifecycle this enables: minor versions write real facts
into the room and are checked in the integer domain on every constrained
weight; a major version renews the room through one projection
($Q\circ Q=Q$, \S\ref{sec:driftlaw}), priced at $0.80\pm0.24$ suite
points over five measured events, with scale drift set by what the fill
wrote ($7.33\,\mathbb{E}|t|{+}0.63$, $r=0.991$, nine arms).
\textbf{c}, The headline measurements: real facts, verified absent where
pre-scored, enter ten 4-bit grids of two families at $83$--$99\%$ recall
(\S\ref{sec:scale}); removing the bound corrupts $29\%$ of the file for
no measurable functional gain (\S\ref{sec:loop8b}); the full cycle holds
every task at or above $92.8\%$ of first learning across two seeds at
zero violations on all $6.9\times10^{9}$ constrained weights
(\S\ref{sec:loop8b}).}
\label{fig:overview}
\end{figure}

The binding cost of continual learning in deployment is not GPU time; it is
loss of certification. A 4-bit checkpoint is released once and depended upon
many times. Between release and retirement it accumulates a benchmark report, a
red-team review, integration tests, downstream caches, and increasingly an
external sign-off, and every one of those artifacts refers to a specific set of
integer codes and group scales. When the operator later needs the model to know
something it did not know at release---a corrected specification, a new API, a
changed policy---every available mechanism (full fine-tuning, adapter merging,
closed-form editing) returns a \emph{different} checkpoint, and the whole
apparatus above must be rebuilt against it.

We show that a deployed model can learn while the released artifact changes
not at all. Every merge in this paper re-quantizes to the shipped codes---up
to $2.4\times10^{10}$ constrained weights at once---with two exceptions we put
on record because they show what breaks the identity and what repairs it: one
weight of $1.4\times10^{9}$ left its cell at the sixth consecutive fold in
each of the two $r{=}64$, $s{=}40$ arms of a six-task sequence driven deep
into the bounded link's saturation (\S\ref{sec:rehearse-old}), and the same
six tasks run as a cycle that consolidates on a schedule return the shipped code at every fold
(\S\ref{sec:loop}). Quantization error is normally treated as damage to be
minimized. A 4-bit weight is a cell, not a point; we instead treat the
interval between a weight's dequantized
anchor and the boundaries of its rounding cell---the \emph{dequantization
gap}---as addressable storage (Fig.~\ref{fig:overview}). That reframing is
the engine of everything below: a cell is capacity, and the parameterization
that writes into it optimizes whatever loss it is given---what we write here
is verified factual knowledge, and nothing in the geometry is specific to
facts, so the same interior is in principle room for a changed behavior or a
self-directed update. With the codes
and scales of the release
frozen, the update is written only into a full-precision residual confined
to the interior of each cell. Re-quantization under those frozen scales then
returns the released codes bit for bit at any point in the model's life
(Prop.~\ref{prop:inv}); the guarantee is checked by integer comparison
rather than argued, and it must be stated against a frozen $(a,s)$ rather
than against re-running a quantizer, since absmax scales are data-dependent
and a single in-cell move can silently reassign an untouched neighbor
(\S\ref{app:frozen}). Dropping the residual restores the release exactly,
and the excursion is bounded per weight by radii the grid supplies for free
(Prop.~\ref{prop:budget}).

What this buys is narrower than it may first appear, and we state it
precisely. The served model is $\hat W+\delta$, and it is a different
function from $\hat W$: its cross-domain perplexity moves with every unit of
knowledge absorbed, so we do \emph{not} claim that the release's validation
transfers to the updated model. What is inherited is the \emph{reference}.
The certified configuration remains recoverable bit-for-bit by truncation at
any time, the update is exactly revocable, and its drift is bounded by a
budget whose radii are known before training begins. An update becomes an
auditable increment against a fixed baseline rather than an opaque
replacement of it.

\textbf{Empirically}, cross-domain perplexity rises with absorption along the
projected paths, but CellFill at rank 64 sits off that curve rather than
further along it, absorbing more than projected dense for substantially less
cost (Table~\ref{tab:frontier}). The bounded reparameterization moves the
trade-off rather than traversing it. The cost of knowledge is nonetheless
real: the in-domain ``free lunch'' this method appears to offer is an
artifact of rehearsal sharing a corpus with the metric
(\S\ref{sec:rehearsal}), and every headline number here is quoted against
LAMBADA, which the rehearsal never touches. Three findings then cut against
expectation.
First, the constraint is not only a tax: against the natural null hypothesis
of serving the same adapter unmerged---which preserves the artifact
trivially, by never touching it---projecting the update into the cells
\emph{reduces} cross-domain forgetting in all sixteen runs that recorded
both stages and converged, by $0.97$ to $456.7$ perplexity points, with the
effect largest exactly where drift is worst---and we report the boundary of that
effect too: on an update whose own optimization has already blown up,
projection amplifies the damage instead of repairing it
(\S\ref{sec:regularizer}). Second, the mechanism survives scale and
architecture: on a 27B hybrid model whose blocks are gated linear attention
rather than softmax attention, all $2.435\times10^{10}$ constrained weights
across 496 matrices verify bit-identical after merging, and matched recall
($25.4\%$ at 27B, single seed, against $24.0\!\pm\!1.7\%$ for the 1.7B
clip-merge arm) costs $+1.29$ LAMBADA points over its own anchor where the
1.7B arm pays $+3.07$ over its own, that figure being a two-seed mean as in
Table~\ref{tab:frontier} (\S\ref{sec:scale}). Third, knowledge is most
cheaply written into the gate and up projections at \emph{all} depths
($130$ absorbed bits per million trainable parameters, and $911$ bits per
point of cross-domain perplexity given up), rather than into
\texttt{down\_proj} ($72$); the early and middle MLP partitions absorbed
nothing distinguishable from the guessing floor at all, so we report them as
non-absorbing rather than as a low yield---contra the localization prescriptions of
ROME and MEMIT~\cite{rome,memit}, a discrepancy we read as a difference
between \emph{editing} an existing association and \emph{writing} a new one
under a norm constraint (\S\ref{sec:where}). Absorption grows with exponent
$\approx\!0.8$ from $10^3$ to $10^4$ presented facts at roughly $1\%$
cell-space utilization: in this range the regime is optimization-limited,
not capacity-limited (\S\ref{sec:results}).

\textbf{The argument in eight steps.} The claim that a released file can keep
learning is not one assertion but a chain, and each link is answerable on its
own.

\emph{One, the region is habitable.} The set of weights that keep the model's
capability is not a surface of measure zero around the released point; at this
scale it has interior, and the interval the quantizer discards sits inside it.
We establish this the way an existence claim about interior is
established---by exhibiting a full-dimensional box and measuring the loss over
it. The box is $\mathcal C=\prod_i C_i$, and every cell has positive width
(the degenerate blocks of \S\ref{sec:floor} are frozen, not zero-width). Along
the segment from the 4-bit anchor to the released checkpoint's own weights,
which lies inside the box ($0\%$ of weights leave their cell at half way, $0.61\%$ at the
end), perplexity falls monotonically the whole way, $11.34\to10.30$ in-domain
and $28.54\to26.02$ on LAMBADA; past the original weights both reverse and
$19.6\%$ of codes break. Under uniform random fills the cost at a quarter
amplitude is $11.34\to11.39$. The box is not uniformly benign, and we say so
because it matters later: a diverged update projected into it stays
diverged---every weight inside its cell and perplexity at $2.1\times10^{6}$
(\S\ref{sec:regularizer}). What the measurements establish is therefore positive
volume of low-loss configurations around the anchor, at sampled points rather
than as a theorem, and that the safe region and the useful region coincide
where the anchor sits (\S\ref{sec:geometry}).

\emph{Two, the bound is real, and its precondition is computable before
training.} Confinement to that box is checkable in the integer domain on
every weight rather than sampled: $6.9\times10^{9}$ weights returned their
vendor's codes at every one of six folds at 8B (\S\ref{sec:theory},
\S\ref{sec:loop8b}). It holds under a condition we had to discover---the
anchor's storage precision must resolve the margin---and the condition is not
a caveat but a check: it is a floor on the block scale, $3.70\times10^{-5}$
for fp16, that a script evaluates on the released file before a single step
is taken (\S\ref{sec:floor}).

\emph{Three, the bound is a learnable structure, not a constraint applied
afterwards.} This is the step we prove rather than measure, and it is what
makes the rest mechanical. CellFill does not train freely and then clip; it
reparameterizes the weight as $\hat w_i + M_i\tanh(s\,(BA^{\!\top})_i)$, so
the reachable set is \emph{contained in} the invariant set and no sequence of
gradient steps, no learning rate and no optimizer state can leave it---%
invariance holds at every step of training, not only at the end
(Prop.~\ref{prop:inv} with the storable-floor mask of \S\ref{sec:floor},
\S\ref{sec:methods}). Containment, not equality: the reachable set is
$\prod_i(\hat w_i-M_i,\,\hat w_i+M_i)$ and the invariant set is
$\prod_i[\ell_i,u_i]$, which coincide only where the anchor is centred in its
cell---two of NF4's sixteen codes, and those only by the capping convention.
Containment is all any argument here uses. Clipping is the alternative and it is worse for the reason step one
exhibited: the box contains diverged configurations, and projection is only
the Euclidean map onto it (Prop.~\ref{prop:proj}), so it returns the nearest
in-box point to an update that has already left---which on our runs is an
in-box point at perplexity $2.1\times10^{6}$. The model that trains and the
model that ships are then two different models. Removing the structure and changing nothing
else is not a small perturbation: at the same rank, and at four learning rates
of its own, the unconstrained parameterization never leaves the released codes
recoverable---from $0.21\%$ of weights outside their cells where it has
learned almost nothing to $11.26\%$ where its perplexity has begun to go
(\S\ref{sec:loop8b}).

\emph{Four, the structure can be written into.} Confinement would be
worthless if the interval were too thin to hold anything; on published
releases it holds $83$--$99\%$ of a corpus of real facts
(\S\ref{sec:results}).

\emph{Five, what is written becomes knowledge the model can use, not a string
it can recite.} Three tests that a memorized surface form would fail. Facts
written in separate sentences chain at inference on a two-hop composition the
corpus never states, at roughly twice the model's rate on its own pretrained
facts of the same form (\S\ref{sec:usable}). On PopQA's long tail the fill
answers $78$--$88\%$ of the questions the released model could not, from the
weights alone and with no retrieval. And an injected library API is used
rather than quoted in the weaker of the two senses we can measure: a served
model ranks the true parameter name at $2.7$ to $3.6\times$ the chance floor
where the released model is at or below it, while the instruction-following
usage rule stays at the floor for every served model below 27B, so whether
ranked use converts into programs that run is open and is listed as open
(\S\ref{sec:usable}).

\emph{Six, how much can be written is estimable---in the weak sense, and we
say which.} Two estimators survive contact with the archive: the released
model's own prior accuracy predicts saturation, the gain being flat wherever
that prior is below $28\%$ and collapsing above it, and the room decays by a
measured constant fraction per fold on self-quantized anchors, which bounds
lifetime capacity. Two do not, and we report them as failures rather than
smoothing them: a diagonal-Fisher budget is off by up to $383\times$ and a
data-processing bound is loose by five orders of magnitude. A third
estimator emerged from the parameterization frontier: absorption is set by
the fill's parameterization before it is set by time, a bounded dense fill
taking $2.4$--$3.5\times$ more than rank 64 under the same cells and the
same invariance, on a frontier with a knee (\S\ref{sec:bfrontier}). What
this buys is a target-specific scan run before writing, not a formula read
off the file (\S\ref{sec:where}, \S\ref{sec:knowable}).

\emph{Seven, an exhausted file can be renewed, and the renewed file accepts
writing again.} Re-quantizing restores the room and issues a new artifact,
and that step costs the pretrained model $0.80\pm0.24$ points of mean suite
accuracy---small, consistently negative over the five events in the archive,
and independent of how much of the file changed. Renewal is also now a
theorem rather than only a measurement: the consolidation map is a bounded
projection ($Q\circ Q=Q$), so an empty major version is bitwise free and
each version's cost is bounded by what its fill wrote
(Prop.~\ref{prop:consproj}). The task written immediately after each
consolidation absorbs normally, which is the direct measurement that the new
artifact's cells are writable---it happens twice per cycle at both scales we
ran (\S\ref{sec:loop}, \S\ref{sec:loop8b}).

\emph{Eight, whether that can be repeated without limit, we do not settle.}
One turn of the loop is measured and so is the next; what we cannot yet say
is whether the sequence has a fixed point or diverges, and the reason to ask
is now a law rather than a hunch. Across nine consolidations in one sequence
the block scale compounds to $\times1.334$ with no equilibrium reached, and
the drift is bought by writing---growth
$=7.3\,\mathbb E|t|$, where $\mathbb E|t|$ is the mean fraction of its room
the fill moved each weight, at $r=0.991$ over nine arms, anchored at zero by
the projection identity (\S\ref{sec:driftlaw}). A cycle on real facts of
$28.1$~bits each holds $85$--$100\%$ of every earlier task through four
tasks at 8B with zero violations, and is in flight. This is the question the
paper leaves open, and \S\ref{sec:limitations} states what would settle
it.

\textbf{Contributions.}
(1)~A reframing of the quantized release from compression artifact to
\emph{fixed reference frame} for continued learning, with a formal update
contract---bitwise artifact invariance, exact revocability, and a drift budget exact to
second order and reported with its measured miscalibration---and the
elementary propositions that establish it
(\S\ref{sec:theory}).
(2)~A law of plasticity under sequential folding: when each task is written
into the room the last one left, the room decays geometrically and lifetime
capacity without re-quantization is finite (Proposition~\ref{prop:decay}).
The decay is measured---$\beta$ in $0.78$--$0.85$ over four matched
sequences on three hosts \emph{at 1.7B}, higher for recipes that use less of
the room, and $0.706$--$0.715$, $0.693$--$0.695$, $0.640$--$0.644$ across
two 8B sequences, where it
is not constant (\S\ref{sec:loop8b})---and its mechanism is a different one from the loss of plasticity that
afflicts unconstrained continual learning~\cite{plasticity}: here the space
runs out geometrically and by construction, as it does for
parameter-isolation methods~\cite{packnet}---though the room is not the whole
story, since a fixed-cell control in which it never decays still loses three
quarters of what the folding sequence loses (\S\ref{sec:sequential},
\S\ref{sec:loop}). The law is a property of the
folding operator and not of the quantization grid: a simulation on the
frozen artifact reproduces $\beta$ on equal-width symmetric cells, which
corrects our first explanation and transfers the law to uniform integer
grids (\S\ref{sec:sequential}). Room that is partitioned among writers in
advance, or kept as a revocable increment on a fixed anchor, is not consumed
in sequence (\S\ref{sec:fusion}).
(3)~Three training paths that realize the constraint, including CellFill, a
bounded reparameterization whose invariance holds by construction at every
step so that the trained and shipped models coincide (\S\ref{sec:methods}).
(4)~A measurement methodology for knowledge capacity in bits---synthetic
corpora with certain novelty and exactly countable information content,
a declared guessing floor, and cross-domain forgetting metrics disjoint from
the rehearsal corpus (\S\ref{sec:setup})---and the map of \emph{where} in a
network new knowledge is cheapest to write that it makes possible
(\S\ref{sec:where}).
(5)~An empirical study across methods, replay recipes, radii, corpus sizes,
scales, and architectures that locates the learning--forgetting frontier of
invariant updates, and reports two negative results we consider load-bearing:
a fixed rehearsal buffer is memorized and is worse than no rehearsal at all,
and the diagonal-Fisher budget of Prop.~\ref{prop:budget} is a scaling law,
not a numerical certificate (\S\ref{sec:results}, \S\ref{sec:geometry}).
(6)~A nested refinement code turning updates into prefix-compatible
$(4{+}k)$-bit checkpoints whose 4-bit truncation is always the original
release (\S\ref{sec:codec}).
(7)~The cycle run as a loop rather than stage by stage, at 1.7B and at 8B: six
tasks written in sequence into one published release, rehearsed throughout and
consolidated every second task, with retention of $94\%$ or better on every
earlier task at 8B and the vendor's codes returned on all $6.9\times10^{9}$
constrained weights at every fold---and, within it, the measurement that
isolates what issuing a major version costs the pretrained model, which is
$0.80\pm0.24$ points of mean suite accuracy
(\S\ref{sec:loop}, \S\ref{sec:loop8b}).

\section{Setting and Guarantees}\label{sec:theory}

\begin{definition}[Frozen artifact]
A blockwise quantizer with frozen group scales $s\in\mathbb{R}^{N/g}_{>0}$
and level table $L$ maps codes $a\in\{0,\dots,15\}^N$ to anchors
$\hat w_i = s_{g(i)}\,L[a_i]$. The artifact is $\mathcal A=(a,s)$. Two sets are needed and the
difference between them does the work. The \emph{decision region} $D_i$ of
weight $i$ is the set of values that round to $a_i$ under frozen $s$: the
interval between the midpoints to the neighbouring levels, half-open at one
end and half-infinite for the two outer codes. The \emph{cell} $C_i$ is what
an implementation can actually write into---$D_i$ capped at the level table's
ends and shrunk by a margin $\mu$ of its width on each side,
\[
C_i=[\ell_i,u_i],\qquad
\ell_i=\mathrm{lo}_i+\mu W_i,\quad u_i=\mathrm{hi}_i-\mu W_i,
\]
with $W_i$ the capped width and $\mu=1\%$ throughout. $C_i$ is a nonempty
compact interval lying strictly inside $D_i$ with clearance $\mu W_i$ at each
end; we write $\Delta_i=u_i-\ell_i=(1-2\mu)W_i$ for its writable width and
$M_i=\min(\hat w_i-\ell_i,\;u_i-\hat w_i)$ for the \emph{room}, the distance
from the anchor to the nearer wall, which is what a bounded update may spend
and is not $\Delta_i/2$ unless the anchor is centred. The invariant set is $\mathcal C=\prod_i C_i$, and every proposition
below is stated over it rather than over $\prod_i D_i$, which is neither
bounded nor closed and on which the clamp of Prop.~\ref{prop:proj} would not
be defined.
\end{definition}

\begin{proposition}[Bitwise invariance]\label{prop:inv}
If $w'\in\mathcal C$, re-quantization under frozen scales returns exactly
$a$; and if $w'$ is stored in any format whose rounding error is smaller than
$\mu W_i$ per coordinate, its stored representation does too. Consequently any
composition of updates whose images lie in $\mathcal C$ preserves the artifact
bit-for-bit.

The content is in the margin and in the word \emph{frozen}, not in the
rounding. Over the decision regions themselves the statement would be a
tautology---the set of values that round to $a$ rounds to $a$---and it would
not cover the two cases an implementation produces: a clipped merge lands
weights exactly on $\partial\mathcal C$, and a checkpoint stores them at
finite precision. The margin buys both, which is why $\mathcal C$ and not
$\prod_i D_i$ is the set the guarantee is stated over. Freezing buys the rest:
re-running a quantizer recomputes the absmax and can reassign an untouched
neighbour, and Appendix~\ref{app:frozen} exhibits a two-weight instance.

The rounding clause is not decoration, and it is not satisfied by every
format at every scale. Write $\omega=\min_a(\mathrm{hi}_a-\mathrm{lo}_a)$
for the narrowest normalized cell ($\omega=0.0805$ for NF4) and let a block
carry scale $s$, so $W_i\ge\omega s$. The clause asks
$\mu\,\omega\,s>\tfrac12\,\mathrm{ulp}(w)$ for every weight in the
block. In a float's normal range $\tfrac12\,\mathrm{ulp}(w)\le u\,|w|$ with $u$
the unit roundoff, and since the capped cells put $|w|\le 1.139\,s$ the scale
cancels: the question is decided once for the whole model, and fp16 clears it
by $\mu\omega/(1.139\,u)=1.45$. In the subnormal range
the spacing is a constant $\eta$ rather than $2u|w|$, the cancellation fails,
and what remains is a floor on the block scale,
\[
  s \;>\; \frac{\eta}{2\mu\omega},
\]
equal to $3.70\times10^{-5}$ for fp16 at $\mu=1\%$. A block below it cannot
be written invariantly at that storage precision no matter what the optimizer
does. Weights in such blocks are given zero room and never move;
\S\ref{sec:floor} reports how many there are at each scale and what
excluding them costs.
\end{proposition}
\begin{proof}
Immediate from the definition of the decision regions. The content of the
proposition is the precise statement of what must be frozen: recomputing
absmax scales after training silently reassigns entire blocks (we exhibit a
two-weight counterexample in the appendix), so invariance must be defined
against frozen $(a,s)$, never against re-running a quantizer.
\end{proof}

\begin{proposition}[Projection optimality]\label{prop:proj}
Coordinatewise clipping to $\mathcal C$ is the Euclidean projection
$P_{\mathcal C}$; clip-merge returns the invariant model nearest to any
proposed update, and projected SGD on $\mathcal C$ is proximal descent on the
box indicator.
\end{proposition}

\begin{proposition}[Nested refinement code]\label{prop:codec}
For in-cell position $x_i\in[0,1)$ let $r^{(k)}_i=\lfloor 2^k x_i\rfloor$ with
sub-cell-center reconstruction. Then (i) $r^{(k)}=r^{(k+j)}\!\gg\!j$;
(ii) the reconstruction error is at most $\mathrm{width}_i/2^{k+1}$;
(iii) reconstructions are strictly interior, so every truncation depth
satisfies Proposition~\ref{prop:inv}.
\end{proposition}

\begin{proposition}[Capacity identity]\label{prop:capacity}
An update message $m=(\text{mask},r^{(k)})$ has length
$|m|\le kN_t+\log_2\binom{N}{N_t}$ bits, and for any fact set $F$ and
updated model $W'=f(\hat W,m)$, the data-processing
inequality~\cite{cover} gives
$I(F;W')\le H(m)\le|m|$. Absorbed knowledge is bounded by shipped bits as a
theorem, not an estimate.
\end{proposition}

\begin{proposition}[Forgetting budget]\label{prop:budget}
For $|\delta_i|\le\rho\,\Delta_i/2$,
$\mathbb E_x\,\mathrm{KL}\!\left(p_{\hat W}\,\|\,p_{\hat W+\delta}\right)
=\tfrac12\,\delta^\top F\,\delta+O(\|\delta\|^3),$
so to second order the drift of any update confined to that box is controlled
by the Fisher form on radii the quantization grid supplies for free, and
scales as $\rho^2$. Which updates the hypothesis covers is worth stating
exactly, because the three paths differ. CellFill moves inside the symmetric
box $|\delta_i|\le\mathrm{room}_i=\min(\hat w_i-\ell_i,\,u_i-\hat w_i)$, and
since $\mathrm{room}_i\le\Delta_i/2$ (Appendix~\ref{app:proofs}) the
proposition applies to it directly. The projected paths do not: clip-merge and
projected dense clamp to the full cell $[\ell_i,u_i]$, so a weight may travel
up to $\max(\hat w_i-\ell_i,\,u_i-\hat w_i)$, which exceeds $\Delta_i/2$
wherever the cell is asymmetric about its anchor---which, for the non-uniform
NF4 level table, is fourteen of the sixteen codes (the two extremes are
symmetric only because capping mirrors their inner half-gap). For those paths the proposition bounds
the symmetric part of the excursion and not the whole of it.
\end{proposition}

We stress what this does and does not give. The expression is exact to
second order, but evaluating it requires the full $F$, and replacing $F$ by
its diagonal gives two different surrogates that must not be confused. The
supremum over the loose box is $B_F(\rho)=\tfrac{\rho^2}{8}\sum_i
F_{ii}\Delta_i^2$; the \emph{mean} over a uniform fill of the true cells is
$\bar B_F(\rho)=\tfrac{\rho^2}{6}\sum_i F_{ii}\,\mathrm{room}_i^2$
(Appendix~\ref{app:proofs}), and it is $\bar B_F$ that our geometry sweep
evaluates. Neither is a usable numerical bound: measured against random
in-cell perturbations, $\bar B_F$ underestimates the true drift by
$15.9\times$ at $\rho{=}0.125$ rising to $383\times$ at $\rho{=}1$, and since
$\mathrm{room}_i\le\Delta_i/2$ gives $B_F\ge3\bar B_F$, the supremum form
underestimates by at most $128\times$---two orders of magnitude either way.
\S\ref{sec:geometry} reports the measurement and the Limitations section the
two explanations our data cannot separate. We therefore treat $\rho$ as a
calibrated dial rather than an a-priori certificate, and report the measured
curve.

\begin{proposition}[Room contraction under folding]\label{prop:decay}
Write $a=\hat w-\ell$ and $b=u-\hat w$ for a weight's distances to the two
walls of its cell, $W=a+b$ for the cell's fixed width, $d=a-b$ for its
asymmetry and $r=\min(a,b)$ for the inradius about its current anchor. Let a
bounded update move the weight to $w'=\hat w+r\,t$ with $|t|<1$, as CellFill
does with $t=\tanh(s\,BA^{\!\top})$, and fold $w'$ into the anchor for the
next task. Then the surviving inradius is exactly
\[
r' \;=\; \min\!\big(a+rt,\;b-rt\big)
      \;=\; \frac{W}{2}-\frac{|d+2rt|}{2},
\qquad\text{while}\qquad r=\frac{W}{2}-\frac{|d|}{2},
\]
and three things follow. \emph{(i)} $r'\ge r(1-|t|)$, with equality when the
update moves toward the nearer wall and always when the cell is symmetric.
\emph{(ii)} A single fold can \emph{increase} the room: moving toward the far
wall recentres the anchor, and for $a=1$, $b=10$, $t=+\tfrac12$ the inradius
goes from $1$ to $\tfrac32$. A lower bound on $r'$ therefore does not by
itself make the room decay, and the finiteness claim needs the next part.
\emph{(iii)} If the sign of $t$ is symmetric, convexity of $|\cdot|$ gives
$\mathbb E|d+2rt|\ge|d|$ and hence
\[
\mathbb E[r']\;\le\;r ,
\]
with equality exactly when $|d|\ge 2r|t|$---when the update never changes
which wall is nearer---and with the maximal contraction
$\mathbb E[r']=r(1-|t|)$ at a symmetric cell. The room is therefore
non-increasing in expectation under folding, and it contracts fastest where
the anchor sits centred in its cell.
\end{proposition}

Part (iii) is what the sequential results rest on, and it predicts their
shape rather than only bounding it. A released grid starts with many
near-centred anchors, so the first fold contracts hard; every fold after it
leaves the positions off-centre, $|d|$ grows, and the contraction weakens
toward its equality case. That is what \S\ref{sec:sequential} measures---the
first fold's ratio is $0.61$ and the ratios from the second fold on sit at
$0.80$--$0.84$---and what the grid simulation there finds when it varies the
level table and changes nothing after the first fold. The geometric law
itself is empirical: writing $\bar r_k$ for the mean inradius after $k$ tasks,
\[
\bar r_k \;\approx\; \bar r_0\,\beta^{\,k},
\qquad \beta\in[0.78,0.85]\ \text{over four matched 1.7B sequences,}
\]
which is well above the per-weight edge $1-\mathbb E|t|\approx0.5$ that (i)
allows, so the inequality is satisfied with slack everywhere we have measured
and is not the operative constraint. If a task's absorption is proportional
to the room available to it, $a_k=c\,\bar r_{k-1}$, then $a_k=a_1\beta^{\,k-1}$
and the knowledge absorbable without ever changing the artifact is finite,
$\sum_{k\ge1}a_k=a_1/(1-\beta)$. That proportionality is an assumption and
\S\ref{sec:loop} reports it holding on one anchor and failing on another, so
the finiteness conclusion is carried by measurement and not by the
proposition.

Two consequences are worth separating. First, the characteristic failure is
not a violated bound but an exhausted one: the sequence becomes unable to
learn while every check still passes, which is a failure mode a
constraint-violation check cannot detect. (The bound is not inviolable in
practice---one weight of $1.4\times10^{9}$ left its cell at the sixth fold of
each of two sequences, \S\ref{sec:rehearse-old}---but that is a numerical
event at the margin, not the mode this proposition describes.)
Second, per-task throughput $a_k\to0$, so a policy that never re-quantizes
has asymptotically zero learning rate; consolidating every $m$ tasks
restores $\bar r$ and yields a positive steady state
$a_1(1-\beta^m)/\big(m(1-\beta)\big)$, decreasing in $m$. Consolidation is
therefore not a fallback for when something goes wrong but a structural
requirement of learning in a bounded space, and $m$ is the dial trading
learning throughput against how often the released artifact changes.

\begin{proposition}[A consolidation is a bounded projection]\label{prop:consproj}
Let $Q$ be the consolidation map as implemented: per block, $s'=\max_i|w_i|$,
$c'=\mathrm{RTN}(w/s')$, $\hat w'=L_{c'}s'$, and write $H_c$ for the margined
half-width of cell $c$. Then (i)~$Q\circ Q=Q$: a consolidation that follows
no writing returns the same codes and scales, bitwise. (ii)~For every weight
whose code is unchanged the snap error decomposes exactly,
$w-\hat w' = M\,t + L_{c}\,(s-s')$: a fill term and a scale-drag term.
(iii)~$|s'|\le(1+\lambda_0)|s|$ with $\lambda_0=0.1489$, attained at code~0
and a property of the capping convention, not of NF4's levels. (iv)~The room
the consolidation restores is $|s'|\,\mathbb E[H_{c'}]$, pinned by the format
to first order. Sketches in \S\ref{app:proofs}; every clause is verified
numerically against the shipped quantizer in \S\ref{sec:driftlaw}.
\end{proposition}

Clause~(i) is what makes repeated renewal safe to reason about: the map is a
projection, so each version's injection is bounded by what the fill wrote
since the last version ($0.145\,s$ per weight at most), not by how many
versions preceded it. Clause~(iv) carries a measured second-order correction:
scale growth re-bins weight mass toward the central, narrow cells, so
$\mathbb E[H_{c'}]$ falls as $|s'|$ grows; \S\ref{sec:driftlaw} measures the
deficit and predicts it from the level table with no fitted parameter.

\section{Methods}\label{sec:methods}

\subsection{CellFill: training the position inside the cell}\label{sec:cellfill}
\begin{figure}[t]\centering
\includegraphics[width=\linewidth]{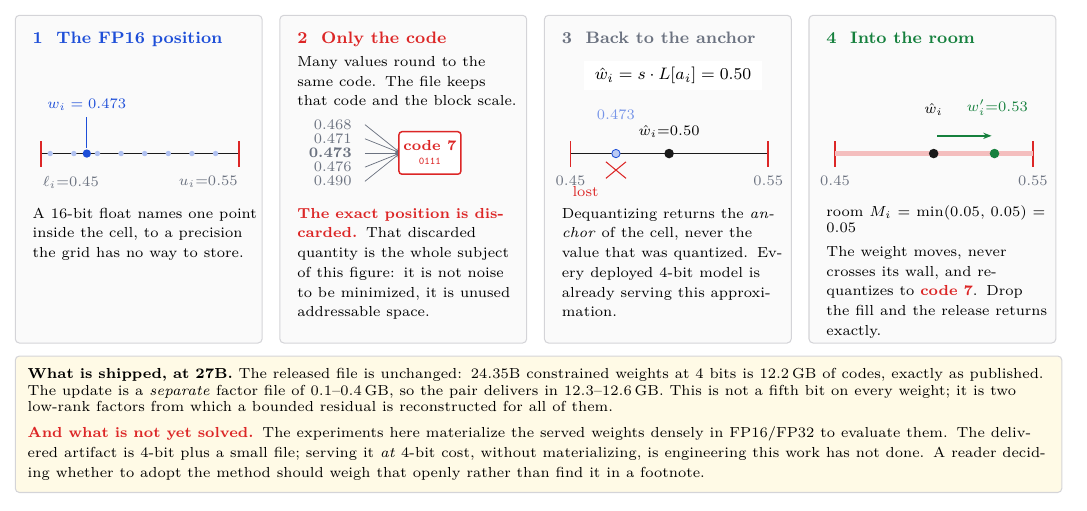}
\caption{Where the update lives, followed through one weight. The released
file stores a code and a block scale; the quantizer discarded where inside
the cell the original weight was, and that discarded interval is the only
place CellFill writes. The bounded link keeps the served weight strictly
inside the cell, so re-quantizing it returns the code that was shipped---the
identity is a property of the parametrization rather than of a constraint
applied afterwards. Serving the fill is not free: the update is a separate
file that must be applied to the released weights before the model runs.}
\label{fig:mechanism}
\end{figure}

What is new here is the object that is trained. A 4-bit release stores,
for every weight, a code $a_i\in\{0,\dots,15\}$ and, for every block of
$g$ weights, a scale $s_{g(i)}$; the served value is the anchor
$\hat w_i=s_{g(i)}L[a_i]$, and the quantizer's rounding discards where
inside the cell $C_i$ the original weight was. CellFill never trains
$\hat w_i$, $a_i$ or $s_{g(i)}$. It trains the discarded quantity---the
\emph{position} of the weight inside its cell---and nothing else, so that
the stored bits are invariant because the trainable object cannot reach
them, rather than because a constraint is enforced after the fact.
Figure~\ref{fig:mechanism} follows one weight through the construction,
which has five parts.

\textbf{The cell and its room.} Under the shipped scale, the cell of code
$a_i$ is the interval between the midpoints to its neighbouring levels
(the outer two cells capped at the level table's ends), shrunk by a margin
$\mu$ of its width on each side so that stored values stay off the
decision boundaries; $\mu=1\%$ throughout. Its \emph{room} is the distance
from the anchor to the nearer wall,
$M_i=\min(\hat w_i-\ell_i,\;u_i-\hat w_i)$. NF4 anchors are not centred
(level 1 sits $14\%$ of a cell off centre), so $M_i$ is a property of the
code and the scale, rebuilt exactly from the sixteen-entry table and the
block's scale, and the scale is used with its sign: a double-quantized
release dequantizes the absmax of near-zero blocks to small negative
numbers, which flip the cell rather than collapse it (Qwen3-8B: $864{,}116$
weights). $M_i=0$ freezes a weight.

\textbf{The fill.} For a matrix with anchors $\hat W\in\mathbb R^{m\times n}$
and room $M$, the served weight is
\[
W \;=\; \hat W + M\odot \tanh\!\big(s\,B A^{\!\top}\big),
\qquad A\in\mathbb R^{r\times n},\; B\in\mathbb R^{m\times r},
\]
with $B=0$ at initialization (the fill starts at the anchor), $A$ random,
and $s$ a fixed slope ($40$ unless stated; $10$ in the repaired recipe of
\S\ref{sec:capacity-xdom}). Every entry of $\tanh(\cdot)$ lies in
$(-1,1)$, so $|W_{ij}-\hat W_{ij}|<M_{ij}$ for every weight no matter what
the optimizer does: invariance is a property of the parametrization, not
of the training run. $A$ and $B$ are the only trainable parameters;
embeddings, norms and the head are frozen; at $r{=}64$ they number
$r(m+n)$ per matrix, $69.7$M at 1.7B, $5\%$ of the weights they move.
The elementwise $\tanh$ and the elementwise $M$ both break the low-rank
structure of $BA^{\!\top}$, so the fill is full-rank: in a $512\times512$
probe at $r{=}16$ the product has rank $16$ and the fill rank $512$
(stable rank $110$). A rank-$r$ fill is therefore not a rank-$r$ update,
which is the most likely reason it exceeds projected dense training at
$20\times$ fewer parameters (\S\ref{sec:efficiency}). The link can be
replaced: a hard bound with a straight-through gradient, a $\tanh$ whose
backward derivative is floored, or the softsign, all keep the same
reachable set and differ only in where gradient survives
(\S\ref{sec:capacity-xdom}).

\textbf{The forward pass.} Each 4-bit linear layer computes its output as
the release would, $\hat W x$ through the vendor's kernel, plus
$(M\odot\tanh(sBA^{\!\top}))\,x$, so the release is never dequantized into
a trainable copy. The fill is materialized per layer in bf16 during
training and recomputed in the backward pass when memory requires
(checkpointing, from 8B up); $M$ is kept as a dense bf16 buffer up to 8B
and rebuilt per forward pass from the codes and the table above that,
where a dense $M$ would be $48.7$~GB at 27B. The final merge recomputes
the fill in fp32.

\textbf{Training.} AdamW at $10^{-3}$, batch 16, sequences of at most 96
tokens, 24 epochs over the corpus, with $10\%$ of every epoch's batches
drawn fresh from WikiText-103's training split---rehearsal that is never
repeated, because a fixed buffer is memorized and sharpens the model onto
itself (\S\ref{sec:rehearsal}). There is no projection step and no
clamping: the optimizer moves $A$ and $B$ freely and the weights stay
inside their cells by construction. Two optional terms change what the
fill is allowed to do: a distillation term that keeps the served
distribution unchanged on a neutral pool (an \emph{inert} fill,
\S\ref{sec:fusion}), and a share $\rho<1$ of the room (a \emph{reserved}
fill) so that $K$ reserved fills sum inside the cell.

\textbf{Materialization, verification, serving.} After training, the
served weights $W=\hat W+M\odot\tanh(sBA^{\!\top})$ are computed in fp32
and every weight is re-binned under the shipped scales: the code it
receives must equal the stored code. This is the check of
Proposition~\ref{prop:inv}, run on every merge in this paper and reported
as a count of violations. It is zero on every run in this paper except two,
a single weight at the sixth fold of each of two long sequences
(\S\ref{sec:rehearse-old}); the count is the evidence, and a run that fails
it fails. Zero is not automatic: it presupposes the rounding clause of
Prop.~\ref{prop:inv}, which an anchor can violate, and \S\ref{sec:floor}
gives the pre-flight check that decides whether a given anchor admits it. What is shipped is the
unchanged 4-bit file and a \emph{fill}---the factors $A$ and $B$, $0.1$--$0.4$~GB
at 27--31B---or, where bandwidth matters, the fill's positions coded as
refinement bits on top of the 4-bit codes (\S\ref{sec:codec}). Revoking
an update discards the fill; the served model is the release again.
Sequential updates \emph{fold}: the current position becomes the next
update's anchor, the room shrinks to the nearer wall, and a fresh fill is
trained (\S\ref{sec:sequential}); several writers share a release by
reserving room, by training inert fills, or by rounds that fold the
average of their fills (\S\ref{sec:fusion}).

\subsection{Two projected paths}\label{sec:paths}
Any procedure that confines learning to the cells is \emph{in-cell
learning}, and two others were run as references. \textbf{A (clip-merge):}
train LoRA~\cite{lora} on the quantized base, materialize the delta
layerwise and set $W'=P_{\mathcal C}(\hat W+\Delta)$---cheapest, optimal in
weights rather than in loss (Prop.~\ref{prop:proj}); \textbf{A+} continues
from $W'$ with projected dense fine-tuning, so the optimizer sees the walls
and re-routes the clipped-away knowledge. \textbf{B (projected dense):}
constraint-aware from step zero, every weight free inside its cell, the
reference solution of the constrained problem. The clip rate of path A is a
computable diagnostic of the divergence between $P_{\mathcal C}(\arg\min f)$
and $\arg\min_{\mathcal C}f$: below $\sim$1\% the paths agree; above
$\sim$5\%, escalate to A+ or B. CellFill is not a variant of any of them,
nor of the PEFT family: DoRA~\cite{dora} and AdaLoRA~\cite{adalora} still
produce an unbounded weight delta added to a frozen base, and QA-LoRA~\cite{qalora}
and LoftQ~\cite{loftq} re-quantize after merging, so their served files
are new files. CellFill's served file is the old one. Two PEFT families do
leave the codes untouched---prefix tuning, whose update lives in the KV
cache, and LN tuning, whose parameters a 4-bit release does not quantize---
and \S\ref{sec:baselines} measures what they can carry.

\section{Experimental setup}\label{sec:setup}
Synthetic biography facts (novelty certain by construction; $21.7$ bits per
fact of countable attribute entropy; three cloze probes per fact), fresh
per-epoch rehearsal from wikitext-train (a fixed replay buffer repeated
across epochs is itself memorized and \emph{raises} test perplexity from
24.6 to 172---a negative result we document), forgetting measured by
wikitext-test and cross-domain LAMBADA perplexity.
Unless stated otherwise the base model is Qwen3-1.7B-Base; the scale ladder
(0.6B--27B, including a hybrid linear-attention model) and the second model
family are reported in \S\ref{sec:scale}.
Quantization: NF4, blocksize 64, double quantization, frozen scales;
invariance is asserted per layer, in the integer domain, on every merge.
Real new knowledge is scored on \emph{NewFacts~v2}, a frozen benchmark whose
files are content-hashed into the loader so that silent drift is an error:
three cells, of which the two lottery cells carry exactly $28.12$ and
$28.11$ bits per fact by combinatorics and differ by $0.03\%$---the one
comparison in this paper matched by construction---and released 1.7B and 8B
anchors score $0.0000$ on all 468 probes, measured per item. Its composition
instrument (v2.1) was rebuilt after the leak reported in
\S\ref{sec:usable}: two disjoint sentence families per drug, verified at
build time never to co-state the composed relation, with both hop directions
taught so the probe isolates chaining.

Each fact carries $21.7$ bits of attribute entropy, but the three cloze
probes cover only city, occupation and employer: $11.9$ bits. Two
denominators therefore appear, and which one is in use is stated wherever
they do: capacity read off recall is denominated in the probed $11.9$ bits,
while the log-probability bit accounting of Table~\ref{tab:capacity-bits}
scores all five attributes and is denominated in $21.7$. Two probes continue the training
phrasing and one paraphrases it, so we report recall per probe kind as well
as pooled. A model that has learned the answer vocabulary but no
name--attribute bindings scores $6.5\%$ under greedy decoding; that is the
floor against which small recall numbers should be read.

\section{Results}\label{sec:results}
The results follow the deployed life of one released file, and are grouped by
the five stages of the cycle the title names.
\S\ref{res:contract} establishes what a minor version \emph{is}---an update
that re-quantizes to the shipped codes---and prices it: what the constraint
costs in recall, where the writable region lies, and what it gives back.
\emph{Write} (\S\ref{res:write}) asks what one such version can put into a
released model: real knowledge the model verifiably lacked, written where in
the network, on whose file, and against what the alternatives achieve.
\emph{Saturate} (\S\ref{res:saturate}) asks what stops one update and what
caps its return---the bounded link saturates before the cells fill, and a
scan of what the model already knows predicts, before training, how much an
injection will give back.
\emph{Validate} (\S\ref{res:validate}) is the gate: whether the written
knowledge can be used, and whether the served model is still the model the
release was evaluated as.
\emph{Consolidate and rehearse} (\S\ref{res:lifelong}) runs updates in
sequence---room consumed at a measured geometric rate, rehearsal recovering
most of what the sequence forgets, consolidation renewing the room at a
stated price in changed codes---and \S\ref{sec:loop} runs the whole cycle as
a loop over six tasks and three turns. \S\ref{sec:loop} is also where we say
which stages that loop did \emph{not} include.

\begin{figure}[t]\centering
\includegraphics[width=\linewidth]{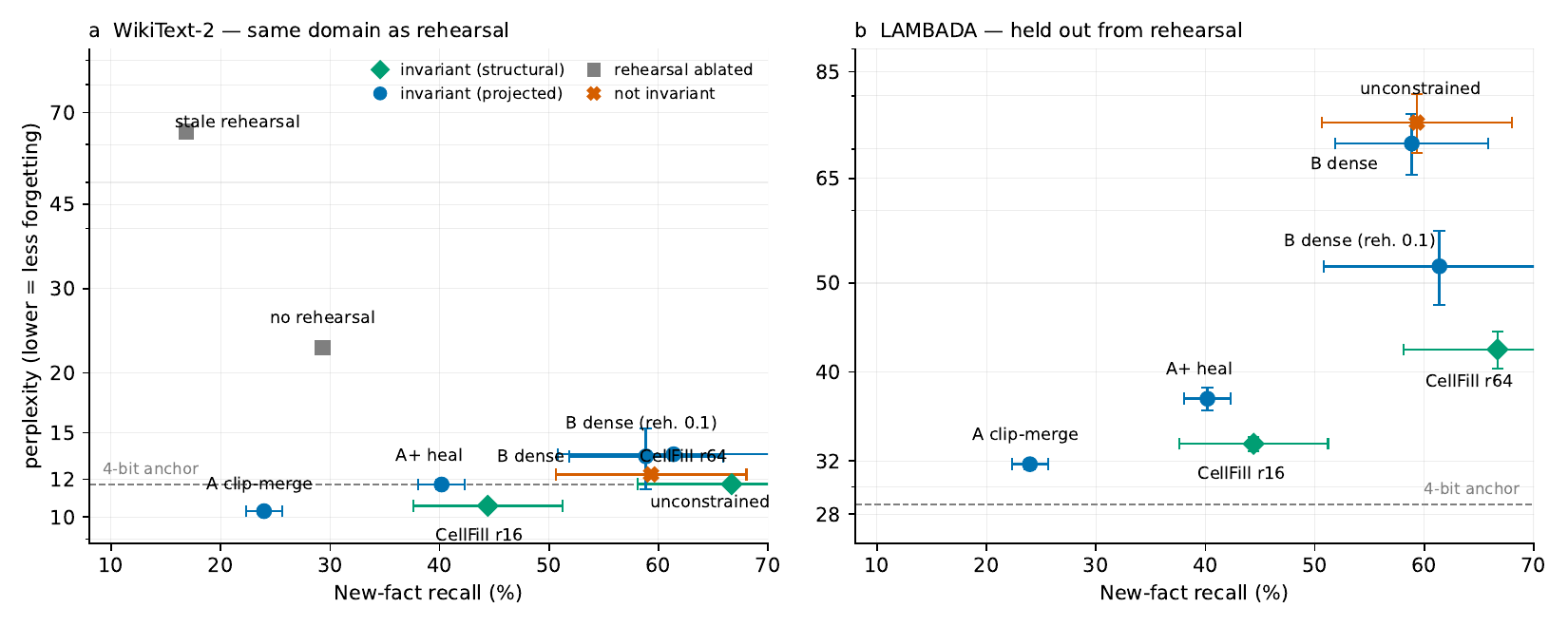}
\caption{The invariant learning--forgetting frontier (Qwen3-1.7B, 1{,}000
synthetic facts, 24 epochs; log perplexity axes, error bars $\pm$1 std over
three seeds). Every diamond, circle and square re-quantizes bit-identically
to the released 4-bit artifact; the cross does not.
\textbf{The two panels disagree, and that disagreement is the point.}
Rehearsal is drawn from WikiText-train, so in (a) every method sits at or
below the anchor and in-cell learning looks free; LAMBADA in (b) never saw
the rehearsal corpus, and there every method is far above it. A reader who
measures only (a) concludes that nothing was given up. The two gray squares
in (a) ablate rehearsal: omitting it doubles perplexity, and re-using a
fixed buffer across epochs is far worse still ($63.8$), because the buffer
itself gets memorized. Note also that ordering by recall is not ordering by
cost: CellFill at rank 64 absorbs \emph{more} than projected dense while
adding $13.7$ points of cross-domain perplexity over the anchor against
dense's $23.5$ at matched rehearsal, so it sits below and to the right of the
projected arm rather than further along the same curve.}
\label{fig:frontier}
\end{figure}

\subsection{The contract: what a minor version is, and what it costs}\label{res:contract}
\subsubsection{The invariant frontier (Qwen3-1.7B, 1k facts, 24 epochs)}\label{sec:frontier}
Learning and forgetting trade off, and the first question is where a
bit-identical method sits on that trade. Table~\ref{tab:frontier} and
Figure~\ref{fig:frontier} place every path we run on the same axes: how much
of a 1{,}000-fact corpus each absorbs, what each gives up on an in-domain and
a held-out perplexity, and whether the released 4-bit file survives.
Every \checkmark\ row re-quantizes to the shipped codes over all
$1.409\times10^9$ constrained weights; the one row that does not is the
unconstrained control, and it is there to price the constraint rather than to
compete. CellFill at $r{=}64$ does not sit further along the projected
methods' curve---it sits off it, absorbing more than projected dense while
giving up less than half again as much held-out perplexity---which is the
observation the rest of \S\ref{res:contract} takes apart.

\begin{table}[t]\centering\small
\setlength{\tabcolsep}{4pt}
\begin{tabular}{llccccrc}
\toprule
method & invariance & reh. & recall & WikiText-2 & LAMBADA ppl$\downarrow$ &bits/pt & seeds \\
\midrule
4-bit anchor & (reference) & --- & --- & 11.71 & 28.67 & --- & --- \\
original fp32 & (reference) & --- & --- & 10.66 & 26.13 & --- & --- \\
\midrule
A: clip-merge & \checkmark & 0.1 & $24.0\!\pm\!1.7$\% & $10.29\!\pm\!0.02$ & $31.74\!\pm\!0.37$ & $917\!\pm\!183$ & 3/2 \\
CellFill $r{=}16$ & \checkmark\ (structural) & 0.1 & $44.4\!\pm\!6.8$\% & $10.56\!\pm\!0.04$ & $33.40\!\pm\!0.59$ & $1125\!\pm\!205$ & 3 \\
A+: clip-merge $\to$ heal & \checkmark & 0.1 & $40.2\!\pm\!2.1$\% & $11.69\!\pm\!0.27$ & $37.40\!\pm\!1.09$ & $568\!\pm\!88$ & 3/2 \\
CellFill $r{=}64$ & \checkmark\ (structural) & 0.1 & $66.7\!\pm\!8.6$\% & $11.71\!\pm\!0.08$ & $42.33\!\pm\!1.99$ & $583\!\pm\!59$ & 3 \\
B: projected dense & \checkmark & 0.1 & $61.4\!\pm\!10.6$\% & $13.53\!\pm\!0.05$ & $52.15\!\pm\!4.82$ & $312\!\pm\!10$ & 2 \\
B: projected dense & \checkmark & 0.3 & $58.9\!\pm\!7.0$\% & $13.39\!\pm\!1.94$ & $70.98\!\pm\!5.43$ & $165\!\pm\!48$ & 3/2 \\
unconstrained dense & $\times$ & 0.3 & $59.3\!\pm\!8.7$\% & $12.27\!\pm\!0.13$ & $74.83\!\pm\!5.57$ & $151\!\pm\!48$ & 3/2 \\
\bottomrule
\end{tabular}

\caption{The invariant learning--forgetting frontier (Qwen3-1.7B, 1{,}000
facts, 24 epochs). Every \checkmark\ row re-quantizes bit-identically to the
released artifact---asserted per layer, in the integer domain, over all
$1.409\times10^9$ constrained weights, on every merge. Cells are
mean$\pm$std; the last column gives the number of seeds (recall/LAMBADA,
where they differ because the cross-domain harness was added after the first
seed). \textbf{The two perplexity columns disagree by construction}:
rehearsal is drawn from WikiText-train, so WikiText-test flatters every
method, while LAMBADA is untouched by rehearsal and prices the true cost.
Rows A and A+ are the same runs before and after healing. The \emph{reh.}\ column gives the fresh-rehearsal fraction: the bottom two rows were run at $0.3$ because the paired test of \S\ref{sec:cost} needs the constrained and unconstrained arms at the same fraction as each other; the matched $0.1$ dense row directly above them is the one to compare against A, A+ and CellFill. \emph{bits/pt} is
the efficiency measure of \S\ref{sec:efficiency}: probed bits of knowledge
absorbed per point of LAMBADA perplexity given up. This table is
generated from the archived result files by
\texttt{scripts/make\_tables.py}.}
\label{tab:frontier}
\end{table}

\subsubsection{Cost is not proportional to absorption: methods differ in
efficiency}\label{sec:efficiency}
It would be convenient if cross-domain damage simply tracked how much a
method absorbs. It does not, and the gap between methods is large. Ranking
the invariant paths by \emph{bits absorbed per point of cross-domain
perplexity} (Table~\ref{tab:frontier}, \emph{bits/pt}) gives
three tiers, not a five-way ordering. The figures below are computed per
seed and then averaged, at a common rehearsal fraction of $0.1$, which is
why they differ by a few points from the table's own \emph{bits/pt} column,
where the ratio is taken between the columns' means: clip-merge
$917\!\pm\!183$ and CellFill $r{=}16$ $1125\!\pm\!205$; A+ $568\!\pm\!88$ and
CellFill $r{=}64$ $583\!\pm\!59$; projected dense $312\!\pm\!10$. The tiers
separate by $3.2\times$ and that separation is far outside the spreads. The
ordering \emph{within} a tier is not resolved and we do not claim it---indeed
A+ and CellFill $r{=}64$ change places depending on whether the ratio is
taken per seed or between means, which is precisely why we now take it per
seed. What the tiers show is that the exchange rate is \emph{not} the recall
ordering: projected dense at rehearsal $0.3$ has near the highest recall and
the worst exchange rate, CellFill $r{=}16$ nearly the reverse, and CellFill
$r{=}64$ is the only arm that is near the top of both.

This is the quantity a deployment actually optimizes, and it separates two
regimes that raw recall conflates. Bounded low-rank paths stay near the
anchor and buy knowledge cheaply but cannot buy much of it; full-rank paths
buy a lot at a steeply rising price. The efficient operating points are in
the middle tier: CellFill at rank 64 absorbs $1.5\times$ what CellFill
$r{=}16$ does at roughly half its efficiency, while projected dense at
the same rehearsal fraction absorbs $8\%$ \emph{less} for $1.7\times$ the
cross-domain damage ($23.5$ points against $13.7$) and without the structural
guarantee. That comparison is a tier crossing and survives the spreads; the
choice between A+ and CellFill $r{=}64$ inside the middle tier does not, and
should be made on the guarantee rather than on these numbers. The bounded
reparameterization therefore does not trade recall for its guarantee against
the projected dense path: on these three seeds it is ahead on both.

\subsubsection{What does exact invariance cost?}\label{sec:cost}
The obvious worry is that forbidding the $0.54\%$ of weights that an
unconstrained run moves out of its cell from going where the optimizer
wants must cost accuracy. We measure it directly: identical dense training,
identical seeds, the only difference being whether the projection is
applied after each step. Across three paired seeds the constrained arm
recovers $58.9\!\pm\!7.0\%$ of facts and the unconstrained arm
$59.3\!\pm\!8.7\%$; the paired difference is $-0.47$ recall points with a
95\% interval of $[-5.0, +4.0]$ ($t=-0.45$, $\mathrm{df}=2$). We therefore cannot exclude a cost of five recall points, nor a gain of
four; the point estimate is indistinguishable from zero, while the unconstrained arm
moves $7.5$--$7.8$ million weights ($0.54\%$) out of their cells per run
and destroys the artifact.

On the strict axis the constrained arm is ahead: LAMBADA $70.98$ versus
$74.83$ (seeds 1--2 only; seed 0 predates the cross-domain harness, so
this pair is $n{=}2$ where the recall test above is $n{=}3$). Constraining the update is not merely affordable here; it is
weakly better, for the reason developed in \S\ref{sec:regularizer}.
The paired analysis is what we report: a single-seed reading of the same
runs gives ``1.8 recall points'' of cost, which is an artifact of $n{=}1$
against a between-seed standard deviation of $7$--$9$ points.

\paragraph{What it costs to run, against the thing it replaces.}
The unconstrained arm above is the closest control we have to full-parameter
fine-tuning of the released model: every one of the $1.409\times10^9$
constrained weights is trainable and free to move. Priced against it,
CellFill at rank 64 trains $69{,}730{,}304$ parameters---$20\times$
fewer---in $12.0$ minutes against $20.1$, with the Adam state falling from
$10.5$\,GB to $0.52$\,GB, and it recovers $66.7\!\pm\!8.6\%$ of the facts,
nominally more than the control's $59.3\%$ though the spreads overlap, so we
claim only that the constrained path does not cost recall. It gives up little
more than half as much cross-domain ability in the process (LAMBADA
$42.33\!\pm\!1.99$ against $74.83\!\pm\!5.57$). The two arms also differ in
kind at the end: the control produces a new checkpoint that replaces the
release, while the constrained run ships an increment that re-quantizes to
the release bit-for-bit and can be dropped to recover it exactly. The
comparison is therefore not only cheaper---it returns a different kind of
object.

\subsubsection{Which anchors admit an invariant fill}\label{sec:floor}

The rounding clause of Prop.~\ref{prop:inv} is a hypothesis about the
anchor, not about the method, and it is the one place where a released
checkpoint can refuse the contract outright. Our implementation stores
anchors in fp16---the choice that makes a sequence fit on one card at
27B---so the floor derived after Prop.~\ref{prop:inv} applies:
a block of scale $s<3.70\times10^{-5}$ puts its weights in the format's
subnormal range, where the spacing no longer shrinks with $s$ while the cell
does, and the $\mu W_i$ clearance is consumed by the write-back.

Scanning every block of four anchors against that floor:

\begin{center}
\begin{tabular}{lrrr}
\toprule
anchor & smallest served scale & blocks below & weights frozen \\
\midrule
Qwen3-1.7B-Base NF4 & $2.754\times10^{-3}$ & $0/2.20\times10^{7}$ & $0$ \\
Qwen3-4B-Base NF4 & $1.882\times10^{-5}$ & $1/5.68\times10^{7}$ & $64$ \\
Qwen3-8B-Base NF4 & $9.872\times10^{-7}$ & $4253/1.09\times10^{8}$ & $272{,}192$ \\
Qwen3.8-27B NF4 (ours) & $1.700\times10^{-4}$ & $0/3.80\times10^{8}$ & $0$ \\
Qwen3-32B NF4 & $1.110\times10^{-6}$ & $37/4.88\times10^{8}$ & $2{,}368$ \\
\bottomrule
\end{tabular}
\end{center}

The column that matters is the first, and it is not ordered by size: the
27B anchor, $2.435\times10^{10}$ constrained weights, clears the floor by
$4.6\times$, while the 8B release carries served scales $37\times$ beneath it and the 32B release carries thirty-seven such blocks out
of $4.9\times10^{8}$. Where they sit is more regular than how many there are:
at 8B they are confined to \texttt{layers.\{1,2,3\}.mlp.gate\_proj} and at 32B
to \texttt{layers.\{1,2,5\}.mlp.gate\_proj}, in both cases scattered over
near-dead output rows in runs of one or two rather than in contiguous
regions.

The first column says \emph{served} scale and means it. Every anchor here
ships with double quantization on---blocksize $64$, the scale vector itself
coded at $8$ bits in groups of $256$---and in this regime the served scale
stops tracking the weights. Turned off, the recovered scale is bit-exactly
$\max|w|$ over each run of $64$: at 1.7B, over all $2.20\times10^{7}$ blocks,
the two agree to $0.000000$. Turned on, they do not. The 8B release serves
$15{,}903$ \emph{negative} scales, and a maximum of absolute values is never
negative. The 27B is the case where both sides can be computed, its dense
release being on disk: $574$ blocks whose true $\max|w|$ falls under the
floor are served scales of $1.70$ to $4.61\times10^{-4}$, so none of them is
flagged, while a run of $48$ consecutive blocks in
\texttt{layers.9.linear\_attn.in\_proj\_qkv} collapses onto the single
constant $-4.606\times10^{-4}$ though their true scales range over
$1.5\times10^{-5}$ to $1.3\times10^{-4}$. The second level cannot resolve
this range, and its error runs both ways: it pushes some blocks under the
floor that are not, and lifts others out of sight that are.

None of that reaches the guarantee, because the served scale is the one the
guarantee is stated over. It is what \texttt{bitsandbytes} dequantizes with,
what the cells are built from, and what re-quantization compares against, so
a fill computed against it is inside the cells that will actually be used.
It does mean the question ``does this anchor admit an invariant fill'' is not
answerable from the base weights. Two things decide it and only one is in the
weights: the checkpoint decides whether blocks exist whose $\max|w|$ falls
orders below their layer's typical block, and the quantization run decides
what scale is then served for such a block. Neither is recoverable from the
other, and the check has to read the file that will ship.

The consequence is measurable, and we found it by running into it. Every
sequence in this paper before the 8B one is at 1.7B, whose served scales
clear the floor by $74\times$ everywhere. Two of them are not clean: one
weight of $1.4\times10^{9}$ leaves its cell at the sixth fold of each of the
two unconsolidated $r{=}64$, $s{=}40$ arms (\S\ref{sec:rehearse-old}). That
is a different event, and the floor cannot explain it---at $74\times$
clearance no 1.7B block is near the storage limit. The first 8B
sequence reported $19{,}791$ violated weights after its first fold and then
failed at the second. The cause is the storage and nothing else: with the
identical fold, identical anchors and identical fill,

\begin{center}
\begin{tabular}{lr}
\toprule
anchors stored & violations after one fold \\
\midrule
fp16 & $37{,}443$ \\
fp32 & $0$ \\
\bottomrule
\end{tabular}
\end{center}

over $6.946\times10^{9}$ weights, and all $37{,}443$ sit in blocks with
$s\le2.869\times10^{-5}$---beneath the $3.70\times10^{-5}$ floor, with none
above it. Neither the anchor extraction nor the fill is implicated: the
stored nibbles and the codes we assign disagree in zero places, no anchor
lies outside its own cell, and saturating the link without folding violates
nowhere.

Enforcing the floor---zero room for weights the storage cannot hold---costs
almost nothing, because the blocks it excludes are the ones whose cells were
negligible to begin with:

\begin{center}
\begin{tabular}{lr}
\toprule
Qwen3-8B-Base, $\mu=1\%$, fp16 anchors & \\
\midrule
weights frozen & $272{,}192$ of $6.946\times10^{9}$ ($3.9\times10^{-5}$) \\
room given up & $7.19\times10^{-9}$ of the total \\
violations after one fold & $0$ \\
violations after a second fold & $0$ \\
\bottomrule
\end{tabular}
\end{center}

Seven parts per billion of the writable room, because a block at
$s\approx6\times10^{-6}$ has cells four orders of magnitude narrower than the
median block's. The check is cheap enough to run before a sequence rather
than after it, and we do: \texttt{check\_anchor\_floor.py} reports a
candidate anchor's smallest block scale against the floor of whatever
precision the anchors will be stored at.

Two ways out exist for an anchor that fails, and both cost more than
freezing. Storing anchors in fp32 removes the floor entirely and doubles the
anchor memory---$48.7\rightarrow97$~GB at 27B, which is what makes it
unavailable at the scales where it would matter. Widening $\mu$ trades room
for clearance uniformly and does not help: the floor falls only as
$1/\mu$, so reaching the 8B minimum of $9.872\times10^{-7}$ would take
$\mu\approx38\%$, which spends three quarters of every cell to rescue
$3.9\times10^{-5}$ of the weights.

\subsubsection{Geometry of the safe region}\label{sec:geometry}
\begin{figure}[t]\centering
\includegraphics[width=\linewidth]{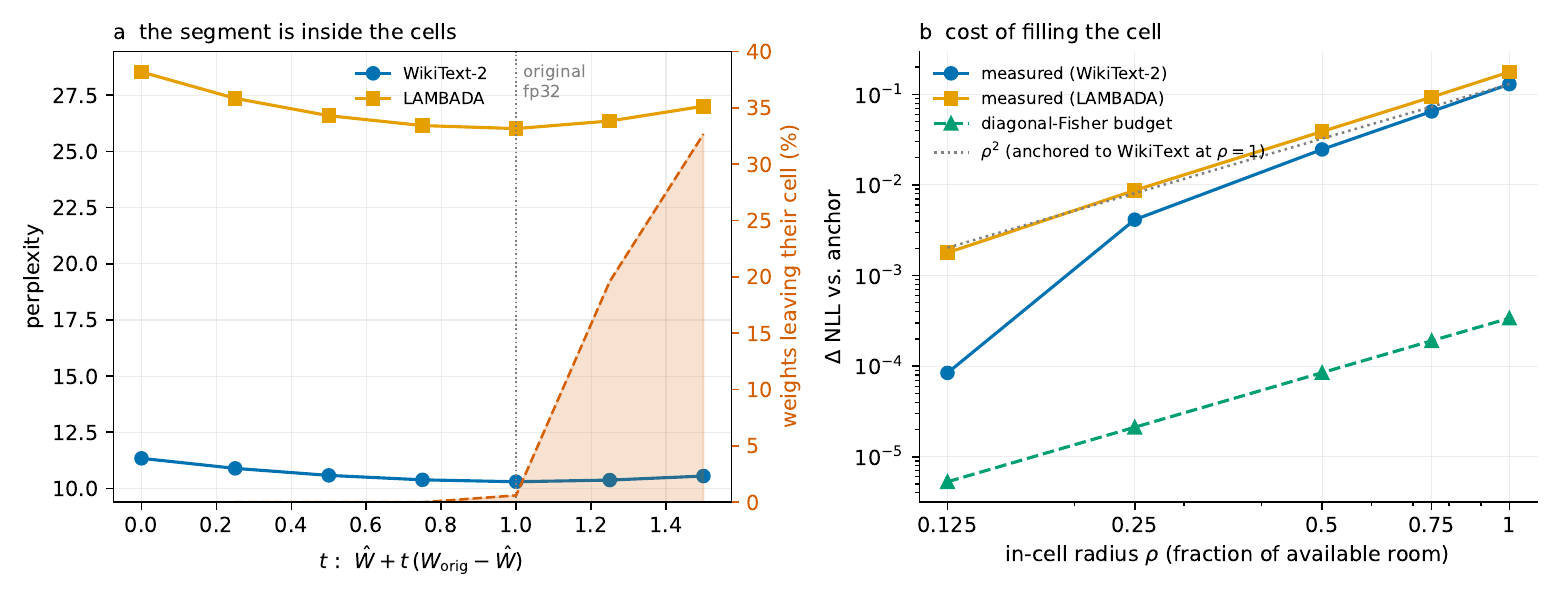}
\caption{(a) Walking from the 4-bit anchor to the released checkpoint's own
weights and past them. Those weights ship in bf16 and the walk is computed in
fp32, so nothing is lost along the segment and the endpoint is the file the
quantizer was given, not an fp32 model. Loss falls monotonically along the
whole segment while essentially no weight leaves its cell; beyond the
released weights, loss turns up and codes break catastrophically
($19.6\%$ escaped at $t{=}1.25$).
(b) Cost of random in-cell perturbation at radius $\rho$, two seeds. The
diagonal-Fisher budget tracks the wrong magnitude by up to $383\times$; the
measured exponent is 2.43 against a predicted 2.}
\label{fig:geometry}
\end{figure}

How much room is there, and what does using it cost? Two measurements
answer this directly (Fig.~\ref{fig:geometry}).

First, the segment from the anchor $\hat W$ to the original weights
$W_{\mathrm{orig}}$ lies inside the cells---by the RTN property, and
confirmed by code assignment: $0\%$ of weights escape up to $t{=}0.5$ and
$0.61\%$ at $t{=}1$ (those are original weights sitting inside the 1\%
safety margin). Along it, perplexity falls monotonically ($11.34\to10.30$
in-domain, $28.54\to26.02$ on LAMBADA): the dequantization gap is not a
loss barrier but a downhill corridor. Extrapolating past $t{=}1$ reverses
both: loss rises and escapes jump to $19.6\%$ then $32.7\%$. The safe
region and the useful region coincide, which is the geometric reason this
method works at all.

Second, filling the cells with \emph{random} content costs
$11.34\to12.9$ in-domain and $28.5\to34.1$ cross-domain at $\rho{=}1$,
but only $11.34\to11.39$ at $\rho{=}0.25$. This is not a pricing of the
radius knob, which is how it reads and is wrong: the random fill draws
$u\sim U(-1,1)$ with mean square $1/3$, a trained fill has a different
distribution the archive does not record (the logged $\mathbb{E}|t|$ of
$0.44$--$0.52$ bounds its mean square between $0.20$ and $0.52$), and damage
is quadratic in the displacement to leading order
(Proposition~\ref{prop:budget}; the measured exponent here is $2.4$). The
two should be matched on Fisher-weighted displacement, not on $\rho$ or
$\mathbb{E}|t|$. What the measurement does establish is that the cells can
be filled to a quarter of their radius with \emph{no} structure at all for
$0.05$ perplexity in-domain ($0.25$ on LAMBADA), which bounds from below how
much room the update has before anything it does becomes visible.

\textbf{The grid is not a curvature estimate, and that is what makes it
auditable.} It is natural to read the cells as more than a box: absmax scaling
gives a narrow cell to a small weight, small weights are often the sensitive
ones, so staying inside the cell might amount to following the local curvature
for free. That is a claim about a correlation and it can be measured. Taking
the diagonal Fisher on the released model and the room from its frozen codes
and scales, the rank correlation between the two over $2\times10^7$ sampled
weights is $-0.175$, and $-0.109$ on average over the 196 matrices
individually, where it ranges from $-0.317$ to $+0.133$
(\texttt{experiments/exp\_curvature\_cells.py}). The sign is the one the
reading wants---sorted into deciles by room, the narrowest tenth of cells
carries $5.3\times$ the mean Fisher of the widest, and the trend falls
$12\times$ from the first decile to the ninth before turning back---but a grid
that approximated curvature would show a strong rank correlation rather than
this, and on some matrices the sign reverses outright. Together with the
budget's $15.9$--$383\times$ underestimate above, the honest description is
that the release hands out a per-weight trust region which is free,
data-independent and weakly aligned with curvature in the favourable
direction, and which is not a curvature estimate. The independence is what
makes the guarantee checkable: a radius derived from curvature would depend on
calibration data and could not be recomputed by a third party from the shipped
file alone.

Repeating the whole measurement on Qwen3-4B replicates the structure and
sharpens one conclusion while weakening another. The segment behaves
identically---monotone descent from anchor to original weights
($8.99\to8.46$), $0.81\%$ of weights outside their cells at $t{=}1$, then
$19.7\%$ and $32.7\%$ at $t{=}1.25$ and $1.5$, within a tenth of a point of
the 1.7B escape fractions, so the geometry of the safe region appears not to
depend much on scale. What does depend on scale is the price of using it:
filling the cells at $\rho{=}1$ costs the 4B model $+6.8\%$ in-domain and
$+6.1\%$ cross-domain against $+13.6\%$ and $+19.4\%$ at 1.7B---roughly half
---which is the same direction as the 27B absorption results
(\S\ref{sec:scale}). Against that, the exponent we measure is $1.7$--$1.9$
at 4B where it was $2.43$ at 1.7B: the quadratic prediction of
Proposition~\ref{prop:budget} brackets both, but the deviation does not
replicate even in sign, and we draw no conclusion from it. The diagonal
Fisher underestimates by $278\times$ at 4B against $383\times$ at 1.7B, so
that failure is not a small-model artifact either.

\subsubsection{The constraint is not only a tax: it reduces cross-domain
forgetting}\label{sec:regularizer}
The natural null hypothesis for this entire paper is ``why not simply keep
the adapter unmerged?''\ --- serving $\hat W$ plus an unconstrained LoRA
costs nothing and preserves the artifact trivially, by never touching it.
Comparing the two is a single-intervention experiment: identical training,
identical adapter, the only difference being whether the update is projected
into the cells. On held-out LAMBADA the projected model is better in
\emph{every} run whose optimization converged---all sixteen of
Table~\ref{tab:regularizer}, spanning five model sizes from $0.6$B to $27$B,
two families and both attention mechanisms:

\begin{table}[t]\centering\small
\setlength{\tabcolsep}{4pt}
\begin{tabular}{lcccccc}
\toprule
 & recall & WikiText & LAMBADA & release & survives & worst-case \\
 & (507) & ppl & ppl $\downarrow$ & bits & merging & drift \\
\midrule
released anchor (4-bit) & $1.6\%$ & $11.71$ & $28.67$ & --- & --- & --- \\
\midrule
LoRA merged, $\eta{=}2\!\times\!10^{-4}$ & $90.9\pm2.9$ & $11.26\pm0.01$ & $\mathbf{35.5\pm1.8}$ & changed & n/a & unbounded \\
LoRA merged, $\eta{=}10^{-3}$ & $0.0\pm0.0$ & $>10^{3}$ & $>10^{4}$ & changed & n/a & unbounded \\
LoRA served unmerged & $\mathbf{95.5\pm1.5}$ & $11.55\pm0.03$ & $44.1\pm3.3$ & unchanged & \textbf{no} & unbounded \\
\midrule
CellFill, $\eta{=}10^{-3}$ & $93.3\pm2.6$ & $11.89\pm0.08$ & $43.6\pm2.6$ & \textbf{unchanged} & \textbf{yes} & \textbf{known before training} \\
\bottomrule
\end{tabular}
\caption{The null hypothesis, priced. One corpus (1{,}164 facts), one base
(Qwen3-1.7B-Base), rank $64$, $24$ epochs, $10\%$ replay, each method at its
own best learning rate, all scored on the same $507$ probes. Every row except
the divergent one is three seeds, mean$\pm$sd, and every row was produced by
one scoring function---the arms were re-run for this table rather than
collected from the archive, because the archive's arms were scored by two
different versions of it. The $\eta{=}10^{-3}$ LoRA row is the one exception
and predates the re-run; reproduced since under today's code it diverges
again on three seeds of three, at two to three and a half orders of
magnitude above the anchor against the archived four---the direction is
not a quantity a scoring change moves, and both runs are reported. Read the first three columns and no method wins them
all: recall separates no pair of the three arms at this sample size---the
largest gap, merged against unmerged LoRA, is $t{=}2.5$ on four degrees of
freedom---LoRA merged
gives up the least cross-domain ability, and the unmerged adapter---the other
configuration that leaves the released bits alone---gives up as much as
CellFill does. Read the last three and the columns are about a different
thing. ``Survives merging'' is the operational difference: an in-cell update
can be merged into the weights and re-quantization still returns the released
codes, so the serving stack is not asked to promise that it will never merge;
an adapter merged is a different file, which is what the second row shows
happening at the learning rate the first row cannot use.}
\label{tab:vslora}
\end{table}

The comparison the null hypothesis asks for is in Table~\ref{tab:vslora}, and
we report it the way it came out. Recall does not separate the three arms:
CellFill $93.3\pm2.6$, a merged LoRA $90.9\pm2.9$, an adapter served unmerged
$95.5\pm1.5$, three seeds each over the same $507$ probes. What does separate
them is cross-domain cost, and there CellFill is behind the merged adapter by
$8.2$ points of LAMBADA perplexity ($43.6\pm2.6$ against $35.5\pm1.8$,
$t{=}4.5$ on four degrees of freedom). That is the honest disadvantage in this
table and we do not qualify it.

What it is a disadvantage \emph{of} is the next question, and the third row
answers it. Serving the adapter unmerged is the other way to leave the
released bits untouched, and it costs $44.1\pm3.3$---indistinguishable from
CellFill's $43.6\pm2.6$ ($t{=}-0.2$). The two configurations that preserve the
released file pay the same cross-domain price; the one that does not preserve
it pays less. Inside LoRA the same trade appears as an internal one: merging
the adapter moves recall from $95.5$ to $90.9$ and LAMBADA from $44.1$ to
$35.5$, so the merge degrades the update and recovers cross-domain ability in
one motion. What this table prices is therefore the cost of writing this much
into a served file, and not a tax charged by the cell bound.

Two retractions belong here rather than in a footnote. An earlier draft
described the trade as equal recall bought with cross-domain cost; that was an
artifact of scoring the arms on different probe sets---$580$ for LoRA, $507$
for CellFill---and restricting both to the smaller set removed it. A later
draft then said that a merged LoRA dominates the first three columns, at
$84.9\%$ against $74.0\%$ recall. That was one CellFill seed, scored by a
superseded version of the scoring function, set against three LoRA seeds. With
three seeds on every arm and one scorer, the recall gap is gone and its sign
is reversed, though not significantly (\S\ref{sec:retracted}).

Two facts sit beside that one. The first is the row below it: at
$\eta{=}10^{-3}$, the learning rate CellFill uses on this corpus, LoRA
diverges on three seeds out of three in the archived run (four orders of
magnitude above the anchor) and on three of three again under today's
code (two to three and a half orders). Its advantage therefore lives inside a window that has to be
found by tuning, and the same shape appears again in a different
parameterization at a different scale---at 8B, a raw low-rank fill with the
cell bound removed has no setting between $10^{-5}$ and $5\times10^{-4}$ that
leaves zero weights outside their cells, the escaped fraction running from
$0.21\%$ to $11.26\%$ monotonically in the learning rate, while recall does
not rise with it---the sweep's best sits at $2\times10^{-4}$ and falls beyond,
so the rate keeps buying escape without buying recall
(\S\ref{sec:loop8b}). Two independent instances of one trade.

The second is what the last three columns are for, and it is the answer to the
null hypothesis rather than a consolation for the first three. Serving an
adapter unmerged leaves the released bits alone, as ours does, and on this
corpus it recalls more. What it cannot do is survive being merged. The moment
an operator folds it into the weights---to remove a per-token matmul, to ship
a single file, to hand a fleet one artifact instead of two---the codes change
and the released model is gone. An in-cell update is merged into the weights
by construction and the file that results re-quantizes to the codes it
shipped with, checked on every constrained weight. The $12.8$ LAMBADA points
are the price of not having to promise that no one will ever merge.

\begin{table}[t]\centering\small
\setlength{\tabcolsep}{4pt}
\resizebox{\linewidth}{!}{
\begin{tabular}{llrrrr}
\toprule
run & model & anchor & unmerged & projected & $\Delta$ \\
\midrule
\texttt{exp22\_100k} & Qwen3-1.7B & 28.67 & 503.03 & 46.37 & $-456.65$ \\
\texttt{exp26\_lora\_lr1e-3} & Qwen3-1.7B & 28.67 & 311.08 & 39.81 & $-271.27$ \\
\texttt{exp17\_aplus10k\_xdom} & Qwen3-1.7B & 28.67 & 108.63 & 36.84 & $-71.79$ \\
\texttt{exp34\_real\_aplus} & Qwen3-1.7B\,\footnotesize(real) & 28.67 & 51.59 & 36.90 & $-14.70$ \\
\texttt{exp8\_e96\_qwen3-1.7b} & Qwen3-1.7B & 28.67 & 46.33 & 33.71 & $-12.62$ \\
\texttt{exp51\_lora4b\_lr1e-4\_s0} & Qwen3-4B-bnb-4bit\,\footnotesize(real) & 24.39 & 42.37 & 30.20 & $-12.18$ \\
\texttt{exp51\_lora\_2e4\_allaug\_s0} & Qwen3-1.7B\,\footnotesize(real) & 28.71 & 47.72 & 37.46 & $-10.26$ \\
\texttt{exp12b\_27b\_e24} & Qwen3.8-27B & 20.03 & 30.92 & 21.25 & $-9.66$ \\
\texttt{exp51\_lora\_lr2e-4\_s0} & Qwen3-1.7B\,\footnotesize(real) & 28.70 & 46.14 & 36.52 & $-9.62$ \\
\texttt{exp51\_lora\_lr2e-4\_s2} & Qwen3-1.7B\,\footnotesize(real) & 28.70 & 44.17 & 35.25 & $-8.91$ \\
\texttt{exp51\_lora\_2e4\_allaug\_s2} & Qwen3-1.7B\,\footnotesize(real) & 28.71 & 43.38 & 34.87 & $-8.51$ \\
\texttt{exp51\_lora\_lr2e-4\_s1} & Qwen3-1.7B\,\footnotesize(real) & 28.70 & 42.27 & 34.61 & $-7.65$ \\
\texttt{exp51\_lora\_2e4\_allaug\_s1} & Qwen3-1.7B\,\footnotesize(real) & 28.71 & 41.30 & 34.04 & $-7.25$ \\
\texttt{exp20\_aplus\_0p6} & Qwen3-0.6B & 38.99 & 47.38 & 42.74 & $-4.64$ \\
\texttt{exp51\_lora4b\_lr5e-5\_s0} & Qwen3-4B-bnb-4bit\,\footnotesize(real) & 24.39 & 31.07 & 27.78 & $-3.29$ \\
\texttt{exp25\_mistral\_a\_lr5e-5} & Mistral-7B-v0.3 & 14.00 & 17.66 & 15.58 & $-2.09$ \\
\texttt{exp1c\_s1} & Qwen3-1.7B & 28.67 & 33.90 & 32.00 & $-1.90$ \\
\texttt{exp1c\_s2} & Qwen3-1.7B & 28.67 & 32.75 & 31.47 & $-1.28$ \\
\texttt{exp12\_27b\_e8} & Qwen3.8-27B & 20.03 & 22.29 & 21.32 & $-0.97$ \\
\midrule
\multicolumn{6}{l}{\footnotesize\emph{unmerged optimization already diverged ($>100\times$ anchor):}} \\
\texttt{exp23\_mistral\_a} & Mistral-7B-v0.3 & 14.00 & 7512.28 & 3165.85 & $-4346.43$ \\
\texttt{exp51\_lora\_lr1e-3\_s0} & Qwen3-1.7B\,\footnotesize(real) & 28.70 & 6168.10 & 6217.47 & $+49.36$ \\
\texttt{exp51\_lora\_lr1e-3\_s2} & Qwen3-1.7B\,\footnotesize(real) & 28.70 & 7115.55 & 7446.50 & $+330.95$ \\
\texttt{exp51\_lora4b\_lr2e-4\_s0} & Qwen3-4B-bnb-4bit\,\footnotesize(real) & 24.39 & 9875.85 & 11307.56 & $+1431.70$ \\
\texttt{exp51\_lora\_lr1e-3\_s1} & Qwen3-1.7B\,\footnotesize(real) & 28.70 & 8547.63 & 16899.51 & $+8351.88$ \\
\texttt{exp26\_lora\_lr3e-3} & Qwen3-1.7B & 28.67 & 6420.07 & 115458.72 & $+109038.65$ \\
\texttt{exp51\_lora4b\_lr1e-3\_s0} & Qwen3-4B-bnb-4bit\,\footnotesize(real) & 24.39 & 7093.20 & 2083988.78 & $+2076895.58$ \\
\bottomrule
\end{tabular}
}
\caption{The paper's answer to its own strongest objection---why not simply
serve the adapter unmerged? Every archived run that recorded LAMBADA at both
stages, enumerated from the archive rather than selected: the same training,
the same adapter, the only difference being whether the update is projected
into the cells. Above the rule are the sixteen runs whose optimization
converged, spanning five model sizes, two families and both attention
mechanisms; projection is better on every one. Below it are the seven whose
unmerged optimization had already exceeded $100\times$ its anchor, where
projection is a trust region rather than a repair and helps only once.}
\label{tab:regularizer}
\end{table}

The box acts as a per-weight trust region whose radii the quantization grid
supplies for free, and clipping pulls the update back toward the anchor
along exactly the coordinates that moved farthest. The effect is largest
where drift is largest: $-456.7$ perplexity points on the $10^5$-fact run,
whose unmerged adapter reaches LAMBADA $503.0$ against $46.4$ after
projection, $-271.3$ on a rank-16 LoRA at lr $10^{-3}$ and $-71.8$ on the
$10^4$-fact A+ run, against $-0.97$ on the gentlest---so it is
self-strengthening exactly when it is most needed.
The seven rows below the rule bound the claim. All are runs whose
\emph{unmerged} optimization had already exceeded $100\times$ its anchor, so
none tests whether projection regularizes training; they test what projection
does to a diverged update, and the answer is not stable. On Mistral it
removes more than half the damage ($7512\to3166$); on the other six it makes
matters worse, by $49$ points at the mildest and by six orders of magnitude
at the worst ($7093\to2.1\times10^{6}$ on 4B at lr $10^{-3}$), because
clipping the weights of an incoherent update---$70.2\%$ of them on plain LoRA
at $3\times10^{-3}$---leaves an update that is incoherent \emph{and}
truncated, retaining $18.2\%$ of its norm. Projection is a trust
region, not a repair: it constrains a converging optimization and cannot
rescue one that has already failed. This is the
robustness argument for the constraint that survives cross-domain
scrutiny---and unlike the in-domain perplexity gains of
\S\ref{sec:rehearsal}, it is attributable to the constraint itself, since
projection is the only variable that differs. The in-domain metric agrees at
a second operating point: with facts-only training and no rehearsal at all,
the unmerged adapter degrades WikiText from $11.71$ to $24.60$ while the
clipped merge scores $22.59$, so what forgetting there is belongs to the
training distribution and not to the constraint, which reduces it. Kernel and
dtype numerics are excluded throughout by an anchor-rebuild control
(bitsandbytes 4-bit $11.727$ against an fp32 rebuild of the same anchors at
$11.710$).

\subsubsection{Post-hoc projection vs.\ constraint-aware training}\label{sec:posthoc}
The clip rate diagnoses the gap between $P_{\mathcal C}(\arg\min f)$ and
$\arg\min_{\mathcal C} f$ (Prop.~\ref{prop:proj}). At 8 epochs the LoRA
delta clips at $0.14\%$ and projection is free; at 24 epochs it clips at
$2.4$--$4.6\%$, costing 3.7 recall points without rehearsal and
$1.9$--$2.4$ with it; at $10^4$ facts it clips at $30.4\%$
and path A collapses to 4.4\% recall---while four epochs of projected
healing recover 19.1\%, \emph{above} the unconstrained rank-16 adapter
itself (14.7\%): dense in-cell degrees of freedom compensate the low-rank
bottleneck. Halving $\rho$ costs only $13\%$ of healed recall ($31.5\to27.3\%$,
$n{=}1$), which is what makes $\rho$ a usable dial; how the damage itself
scales is measured directly in \S\ref{sec:geometry}, where halving $\rho$
shrinks measured $\Delta$NLL by a factor $0.19$---an exponent of $2.4$,
not $2$.
Two further ablations locate the binding constraint. Quadrupling exposures
(96 epochs, rank 16) moves healed recall only $37.9\!\to\!40.5\%$ while the
clip rate quadruples: more optimization time is not what is missing.
Rank is. At matched rehearsal, raising CellFill's rank from 16 to 64 lifts
recall from $44.4\!\pm\!6.8\%$ to $66.7\!\pm\!8.6\%$ (three seeds each) while
cross-domain perplexity rises from $33.40\!\pm\!0.59$ to
$42.33\!\pm\!1.99$---so rank buys knowledge, and buys it at a worsening
exchange rate ($1125\!\pm\!205$ down to $583\!\pm\!59$ bits per point).
The rank-16 point is a single seed against a rank 64 spread of $8.6$ recall
points, so the size of the rank effect is not resolved here; its direction
is. Capacity in this regime
is rank-limited rather than time-limited, and the rank knob is a position on
the efficiency curve rather than a free improvement.

We report this ablation twice because the first version of it was wrong in
an instructive way. The rank-16 and rank 64 runs were originally compared
across different rehearsal fractions ($0.3$ vs.\ $0.1$), which inflated the
apparent rank effect from $+29.8$ to $+41.1$ recall points---and the
paper's own \S\ref{sec:rehearsal} supplies the evidence that rehearsal
fraction matters that much. Rerunning rank 16 at matched rehearsal also
moved its cross-domain perplexity from $45.81$ to $33.06$, i.e.\ the
confound was distorting both axes at once. Comparisons in
Table~\ref{tab:frontier} now carry their rehearsal fraction explicitly.

\subsection{Write: what one minor version puts into a released file}\label{res:write}
\subsubsection{Real knowledge, verified absent}\label{sec:realcorpus}
Everything above injects synthetic biographies. They buy a guarantee no real
corpus can: novelty is certain because the facts are generated. They also
invite the reading that the result is a property of templated text, so we
repeat the experiment on real knowledge and establish the same guarantee by
measurement instead.

\textbf{Absence, not recency.} We assembled 291 dated facts from eight
domains: FDA novel drug approvals (108), Powerball draws (99), Rocket Lab
launches (35), Formula One races (28), the Abel Prize and 2026 Fields Medals
(11), and smaller sets from science, film and news. Each is probed
zero-shot on the released model before any training. It is tempting to call
these post-cutoff facts, and an earlier version of this section did: FDA
approvals from 2024--2026 score $0/324$ while household drugs
(Tylenol/acetaminophen, Ozempic/semaglutide) score $53\%$, which looks like a
boundary in time. It is not one. Rocket Lab missions score $0/35$ after 2024
and $0/14$ before it, back to 2018; Abel Prize recipients score $0/5$ back to
2016, Andrew Wiles included. What separates the controls that pass from those
that fail is fame, not recency---a 1.7B base model holds very little of this
kind of fact at any date. So we claim only what we measure: these facts are
absent from the model, verified per domain with a positive control from that
same domain, because a zero with no control is indistinguishable from a
broken prompt format.

\textbf{Every probe is a generalization test.} Two of the three synthetic
probes are near-verbatim continuations of their training sentence
(\S\ref{sec:setup}); none of the real ones is. A check over the merged corpus
asserts that no training sentence begins with a probe prompt---it caught 52
that did, in the award and mission templates, which would have manufactured a
domain effect out of probe construction alone.

\begin{center}\small
\begin{tabular}{lrrrr}
\toprule
domain & probes & guess floor & base model & after injection \\
\midrule
medicine (FDA) & 324 & 0.3\% & 0.0\% & 94.1\% \\
lottery (Powerball) & 99 & 1.0\% & 0.0\% & 100.0\% \\
technology (launches) & 35 & 5.9\% & 0.0\% & 100.0\% \\
sports (F1) & 28 & 20.0\% & 21.4\% & 100.0\% \\
awards (Abel/Fields) & 11 & 12.5\% & 18.2\% & 72.7\% \\
science & 4 & 25.0\% & 0.0\% & 100.0\% \\
film (Oscars) & 4 & 25.0\% & 0.0\% & 100.0\% \\
news & 2 & 50.0\% & 0.0\% & 100.0\% \\
\midrule
all & 507 & --- & 1.6\% & 95.7\% \\
\bottomrule
\end{tabular}
\end{center}

Recall goes from $1.6\%$ to $95.7\%$ across 507 probes, at a cost of
$28.67\to48.29$ LAMBADA and $11.71\to11.93$ WikiText, with $1.54\%$ of the
available cell space used and the artifact bit-identical, in $11.2$ minutes.
The floors are reported per domain because they span two orders of magnitude:
28 Formula One races share five winners, so the base model's $21.4\%$ there is
guessing and only the injected $100\%$ is not. The four smallest domains are
listed for coverage and their rates should not be quoted.

\textbf{The Powerball domain is the strictest test in this paper, and it
recovers the bit accounting the rest of the real corpus gives up.} A draw is
one of $\binom{69}{5}\times26=292{,}201{,}338$ equally likely outcomes, so each
carries exactly $28.12$ bits and the 99 draws are $2{,}784$ bits of
incompressible information. Nothing semantic helps: no prior makes
$10\text{--}21\text{--}58\text{--}61\text{--}64$ more plausible than any other
draw. All 99 are recalled exactly.

That result nearly did not survive its own measurement. Under a fixed
16-token generation budget the domain read $0/99$, and the natural reading
---that semantically arbitrary knowledge is qualitatively harder---was a
finding we were prepared to report. The model was in fact emitting
\texttt{\ldots\ Powerball 1} where \texttt{\ldots\ Powerball 11} was
expected: the budget cut the last token off a 25-token answer and prefix
matching scored the truncation as a miss. The budget is now taken from the
longest answer present. Exact-match recall cannot distinguish a near miss
from a failure, and the same trap is available to anyone scoring long-form
recall against a constant.

\textbf{On a published artifact.} Every anchor so far was a file we
quantized ourselves. The claim is about files other people ship, so we
repeat the injection on unsloth's published NF4 releases of Qwen3-1.7B, 4B
and 8B, loaded exactly as published---codes, block scales and double
quantization untouched---and checked against the originals with
\texttt{experiments/verify\_official\_anchor.py}. The first attempt scored
$88.8\%$ on the 507 probes against $95.7\%$ on our own anchor ($78.6\%$
over all 580 probes, composition included, which is the number in our
log), and the gap was not the anchor: padded the same way, the two models'
first-batch losses agree to $0.001$ nats ($3.531$ against $3.532$). The
release's tokenizer is configured for left padding, which puts a label---the
sentence's first word---on the final pad token, so every training sequence
carries a spurious ``predict the opening word from padding'' term ($3.53$
nats right-padded against $3.87$ left-padded, same release, same batch).
With training padded on the right regardless, the published 1.7B anchor
reaches $94.3\pm2.2\%$ on the 507 probes over three seeds ($96.8\%$ for
the seed that the other anchors were run with), the 4B release
$96.9\pm0.4\%$ over two and the 8B release $97.6\%$ on its first seed
($96.3\pm2.1\%$ over the seeds run), each from a base below $2\%$
(Table~\ref{tab:official} in \S\ref{sec:scale}). We report the detour
because it is the kind of defect that produces a real-looking effect---a
published artifact ``harder to write into'' than our own---with nothing
behind it but a tokenizer flag.

\subsubsection{Editing methods: verbatim recall does not generalize}\label{sec:baselines}
Recall in both editing baselines depended on the probe having been supplied
as an edit, GRACE~\cite{grace} falling from $91.1\%$ to $2.0\%$ and
MEMIT~\cite{memit} from $22.7\%$ to $9.5\%$ once the edits were restricted to
the 108 ingredient probes and all 507 were scored
(Table~\ref{tab:baselines}). CellFill $r{=}64$, trained on declarative
sentences and scored on the same 507 probes, reached $95.7\%$. GRACE holds
its edits in a codebook that a served call must consult, whereas the served
CellFill model pays a forward pass and nothing else. CellFill paid for its
recall in cross-domain perplexity instead, and the two costs must be read
against different anchors: GRACE runs on the unquantized model, whose LAMBADA
is $26.13$, and leaves that number untouched to sixteen digits because it
never writes to a weight; CellFill runs on the 4-bit release, whose LAMBADA
is $28.67$, and takes it to $48.29$. The comparable quantities are therefore
$+0.0$ and $+19.6$ points over each method's own starting model, not
$26.13$ against $48.29$.

The two editors ran on the same facts and were scored with the same ruler.
MEMIT solves for a closed-form update to a band of MLP matrices; GRACE, the
closest existing analogue on revocability, keys a codebook of adapter values
on the edit prompts and consults it at inference. The comparison is
asymmetric by construction and the asymmetry favours the editors. They take
(prompt, target) pairs, so under ``verbatim'' each is handed every probe
prompt as an edit (one edit per probe, $507$ of $507$, which is training on
the test set) and scored on exactly what it was given. Under ``generalize''
only the 108 ingredient probes are edited and all 507 are scored, which is
the condition CellFill is always in. MEMIT is a batch method and is run as one: applied
sequentially over 108 edits it destroys the model (WikiText $20{,}348$,
archived as the record of the misuse). The EasyEdit~\cite{easyedit}
harness restores the original weights after scoring by default; our first
batch run reported the untouched model's recall and perplexity after two
hours of solving for deltas it discarded, and that record is archived too.

\begin{table}[t]\centering\small
\setlength{\tabcolsep}{4pt}
\resizebox{\linewidth}{!}{
\begin{tabular}{lrrrll}
\toprule
method & recall & WikiText-2 & LAMBADA ppl$\downarrow$ & artifact & at inference \\
\midrule
MEMIT, verbatim & 22.7\% & 10.74 & 26.48 & no & no \\
MEMIT, generalize & 9.5\% & 10.69 & 26.23 & no & no \\
GRACE, verbatim & 91.1\% & 10.66 & 26.13 & yes & codebook \\
GRACE, generalize & 2.0\% & 10.66 & 26.13 & yes & codebook \\
CellFill $r{=}64$, 4 phrasings & 95.7\% & 11.93 & 48.29 & yes & no \\
\bottomrule
\end{tabular}
}
\caption{Editing baselines on the real corpus, Qwen3-1.7B-Base. ``artifact''
is whether the released 4-bit file survives the update; ``at inference'' is
what a served call pays beyond a forward pass. Perplexities are the
repository's own measurement on the same corpora as every other table.}
\label{tab:baselines}
\end{table}

The aggregate hides the shape of the result, so Table~\ref{tab:baseldom}
breaks GRACE's strongest condition out by domain against CellFill's
ordinary one.

\begin{table}[t]\centering\small
\setlength{\tabcolsep}{3pt}
\resizebox{\linewidth}{!}{
\begin{tabular}{lrrrrrrrr}
\toprule
method & awards & film & lottery & medicine & news & science & sports & technology \\
\midrule
GRACE, one edit per probe (train-on-test) & 18\% & 100\% & 100\% & 100\% & 0\% & 0\% & 79\% & 34\% \\
CellFill $r{=}64$, declarative corpus, paraphrased probes & 73\% & 100\% & 100\% & 94\% & 100\% & 100\% & 100\% & 100\% \\
\bottomrule
\end{tabular}
}
\caption{Per-domain recall: GRACE handed every probe as an edit, against
CellFill trained on declarative sentences and probed on paraphrases and
reversals.}
\label{tab:baseldom}
\end{table}

GRACE with every probe as an edit reaches $91.1\%$ with the weights
untouched---its WikiText is the released model's to sixteen digits---and
still scores $0\%$ on two domains and $18\%$ on a third even though every
one of those probes was handed to it as an edit: GRACE keys its codebook on
the edit prompt, prompts that share a template land close together, and a
later edit inside an existing key's ball displaces the earlier one. Under
the honest condition its recall is the untouched model's to the last digit,
$2.0\%$: the codebook never fires on a prompt it was not keyed on. CellFill,
in its honest condition, beats the train-on-test GRACE on five of the eight
domains, ties it at $100\%$ on two, and trails it only on medicine ($94\%$
against $100\%$). The contract axes are where the methods differ in kind
rather than degree: GRACE keeps the release intact by adding a structure at
inference; CellFill keeps it intact and adds nothing, and what it writes is
a bounded displacement that a later writer can add to
(\S\ref{sec:fusion}); whether two GRACE codebooks merge is not tested here.

\textbf{Two other PEFT families leave the codes intact, and neither can
carry the corpus.} LoRA is the natural comparison because its update lives in
the same place CellFill's does---on the weight matrix---but it is the wrong
test of the contract, since an unmerged adapter preserves the artifact
trivially. The sharper question is what happens to methods whose update lives
somewhere the 4-bit codes do not reach at all. Two exist. Prefix
tuning~\cite{prefix} trains key/value states prepended at attention and never
writes to a weight; LN tuning trains only the LayerNorm parameters, which a
4-bit release leaves unquantized. Both therefore ship the released integer
codes unchanged, by a completely different route from in-cell learning. On
the 291-fact real corpus and the published NF4 Qwen3-1.7B, with the learning
rate swept per method rather than borrowed from LoRA, both cap at $1.8\%$
recall---LN at $3\times10^{-3}$ with $57{,}344$ trainable parameters and
WikiText $11.69$, prefix at $10^{-2}$ with $3.67$M trainable and WikiText
$12.13$---and the cap holds across a $30\times$ range of rates, so it is a
property of where the update lives and not of tuning. CellFill on the same
corpus and anchor reaches $94.3\pm2.2\%$. Preserving the codes is therefore
not by itself the achievement; preserving them \emph{and} writing a corpus
into the weights is. (IA3~\cite{ia3} belongs in this comparison and is
reported as not run: it raises on a dtype mismatch inside the fused attention
kernel, and still raises on the eager path.)

\textbf{LoRA recalled the facts only at a lower learning rate.} We trained
a LoRA adapter of the same rank on the same corpus, with the same
rehearsal and three seeds per configuration, at $2\times10^{-4}$ (its
default) and at the $10^{-3}$ CellFill uses throughout. The CellFill
arm is on unsloth's published NF4 release; the LoRA arm is on our own
bitsandbytes quantization of the same model, which agrees with the release
to $0.001$ nats on the first batch (\S\ref{sec:realcorpus}).

\begin{table}[t]\centering\small
\setlength{\tabcolsep}{4pt}
\resizebox{\linewidth}{!}{
\begin{tabular}{lrrrrl}
\toprule
method & recall & \multicolumn{2}{c}{WikiText} & LAMBADA ppl$\downarrow$ &artifact \\
 & served & unmerged & served & served & \\
\midrule
LoRA $r{=}64$, $\eta{=}2\times10^{-4}$ & $90.9\!\pm\!2.9$\% & 11.55 & 11.26 & 35.46 & 8.2\% of weights clipped \\
LoRA $r{=}64$, $\eta{=}10^{-3}$ & $0.0\!\pm\!0.0$\% & 796.48 & 4,084.89 & 10,187.83 & 61.7\% of weights clipped \\
CellFill $r{=}64$, $\eta{=}10^{-3}$ & $94.3\!\pm\!2.2$\% & --- & 11.91 & 43.22 & bit-identical \\
\quad 4B: LoRA $r{=}64$, $\eta{=}10^{-3}$ & 0.0\% & 789.02 & 3,217,228.00 & 2,083,988.78 & 54.6\% of weights clipped \\
\quad 4B: LoRA $r{=}64$, $\eta{=}2\times10^{-4}$ & 0.0\% & 711.59 & 4,943.06 & 11,307.56 & 27.0\% of weights clipped \\
\quad 4B: LoRA $r{=}64$, $\eta{=}10^{-4}$ & 97.6\% & 9.38 & 9.17 & 30.20 & 8.0\% of weights clipped \\
\quad 4B: LoRA $r{=}64$, $\eta{=}5\times10^{-5}$ & 94.9\% & 8.53 & 8.52 & 27.78 & 1.2\% of weights clipped \\
\quad 4B: CellFill $r{=}64$, $\eta{=}10^{-3}$ & $96.9\!\pm\!0.4$\% & --- & 9.93 & 42.34 & bit-identical \\
\bottomrule
\end{tabular}
}
\caption{LoRA against CellFill on the 291-fact real corpus and Qwen3-1.7B
in NF4: CellFill on unsloth's published release, LoRA on our own
bitsandbytes quantization of the same model ($r{=}64$, $\alpha{=}128$,
rehearsal $0.1$, 24 epochs, three seeds, recall over the 507
non-composition probes). ``unmerged'' is the adapter's own perplexity
served on top of the 4-bit base; ``served'' is the model actually shipped.
LoRA's served model is the adapter folded into the release and projected
back into the cells, which is what the clipped fraction counts.}
\label{tab:lorareal}
\end{table}

At $\eta=10^{-3}$, the rate CellFill uses throughout, CellFill recalled
$94.3\pm2.2\%$ of the 507 probes and left the released 4-bit file
bit-identical, while LoRA at the same rate and rank recalled
$0.0\pm0.0\%$ (Table~\ref{tab:lorareal}). LoRA's served model gave a
WikiText perplexity of $4{,}084.89$ and a LAMBADA perplexity of
$10{,}187.83$, and folding the adapter into the release clipped $61.7\%$
of the constrained weights back into their cells. LoRA reached
$84.9\pm3.4\%$ only at $2\times10^{-4}$, and at that rate the merge still
clipped $8.5\%$ of the constrained weights, so what is served is a clipped
version of the adapter rather than the release, and every evaluation of
that release has to be repeated.

With the same corpus, rate and rank, LoRA did not merely do worse: it
stopped being a language model. This is
the sharpest statement the paper can make about \emph{why} a per-weight
trust region is worth having, and it is not about invariance at all: the
quantization grid hands out a bounded, per-weight region for free, and
optimizing inside it converts the failure mode of aggressive updating from a
cliff into a gradient. A method that cannot diverge can be run at learning
rates that a method that can diverge cannot. Whether that alone accounts
for the gap at $2\times10^{-4}$ would need CellFill at that rate on this
corpus, which we have not run; what the table shows is that the rate
CellFill uses is one LoRA cannot survive.

\subsubsection{Knowledge in the weights or in the prompt}\label{sec:rag}
Retrieval is the better instrument on the PopQA tail when its index is
clean, and the worse one when the index is chunked as a production store
chunks it; the fill does not move between the two conditions, which are
identical in their questions. Retrieval-augmented generation leaves the
release untouched and puts the
knowledge in the context; a wiki the model consults is the same idea with
a human index. Both preserve the artifact trivially, so the comparison is
about what each costs and where each fails, and it has to be made on facts
the model does not already have. PopQA~\cite{popqa} is 14k questions over
Wikidata triples, each tagged with its subject's Wikipedia page views. We
take the tail: subjects under 300 monthly views, the guessable relations
dropped (a team's sport is in its name; colours; ``capital of'', whose
listed answers are noisy), one fact per subject, the 2{,}000 most obscure.
Each fact is stated in three template sentences with the object written as
PopQA writes it, and \emph{these sentences are the documents}: the
retriever indexes exactly the text the fill was trained on, so neither
side sees a different corpus. The probe is PopQA's own question in a fixed
four-shot format, scored by PopQA's rule (an alias appears in the answer
line; whole-word, aliases of at least three characters). The index is the
6{,}000 sentences plus 200{,}000 paragraphs of WikiText-103, which is what
an index of anything real looks like; a second condition (\emph{chunks})
puts each sentence in the middle of one of those paragraphs, as a chunked
document store would. Retrieval is BM25 and a dense encoder
(bge-small-en-v1.5); the oracle arm hands the model the sentences that
state the answer; the fill is the released model with the facts written
into its cells; the released model alone is the floor, and the probes it
misses---$84\%$ of them---are the realistic case, reported separately
(\texttt{experiments/build\_popqa.py}, \texttt{experiments/exp\_rag.py}).

\begin{table}[t]\centering\small
\setlength{\tabcolsep}{3.5pt}
\begin{tabular}{lllrrrrr}
\toprule
model & index & arm & recall & on unknown & answer in context & tokens/probe & s/probe \\
\midrule
Qwen3-1.7B & sentences & none & 15.7\% & 0.0\% & 0.0\% & 75 & 0.092 \\
Qwen3-1.7B & sentences & oracle & 98.4\% & 98.1\% & 100.0\% & 226 & 0.108 \\
Qwen3-1.7B & sentences & bm25@1 & 97.1\% & 96.6\% & 98.0\% & 92 & 0.094 \\
Qwen3-1.7B & sentences & bm25@5 & 97.7\% & 97.3\% & 99.4\% & 397 & 0.132 \\
Qwen3-1.7B & sentences & dense@1 & 98.4\% & 98.0\% & 99.6\% & 90 & 0.093 \\
Qwen3-1.7B & sentences & dense@5 & 98.7\% & 98.5\% & 100.0\% & 236 & 0.114 \\
Qwen3-1.7B & sentences & fill & 81.8\% & 78.5\% & 0.0\% & 75 & 0.207 \\
Qwen3-1.7B & sentences & fill+bm25@1 & 98.3\% & 98.0\% & 98.0\% & 92 & 0.214 \\
Qwen3-1.7B & sentences & fill+bm25@5 & 97.2\% & 96.9\% & 99.4\% & 397 & 0.281 \\
Qwen3-1.7B & chunks & none & 15.8\% & 0.0\% & 0.0\% & 75 & 0.173 \\
Qwen3-1.7B & chunks & oracle & 97.5\% & 97.4\% & 100.0\% & 754 & 0.239 \\
Qwen3-1.7B & chunks & bm25@1 & 63.4\% & 57.5\% & 63.7\% & 212 & 0.192 \\
Qwen3-1.7B & chunks & bm25@5 & 73.8\% & 69.9\% & 79.6\% & 827 & 0.244 \\
Qwen3-1.7B & chunks & dense@1 & 31.5\% & 23.9\% & 29.4\% & 199 & 0.189 \\
Qwen3-1.7B & chunks & dense@5 & 40.5\% & 33.4\% & 49.2\% & 686 & 0.228 \\
Qwen3-1.7B & chunks & fill & 81.7\% & 78.3\% & 0.0\% & 75 & 0.413 \\
Qwen3-1.7B & chunks & fill+bm25@1 & 78.5\% & 74.8\% & 63.7\% & 212 & 0.428 \\
Qwen3-1.7B & chunks & fill+bm25@5 & 81.4\% & 78.4\% & 79.6\% & 827 & 0.520 \\
Qwen3-4B & sentences & none & 16.4\% & 0.0\% & 0.0\% & 75 & 0.233 \\
Qwen3-4B & sentences & oracle & 98.8\% & 98.6\% & 100.0\% & 226 & 0.272 \\
Qwen3-4B & sentences & bm25@1 & 97.3\% & 96.8\% & 98.0\% & 92 & 0.237 \\
Qwen3-4B & sentences & bm25@5 & 98.3\% & 98.1\% & 99.4\% & 397 & 0.328 \\
Qwen3-4B & sentences & dense@1 & 98.9\% & 98.7\% & 99.6\% & 90 & 0.239 \\
Qwen3-4B & sentences & dense@5 & 98.9\% & 98.9\% & 100.0\% & 236 & 0.283 \\
Qwen3-4B & sentences & fill & 89.8\% & 87.9\% & 0.0\% & 75 & 0.721 \\
Qwen3-4B & sentences & fill+bm25@1 & 98.7\% & 98.4\% & 98.0\% & 92 & 0.733 \\
Qwen3-4B & sentences & fill+bm25@5 & 99.0\% & 98.7\% & 99.4\% & 397 & 0.925 \\
Qwen3-4B & chunks & none & 16.4\% & 0.0\% & 0.0\% & 75 & 0.237 \\
Qwen3-4B & chunks & oracle & 98.0\% & 97.7\% & 100.0\% & 754 & 0.382 \\
Qwen3-4B & chunks & bm25@1 & 63.2\% & 57.4\% & 63.7\% & 212 & 0.273 \\
Qwen3-4B & chunks & bm25@5 & 74.9\% & 70.7\% & 79.6\% & 827 & 0.390 \\
Qwen3-4B & chunks & dense@1 & 30.0\% & 23.6\% & 29.4\% & 199 & 0.263 \\
Qwen3-4B & chunks & dense@5 & 41.6\% & 35.2\% & 49.2\% & 686 & 0.354 \\
Qwen3-4B & chunks & fill & 89.8\% & 87.9\% & 0.0\% & 75 & 0.712 \\
Qwen3-4B & chunks & fill+bm25@1 & 85.4\% & 82.6\% & 63.7\% & 212 & 0.787 \\
Qwen3-4B & chunks & fill+bm25@5 & 92.1\% & 90.6\% & 79.6\% & 827 & 1.060 \\
\bottomrule
\end{tabular}

\caption{The PopQA tail (2{,}000 facts the released models answer at
$16\%$), the same facts in the prompt or in the weights. ``on unknown'' is
recall on the probes the released model missed; ``answer in context'' is
the retriever's hit rate; tokens and seconds are per probe. The
\emph{chunks} index holds each fact inside a WikiText-103 paragraph.}
\label{tab:popqa}
\end{table}

With a clean sentence index, retrieval does not miss: BM25 at $k{=}1$ finds
the answer for $98\%$ of probes and the model answers $97\%$, and dense
retrieval at $k{=}5$ answers $98.7\%$ at 1.7B and $98.9\%$ at 4B. The fill
alone answers $82\%$ at 1.7B and $90\%$ at 4B with no index and a
75-token prompt, against $90$--$827$ tokens for the retrieval arms;
fill and retrieval together reach $98$--$99\%$. With the chunked index,
the retrievers miss: BM25 finds the answer in its top five for $80\%$ of
probes and the model answers $74\%$ of them, dense retrieval finds $49\%$
and answers $41\%$, and the oracle---the right chunks handed over---still
costs $754$ tokens per probe. The fill is unaffected ($82$ and $90\%$,
above every retrieval arm at both sizes); adding BM25 to it helps at 4B
($92\%$) and hurts at 1.7B, where a wrong chunk in the context pulls a
known answer off. On the repository's own corpus, with the same
200{,}000 distractors, the composition probes show the multi-hop case: a
retriever has to find both sentences, and BM25 at $k{=}3$ drops from $86\%$
of compositions without distractors to $68\%$ with them, while the injected
model with a single retrieved sentence composes $96\%$ of them---the
context supplies the first hop's key and the weights supply the second.
The honest summary is not that retrieval misses; it is that a
single-sentence index does not, a chunked one does, multi-hop does, and
the weights cost an order of magnitude fewer tokens per question and no
index at all, while combining the two is better than either where the
model is large enough to ignore a bad chunk (Fig.~\ref{fig:wvp}).

\begin{figure}[t]\centering
\includegraphics[width=\linewidth]{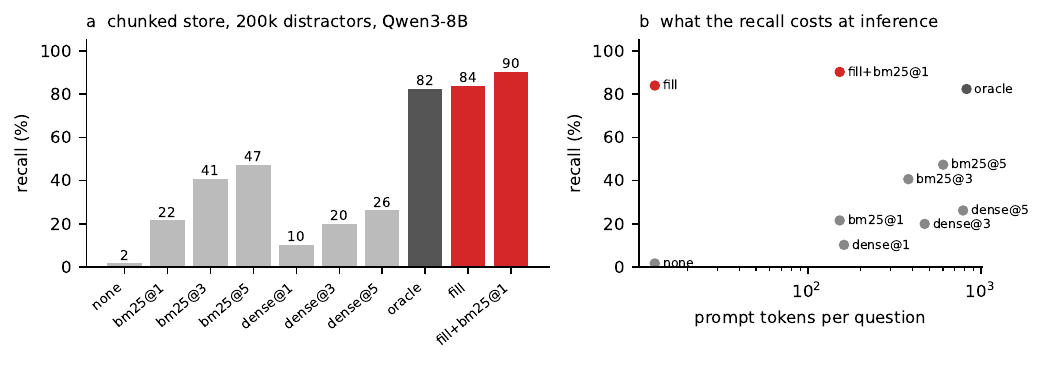}
\caption{Knowledge in the weights against knowledge in the prompt at 8B,
chunked store with 200{,}000 distractors. \textbf{a}, Recall by arm: the
fill (red) with no index beats every retrieval arm and the oracle
passage. \textbf{b}, The same arms by prompt tokens per question.}
\label{fig:wvp}
\end{figure}

\subsubsection{Where the fill is written decides what kind of knowledge arrives}
\label{sec:where}
On the real corpus, a fill restricted to the attention matrices recalled
$79.8\pm3.4\%$ of the 507 probes and answered $1.4\pm1.4\%$ of the 73
two-hop composition probes, an interval that contains zero. A fill
restricted to the MLP recalled $90.0\pm0.8\%$ and kept composition at
$13.2\pm4.8\%$; a fill free to write into every matrix reached
$94.6\pm1.9\%$ and $9.2\pm2.8\%$ (Table~\ref{tab:structure}). The choice of
target matrices therefore sets what kind of knowledge arrives, not only how
much.

That coarse three-way split says which structure to prefer but not why, and
it prices nothing per unit of budget. The sweep below asks the finer
question on synthetic facts, where novelty is certain and the partitions can
be made disjoint: we partition the constrained
matrices and let CellFill ($r{=}64$) write into one partition at a time,
freezing the rest---structural invariance means no clipping confound enters
the comparison. Normalization matters: dividing absorbed bits by
\emph{constrained} weights asks which substrate holds the most, while
dividing by \emph{trainable} parameters asks where a fixed low-rank budget
is best spent. The two disagree, so we give both, plus the price in
cross-domain perplexity.

\begin{center}\small
\begin{tabular}{lrrrrrr}
\toprule
partition & trainable & recall & own anchor & LAMBADA & bits/M$_{\text{train}}$ & bits/pt \\
\midrule
gate+up (all layers) & 29.4M & 32.1\% & 26.82 & 31.01 & \textbf{130} & \textbf{911} \\
MLP (all layers) & 44.0M & 43.7\% & 27.92 & 36.66 & 118 & 594 \\
attention (all layers) & 25.7M & 25.3\% & 26.92 & 31.19 & 117 & 705 \\
MLP layers 19-27 & 14.2M & 13.5\% & 27.48 & 34.09 & 113 & 243 \\
down\_proj (all layers) & 14.7M & 8.9\% & 27.33 & 32.36 & 72 & 210 \\
MLP layers 0-8 & 14.2M & 5.8\% & 26.45 & 27.98 & 49 & 456 \\
MLP layers 9-18 & 15.7M & 6.0\% & 26.33 & 27.90 & 46 & 455 \\
\bottomrule
\end{tabular}
\end{center}
Each partition run quantizes only its own matrices and leaves the rest of
the model at original precision, so the seven runs do \emph{not} share a
baseline; every row is therefore priced against the anchor archived with it.
(An earlier version of this analysis priced all seven against the
fully-quantized model's anchor of 28.67, which made the two smallest
partitions appear to cost nothing. They do not.)

Three findings, two of which cut against the standard picture.

\textbf{Module type barely matters per unit of budget.} gate+up (130),
full MLP (118) and attention (117) fall within 11\% of each other. Measured
per \emph{constrained} weight, attention appears $74\%$ better than MLP---
but that gap is entirely an artifact of attention's smaller matrices under
a per-matrix rank budget, and it disappears under the normalization an
implementer actually pays. We report this explicitly because the wrong
normalization here yields a confident and false headline.

\textbf{Depth dominates module type, but only one end of it is measurable.}
Late MLP layers absorb $113$ bits per million trainable parameters against
$49$ and $46$ for early and middle. We do not report that as a
$2.3\text{--}2.5\times$ ratio, because the two denominators are not
measurements: early and middle MLP recall $5.8\%$ and $6.0\%$, both at or
below the $6.5\%$ marginal-guess floor of \S\ref{sec:limitations}. What the
partition sweep establishes is one-sided --- late MLP absorbs, early and
middle did not absorb anything distinguishable from guessing under this
budget --- and a ratio computed against a floor would put a number on the
difference that the data cannot support. This runs against the localization results of
ROME and MEMIT~\cite{rome,memit}, which place factual associations in
\emph{middle} MLP layers, and against the prescription to write through
\texttt{down\_proj} (worst non-degenerate partition here, 72). We read the
discrepancy as a genuine difference of task rather than a contradiction:
that literature locates where an existing association can be \emph{found and
edited}, whereas we measure where a \emph{new} association is most cheaply
\emph{written} under a norm constraint. The middle layers are not idle---they
show the highest cell saturation of any partition (16.6\%), i.e.\ the
optimizer pushes hardest there and gets the least back.

A second, independent line agrees, and it comes from a different model and a
different quantity. Aggregating the archived per-matrix fill statistics of the
27B run (\texttt{fillstats\_27b\_all}, 496 matrices over 64 layers) by
layer, the fill energy is close to uniform in depth: the four
depth quartiles hold $23.5$, $24.7$, $24.6$ and $27.2\%$ of it, the heaviest
single layer $1.93\%$ and the lightest $1.39\%$. An unconstrained optimizer
given the whole depth does not concentrate the update anywhere; it spreads it.
Read with the partition sweep, which says any single depth band recovers at
most a third of what the full depth does, this closes off the most tempting
engineering shortcut available here---train only the upper layers and cache
activations below them---because the knowledge is not there to be trained.

\textbf{The best target is the same under both prices.} Ranking by damage
rather than by budget---bits absorbed per point of LAMBADA given up---keeps
gate+up on top ($911$) and keeps \texttt{down\_proj} at the bottom ($210$),
with attention ($705$) ahead of full MLP ($594$). Early and middle layers
are not free, as an incorrect shared-anchor calculation initially suggested;
they degrade their own anchors by $1.5$ points, which is little only because
they absorb little. A practitioner optimizing absorbed knowledge per unit of
forgetting should write to the gate and up projections across all depths---
not to \texttt{down\_proj}, and not to the middle of the network.

\textbf{On the real corpus, the composition survives in the MLP and not in attention.}
The partition sweep above used synthetic facts and measured absorption.
Table~\ref{tab:structure} repeats the coarsest split on the real corpus and
the published anchor, and adds the composition probes of
\S\ref{sec:usable}, which is the question the sweep could not ask.

\begin{table}[t]\centering\small
\setlength{\tabcolsep}{4pt}
\begin{tabular}{lrrrrr}
\toprule
fill restricted to & recall (507) & composition & WikiText & LAMBADA ppl$\downarrow$ &saturation \\
\midrule
all matrices (196) & $94.6\!\pm\!1.9$\% & $9.2\!\pm\!2.8$\% & $11.91\!\pm\!0.05$ & $45.0\!\pm\!3.9$ & 1.64\% \\
attention only (112) & $79.8\!\pm\!3.4$\% & $1.4\!\pm\!1.4$\% & $10.51\!\pm\!0.02$ & $32.0\!\pm\!0.7$ & 3.98\% \\
MLP only (84) & $90.0\!\pm\!0.8$\% & $13.2\!\pm\!4.8$\% & $10.84\!\pm\!0.05$ & $36.4\!\pm\!1.8$ & 2.11\% \\
\bottomrule
\end{tabular}

\caption{The fill restricted to one structure, real corpus, published NF4
Qwen3-1.7B, $r{=}64$. Composition is the raw rate on the 73 two-hop
probes. Cells with $\pm$ are over three seeds.}
\label{tab:structure}
\end{table}

Both restrictions cost less than writing everywhere, and attention costs the
least: LAMBADA $32.0\pm0.7$ for attention and $36.4\pm1.8$ for the MLP
against $45.0\pm3.9$ for the unrestricted fill, WikiText $10.51\pm0.02$ and
$10.84\pm0.05$ against $11.91\pm0.05$. The cheapest row is also the one
without composition; the MLP alone carries most of what the full model
carries, the composition included.
Where the fill's energy lands when every matrix is writable
(Fig.~\ref{fig:fillmap}) says the same thing from the other side: across
the eight-domain corpus and each of five single-domain corpora, the gate and
up projections take $53$--$63\%$ of the fill energy, attention $21$--$26\%$
and \texttt{down\_proj} $16$--$23\%$, and the last quarter of the depth
takes the largest share in every corpus. What differs between domains is
small next to what they share; drug approvals and lottery draws are
written into the same place.

\begin{figure}[t]\centering
\includegraphics[width=0.9\linewidth]{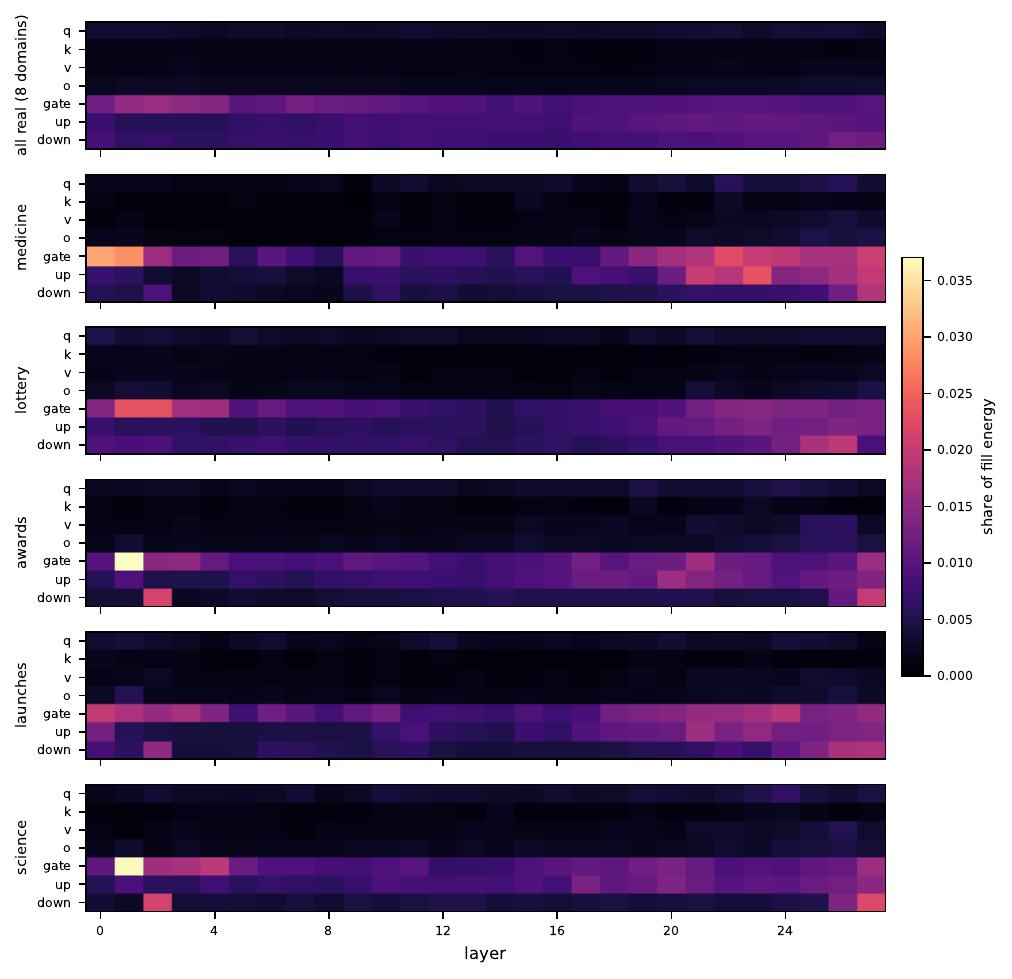}
\caption{Share of the fill's energy by layer and matrix type, one panel per
corpus (the eight-domain corpus and five single-domain corpora), Qwen3-1.7B
published anchor, $r{=}64$. Each panel is normalised to its own total.}
\label{fig:fillmap}
\end{figure}

\subsubsection{Scale, architecture, and family}\label{sec:scale}
The mechanism touches only quantized linear layers, so it should be
indifferent to what surrounds them. We test that on a $12\times$ parameter
ladder within one family, on a second family with a different tokenizer and
initialization, and on a hybrid model whose blocks are gated linear
attention rather than softmax attention.

\textbf{Published anchors.} Table~\ref{tab:official} gathers every run on a
release quantized by someone else, on the real corpus of
\S\ref{sec:realcorpus}, and Figure~\ref{fig:releases} draws its recall
column---one bar per release, including the GGUF file \texttt{llama.cpp}
serves:
the bitsandbytes NF4 releases from unsloth, Google's quantization-aware
W4A16 release of Gemma~4 (a uniform int4 grid with a scale per 32 weights,
written into with the cell arithmetic of \S\ref{sec:theory} applied to
integer cells), and RedHat's W4A16 release of Qwen3-4B, the same model as
the NF4 row on a second grid. ``base'' is the release before any update.
Two things the table shows that the NF4 rows alone could not. The cell
arithmetic carries to a uniform integer grid unchanged: the W4A16 releases
take the corpus at $90.5$ and $93.9\%$ against $94.3$ and $96.9\%$ for the
NF4 releases at the same two sizes, and RedHat's W4A16 Qwen3-1.7B is a worse
quantization to begin with (WikiText $19.5$ against NF4's $11.7$), which the
fill partly repairs. We do \emph{not} read the gap between the two grids as
the grid's effect. The releases do not share a parent: comparing the tensors
neither format quantizes, the unsloth NF4 releases reproduce
\texttt{Qwen3-1.7B-Base} exactly (embedding and all $113$ normalization
tensors identical to the last bit), while RedHat's differ from them by $17$
to $28\%$ of the embedding's range, so they are quantizations of the
post-trained checkpoints rather than of the base ones. Grid and post-training
are confounded in that comparison and only a matched pair would separate
them. And an
instruction-tuned release behaves differently as an anchor: Gemma-4-E2B's
QAT release, evaluated as a plain language model, has a WikiText perplexity
in the hundreds that the fill brings to $17$, so its perplexity ratios sit
below one and mean something else---training on declarative text teaches a
chat-tuned model to continue plain text---and the recall column is the one
to read ($82.9\pm1.0\%$ across two hosts). The 27B rerun with the signed-scale cells completed at $96.4\%$
on the 507 probes with the stored code returned on all $2.4\times10^{10}$
constrained weights; its row is in the table.

\begin{figure}[t]\centering
\includegraphics[width=\linewidth]{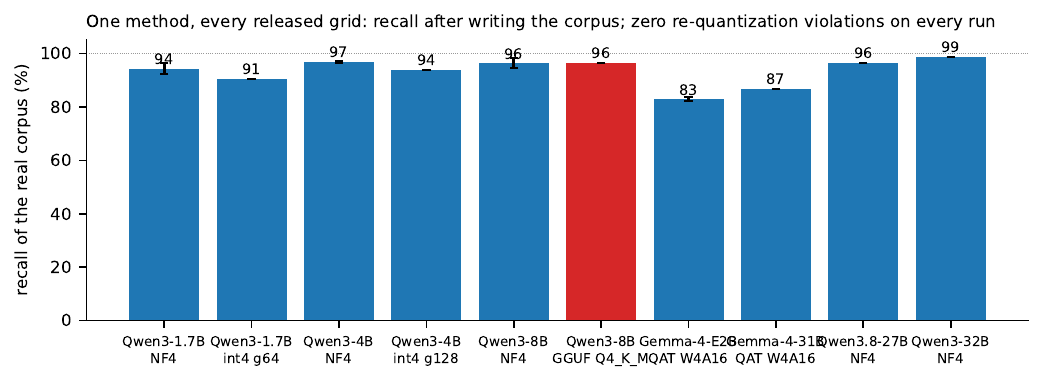}
\caption{Every released grid we wrote into, one bar per release: recall
of the real corpus after the update, with the release bit-identical on
every run (GGUF in red---the file \texttt{llama.cpp} serves). Error bars
span seeds where more than one was run.}\label{fig:releases}
\end{figure}

\begin{table}[t]\centering\small
\setlength{\tabcolsep}{4pt}
\resizebox{\linewidth}{!}{
\begin{tabular}{lrrrrrr}
\toprule
released 4-bit anchor & base & recall & WikiText & LAMBADA ppl$\downarrow$ &saturation & min \\
 & & (507 probes) & ratio & ratio & & \\
\midrule
Qwen3-1.7B (NF4) & 1.4\% & $94.3\!\pm\!2.2$\% & 1.017 & 1.51 & 1.65\% & 14 \\
Qwen3-1.7B (W4A16, uniform int4 g64, asym.) & 0.5\% & 90.5\% & 0.714 & 1.31 & 19.48\% & 83 \\
Qwen3-4B (NF4) & 0.3\% & $96.9\!\pm\!0.4$\% & 1.070 & 1.96 & 2.00\% & 30 \\
Qwen3-4B (W4A16, uniform int4 g128) & 0.2\% & 93.9\% & 0.724 & 1.99 & 2.10\% & 64 \\
Qwen3-8B (NF4) & 1.7\% & $96.3\!\pm\!2.1$\% & 1.014 & 1.90 & 46.53\% & 231 \\
Gemma-4-E2B-it (QAT W4A16 g32) & 0.0\% & $82.9\!\pm\!1.0$\% & 0.057 & 0.23 & 13.57\% & 94 \\
Qwen3.8-27B (self-quantized NF4, stage 4) & -- & 96.4\% & -- & -- & 60.35\% & 676 \\
Gemma-4-31B-it (QAT W4A16 g32) & -- & 86.8\% & -- & -- & 62.63\% & 851 \\
Qwen3-32B (NF4, stage 4) & -- & 98.8\% & -- & -- & 56.67\% & 735 \\
\bottomrule
\end{tabular}
}
\caption{Real knowledge (291 facts, 507 probes) written into published
4-bit releases, each loaded exactly as published and bit-identical after
the update. WikiText and LAMBADA are the served model's perplexity as a
ratio to the release's own; saturation is the fraction of weights whose
fill reached $99\%$ of its cell half-width. Cells are mean$\pm$std over the
seed count in parentheses where more than one seed was run.}
\label{tab:official}
\end{table}

\begin{table}[t]\centering\small
\setlength{\tabcolsep}{4pt}
\begin{tabular}{llrrrr}
\toprule
model & attention & recall & WikiText & LAMBADA & bits/pt \\
\midrule
Qwen3-0.6B  & softmax & 63.8\% & $16.62\!\to\!16.58$ & $38.99\!\to\!59.47$ & 371 \\
Qwen3-1.7B  & softmax & 66.7\% & $11.71\!\to\!11.71$ & $28.67\!\to\!42.33$ & 581 \\
Qwen3-4B    & softmax & \textbf{86.8\%} & $\;\,9.33\!\to\!\;\,9.97$ & $24.41\!\to\!39.35$ & 691 \\
Qwen3-8B    & softmax & 75.9\% & $\;\,8.44\!\to\!\;\,8.42$ & $22.34\!\to\!31.03$ & 1039 \\
Mistral-7B  & softmax & 71.6\% & $\;\,6.62\!\to\!\;\,6.51$ & $14.00\!\to\!21.11$ & 1197 \\
\midrule
Qwen3.8-27B & linear  & 72.1\% & $\;\,7.44\!\to\!\;\,7.27$ & $20.03\!\to\!21.25$ & \textbf{6981} \\
\bottomrule
\end{tabular}
\caption{The matched scale ladder on the synthetic corpus, from $0.6$B to
$27$B and across two attention mechanisms and two families. ``bits/pt'' is
knowledge absorbed per point of cross-domain perplexity, the efficiency
measure of \S\ref{sec:efficiency}. Raw recall is not monotone in scale;
efficiency is.}
\label{tab:ladder}
\end{table}

The first five rows of Table~\ref{tab:ladder} are CellFill $r{=}64$ under
one recipe; the last is clip-merge at 24 epochs, because the first CellFill
implementation did not fit on one 80\,GB A100 at 27B. The obstacle was an
implementation choice and not the method: each wrapped layer stored its
per-weight room $M$ densely, which over $2.435\times10^{10}$ constrained
weights is $48.7$\,GB in bf16 on top of the $12.2$\,GB of 4-bit weights, and
neither gradient checkpointing of the fill (which suffices at 8B) nor
restricting the fill to \texttt{gate\_proj} and \texttt{up\_proj} closes a
$61$\,GB gap. $M$ is a deterministic function of the frozen codes and scales,
so it can be recomputed per block instead of stored; with that change
CellFill does run at 27B, and Table~\ref{tab:official} carries the result
($96.4\%$ on the real corpus, every code returned over all
$2.435\times10^{10}$ constrained weights). The ladder row is left as
clip-merge because the ladder is a matched sweep and re-running the other
five under the new memory path would change more than the 27B point.
Raw recall is
\emph{not} monotone in scale---4B absorbs more than 8B ($86.8\%$ against
$75.9\%$)---and we do not have an explanation we can defend; a larger model
has more to protect, but we have not tested that. Efficiency measured as knowledge per point of cross-domain perplexity rises
steeply and cleanly with scale, $\eta \propto N^{0.42}$ across the five
matched runs---but most of that slope is not ours. Writing
$\Delta=\log(\mathrm{ppl}_1/\mathrm{ppl}_0)$ for the \emph{relative}
cross-domain damage and differentiating
$\log\eta=\log K-\log\mathrm{ppl}_0-\log(e^{\Delta}-1)$ in $\log N$ gives
\[
a_\eta \;=\; a_K \;-\; a_{\mathrm{ppl}_0} \;-\; a_{(e^{\Delta}-1)} ,
\]
and on our ladder $a_K=+0.09$ and $a_\Delta=-0.03$ are both near zero while
$a_{\mathrm{ppl}_0}=-0.30$: the identity closes ($0.09+0.30+0.03=0.42$) and
$71\%$ of the exponent is the base models' own perplexity scaling, inherited
by any metric with perplexity points in its denominator. A point of
perplexity simply means more at a model whose perplexity is lower.

The scale-invariant statement is the one worth making. Knowledge per
\emph{nat} of relative damage, $K/\Delta$, varies by $1.5\times$ across the
ladder (17{,}973 to 27{,}472, cv $17\%$) where $\eta$ varies by $3.2\times$,
and the relative damage itself is flat at $\Delta=0.40$ (cv $15\%$) across a
$12\times$ range in parameters and two model families. What in-cell learning
holds roughly fixed with scale is not the absolute price of knowledge but
the \emph{fraction} of the anchor's cross-domain ability it costs. The 27B point sits far above the $\eta$ trend, but it is a different path,
a different attention mechanism and a single seed, so we quote it as an
observation rather than part of any fit.

We resist the reading this invites. It is tempting to say the price of
in-cell learning falls as models grow---the models actually shipped as 4-bit
artifacts are the large ones---but our own scale-invariant metric does not
show that: $\Delta$ is flat at $0.40$ (cv $15\%$) and $K/\Delta$ spans only
$1.5\times$ with cv $17\%$ across a $12\times$ parameter range. The falling
price is carried by the 27B point, which the paragraph above excludes from
the fit for three separate reasons. A point cannot be excluded from a fit and
then be allowed to carry the conclusion. What the ladder supports is that the
relative cost does \emph{not} rise with scale, which is the property a
deployment needs; whether it falls is unresolved here. The 27B run also shows the difference
is not only quantity: at 1.7B recall is dominated by the probe that
continues the training sentence almost verbatim (city $65.9\%$ against
company $17.7\%$ at $10^4$ facts), while at 27B the paraphrased probe
reaches $53.5\%$ and the gap between probe kinds narrows sharply
(city $74.8\%$, occupation $88.0\%$).

\paragraph{A divergence, and what it says about the bound.}
Our first Mistral-7B run did not work at all: transplanting the 1.7B
hyperparameters ($r{=}16$, lr $2\times10^{-4}$) drove plain LoRA to
perplexity $2311$ \emph{before any merging}, with a $39.6\%$ clip rate and
only $40\%$ of the update norm surviving projection---all signatures of a
diverged optimization, faithfully reproduced by clip-merge. We report it
because the archived file is public and because the contrast is
informative: CellFill on the same model at \emph{five times} that learning
rate ($10^{-3}$) trained cleanly to the $71.6\%$ in the table. This is what
the parameterization predicts. Since the fill is $M\odot\tanh(\cdot)$ with
$|\tanh|<1$, the update cannot leave the cell however large the gradients
become, so the bound is a stability property and not only an invariance
property. Lowering the learning rate to $5\times10^{-5}$ makes plain LoRA work on
Mistral, and work well---$34.0\%$ recall at $2557$ bits per point, the second
most efficient run in this paper---so the divergence was a transplanted
hyperparameter and nothing more. What survives is the margin: the bounded
parameterization was stable at a learning rate $20\times$ larger than the one
plain LoRA needed, which is the practical form the guarantee takes. 

\paragraph{The controlled version of that observation.}
The anecdote above confounds two things, so we ran the matched sweep: one
model (Qwen3-1.7B), one seed, rank $16$, rehearsal $0.1$, $24$ epochs, and
only the parameterization and the learning rate varying.

\begin{center}\small
\begin{tabular}{llrrrrl}
\toprule
lr & parameterization & \multicolumn{2}{c}{perplexity} & recall & LAMBADA ppl$\downarrow$ & bound engaged \\
& & unmerged & served & & & \\
\midrule
$10^{-3}$ & LoRA $\to$ clip-merge & 34.30 & 13.35 & 6.3\% & 39.81 & 63.3\% clipped \\
$10^{-3}$ & CellFill & --- & 10.51 & 34.0\% & 32.85 & 0.1\% saturated \\
\midrule
$3\times10^{-3}$ & LoRA $\to$ clip-merge & 1,272.42 & 20,580.44 & 0.0\% & 115,458.72 & 70.2\% clipped \\
$3\times10^{-3}$ & CellFill & --- & 11.82 & 54.2\% & 39.55 & 47.5\% saturated \\
\bottomrule
\end{tabular}
\end{center}

At $10^{-3}$ plain LoRA has already damaged the model before anything is
merged---WikiText $34.30$ against an anchor of $11.71$, LAMBADA $311.1$
against $28.67$---and clip-merge, which must clip $63.3\%$ of the weights and
keeps $37.5\%$ of the update norm, delivers $6.3\%$ recall. Tripling the
learning rate destroys it outright. CellFill is stable at both settings and
\emph{better} at the larger one: recall $34.0\to54.2\%$, for WikiText
$10.51\to11.82$ and LAMBADA $32.85\to39.55$. The mechanism is in the last
column. The excess gradient never leaves the cells; it is absorbed by
saturation, which rises from $0.1\%$ to $47.5\%$ of coordinates pressed
against $|\tanh|\to1$. Divergence is a cliff and saturation is a ceiling, and the parameterization
converts the first into the second. This is the paper's most direct evidence
on catastrophic forgetting, in the sense the term was coined
for~\cite{mccloskey}---an abrupt collapse of prior capability rather than a
graded cost: at $3\times10^{-3}$
the LoRA arm loses everything it had---WikiText $2.1\times10^{4}$, LAMBADA
$1.2\times10^{5}$, recall $0$---while the bounded arm at the identical
learning rate, seed and rank keeps WikiText within $0.11$ of the anchor and
returns its best recall of the sweep.

The claim this licenses has to be made in three parts, because the strongest
version of it is false and this paper contains the counterexample. First,
collapse \emph{by way of unbounded drift} is structurally unavailable: no
gradient can move a weight past a wall the grid fixed before training began,
so the failure mode in which the served weights wander arbitrarily far from
the release cannot occur. Second, that is a statement about the weights and
not about the function, and staying inside the cells does not by itself make a
configuration safe, and this paper contains a stark counterexample. The last
row of Table~\ref{tab:regularizer} is a clip-merged Qwen3-4B---every weight
inside its cell, the released codes intact, the identity satisfied---whose
WikiText perplexity is $3.2\times10^{6}$ against its anchor's $9.33$, whose
LAMBADA is $2.1\times10^{6}$ against $24.39$, and whose recall is zero on all
nine domains. A guarantee about the artifact is not a guarantee about the
function, and the archive says so by five orders of magnitude. Third, and empirically, bounded \emph{training}
has not produced a collapse in any run in this paper, including the matched
pair above where the unbounded arm fails by two to four orders of
magnitude, across the archived and reproduced runs, at the same learning
rate. The third is the claim a deployer cares about and it is
evidence rather than a theorem: it is a property of the link and of the
optimization under it, not of the box's geometry. What remains in every case
is a graded cost, and \S\ref{sec:cost} and \S\ref{sec:regularizer} measure it. Notably, the higher learning rate at rank $16$
does not substitute for quadrupling the rank: $54.2\%$ at
$r{=}16$ against $66.7\!\pm\!8.6\%$ for $r{=}64$ at $10^{-3}$, though at a
comparable exchange rate ($592$ against $583\!\pm\!59$ bits/pt). An earlier
version of this paper read these as substitutable, on a single rank 64 seed
that turned out to be the worst of three.

Two caveats. This is one seed per cell, so we report the direction, which is
large, and not the size. And the CellFill $10^{-3}$ cell is a repeat of the
configuration in Table~\ref{tab:frontier}, whose seed-0 run recorded
$36.9\%$: the same code, the same seed and the same hyperparameters on
different GPUs (RTX 4090 against A100) differ by $2.9$ recall points, since we do not force
deterministic kernels. That gap is a lower bound on the noise in every $n{=}1$
comparison in this paper, and we have not subtracted it anywhere.

\subsection{Saturate: what stops one update, and what caps its return}\label{res:saturate}
\subsubsection{Capacity is limited by the optimizer, not by the cell space}
\label{sec:capacity-xdom}
CellFill at $r{=}64$ on Qwen3-1.7B took the 291-fact real corpus to
$95.7\%$ recall over 507 probes while using $1.54\%$ of the available cell
space (\S\ref{sec:realcorpus}), and absorption on the two projection paths
had not flattened by $10^4$ synthetic facts at $0.74$--$1.03\%$ cell-space
utilization (Fig.~\ref{fig:capacity}). Storage was not the binding
constraint at either scale; the optimization under the bound was. The
frontier of \S\ref{sec:bfrontier} sharpens this: the optimizer binds through
its parameterization, and granting it one bounded degree of freedom per
weight moves the ceiling further than the rank ladder ever did.

\begin{figure}[t]\centering
\includegraphics[width=0.55\linewidth]{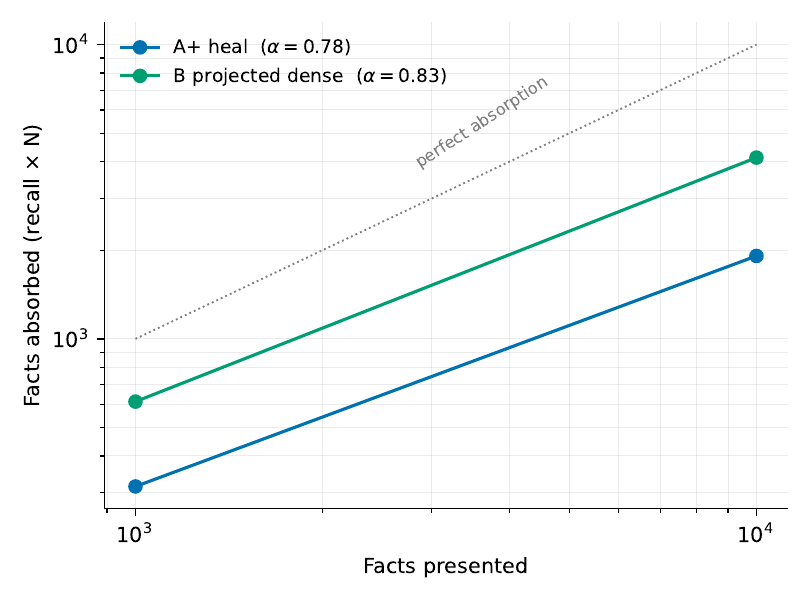}
\caption{Absorbed vs.\ presented facts (log--log), two points per path.
Both invariant paths scale with exponent $\approx\!0.8$ at $0.74$--$1.03\%$
cell-space utilization. Two points cannot resolve a saturation knee; the
claim is only that absorption has not flattened by $10^4$ facts.}
\label{fig:capacity}
\end{figure}

\textbf{With the paper's own method.} The series above was measured on the
two projection paths at $r{=}16$ and rehearsal $0.3$;
Table~\ref{tab:capacity-cellfill} is the same fact counts with CellFill at
$r{=}64$ and rehearsal $0.1$, and it raises a different candidate for
\emph{what} limits capacity. From $10^3$ to $10^4$ facts, absorbed facts grow $667\to3{,}538$
($\alpha\approx0.72$) while the saturated fraction of the fill goes from
$3.7\%$ to $78.4\%$: at $10^4$ facts three quarters of the coordinates have
their $\tanh$ pinned past $0.99$, where the fill is a sign and the gradient
is below $2\times10^{-2}$. Past that the series collapses rather than
bends: $3\times10^4$ facts absorb $3.7\%$ and $10^5$ absorb $3.6\%$ at
$90$--$92\%$ saturation. Whether the knee is the cells filling or the link
dying is a question the bounded parameterization makes askable, and
Table~\ref{tab:capacity-bits} asks it at $10^4$ facts with one knob moved
at a time. Four ways of keeping gradient alive at the same bound and the
same reachable set---a straight-through hard-tanh, a tanh whose backward
derivative is floored at $0.05$, the softsign link whose derivative decays
polynomially, and the same tanh at a quarter of the slope ($s{=}10$)---take
absorption from $22$--$24\%$ (two seeds on this card) to $36$--$39\%$.
Rank moves it the other way past $r{=}32$: $r{=}16$, $32$, $64$ and
$128$ absorb $31$, $44$, $23$ and $13\%$ at $37$, $62$, $78$ and $88\%$
saturation---saturation rises monotonically with rank, and beyond
$r{=}32$ more parameters store less. Narrowing the fill's share of the
cell costs absorption without relieving the link: $23.8\%$ at full width,
$21.1\%$ at a half and $5.6\%$ at a quarter, while the share of coordinates
past $|t|{=}0.99$ \emph{rises} across the three, $78.2$, $82.2$ and $85.8\%$.
A narrower box does not bring the walls closer in the sense that matters---it
scales down the displacement the link is asked to produce, so the link works
harder for less. Two of these saturations need care and the table's column
does not separate them: that column is computed on the \emph{scaled}
displacement, which at width $\tfrac12$ cannot exceed $\tfrac12$ and so reads
$0.0\%$ by construction, while the figures above come from the per-epoch
liveness probe on the unscaled link, which is the quantity the argument
needs. Softsign's $0.0\%$ in the same column is genuine. Four paraphrases of each fact at the same number
of exposures add nine points.
The knee at $10^4$ is therefore the link's dead gradient first and the
cells second: a coordinate at its wall cannot move whatever the gradient,
but most of the coordinates that stopped learning were not at a wall they
needed to be at.

\textbf{Capacity in bits.} Exact-match recall is a threshold on the argmax
and swung by ten points between seeds at $10^3$ facts; a law should be
written in information. For every attribute $a$ of every fact the model's
log-probability of the true value, renormalized over the generator's
vocabulary $V_a$, gives the residual uncertainty $H(a\mid\text{model})$
and the bits stored, $\log_2|V_a|-H(a\mid\text{model})$, clipped at
zero; the released model scored the same way with its fills switched off
is the floor, and the law is written in the excess over it
(\texttt{experiments/eval\_bits.py}; the bit-complexity measure
of~\cite{physics33} on cloze prompts). At $10^4$ facts the fills hold
$4$--$7$ bits of each fact's $21.7$, $4\times10^4$--$7\times10^4$ bits in
all, which is $0.6$--$3\times10^{-3}$ bits per trainable parameter and
$3$--$5\times10^{-5}$ per written cell: three orders of magnitude below
the two bits per parameter a model stores at full
precision~\cite{physics33}. The parametrization is nowhere near its
information capacity; what limits it is the optimization under the bound,
which is what the links repair.

\begin{table}[t]\centering\small
\setlength{\tabcolsep}{4pt}
\resizebox{\linewidth}{!}{
\begin{tabular}{llrrrrrr}
\toprule
facts & fill & bits/fact & total bits & bits/param & bits/cell & recall & saturation \\
\midrule
3{,}000 & r16 & 11.24 of 21.7 & 3.37e+04 & 0.0019 & 2.4e-05 & 37.6\% & 4.5\% \\
3{,}000 & r256 & 11.15 of 21.7 & 3.35e+04 & 0.0001 & 2.4e-05 & 37.8\% & 78.1\% \\
10{,}000 & r16 & 4.97 of 21.7 & 4.97e+04 & 0.0029 & 3.5e-05 & 31.1\% & 36.9\% \\
10{,}000 & r32 & 4.98 of 21.7 & 4.98e+04 & 0.0014 & 3.5e-05 & 44.2\% & 61.7\% \\
10{,}000 & r64 (s1) & 2.28 of 21.7 & 2.28e+04 & 0.0003 & 1.6e-05 & 23.8\% & 78.2\% \\
10{,}000 & r128 & 0.54 of 21.7 & 5.41e+03 & 0.0000 & 3.8e-06 & 12.7\% & 87.9\% \\
10{,}000 & r64, width 1/2 & 2.41 of 21.7 & 2.41e+04 & 0.0003 & 1.7e-05 & 21.1\% & 0.0\% \\
10{,}000 & r64, hardtanh-STE & 4.41 of 21.7 & 4.41e+04 & 0.0006 & 3.1e-05 & 39.4\% & 84.5\% \\
10{,}000 & r64, tanh floor 0.05 & 4.22 of 21.7 & 4.22e+04 & 0.0006 & 3.0e-05 & 36.3\% & 77.0\% \\
10{,}000 & r64, softsign & 5.93 of 21.7 & 5.93e+04 & 0.0009 & 4.2e-05 & 36.1\% & 0.0\% \\
10{,}000 & r64, $s{=}10$ & 6.70 of 21.7 & 6.70e+04 & 0.0010 & 4.8e-05 & 38.9\% & 28.2\% \\
10{,}000 & r32, $s{=}10$ & 6.50 of 21.7 & 6.50e+04 & 0.0019 & 4.6e-05 & 26.5\% & 6.7\% \\
10{,}000 & r64, 4 paraphrases $\times$ 6 epochs & 1.86 of 21.7 & 1.86e+04 & 0.0003 & 1.3e-05 & 31.1\% & 78.5\% \\
10{,}000 & r64, 4 paraphrases $\times$ 24 epochs & 18.18 of 21.7 & 1.82e+05 & 0.0026 & 1.3e-04 & 91.7\% & 90.6\% \\
10{,}000 & r64, Qwen3-4B & 1.30 of 21.7 & 1.30e+04 & 0.0001 & 3.6e-06 & 27.2\% & 87.9\% \\
\bottomrule
\end{tabular}
}
\caption{What sets the stored information, one knob at a time
(Qwen3-1.7B, synthetic facts, 24 exposures unless stated). Bits are the
excess over the released model's priors, summed over the five attributes;
``bits/param'' is per trainable parameter of the fill; ``saturation'' is
the share of coordinates past $|t|{=}0.99$ at the end of training, computed on
the \emph{scaled} displacement: for the reduced-width rows that quantity is
bounded by the width and reads $0.0\%$ by construction, and the text quotes
the per-epoch liveness probe on the unscaled link instead. Softsign's $0.0\%$
is genuine. The
$r{=}64$, $s{=}40$ tanh baseline on this card absorbs $22$--$24\%$ (two
seeds).}
\label{tab:capacity-bits}
\end{table}

\paragraph{The budget above is spent over every matrix, which is the wrong
default.} Restricting the same real-corpus fill to the MLP gives up $4.6$
points of recall and buys back a point of WikiText, nine points of LAMBADA
and more composition than writing everywhere;
\S\ref{sec:where} gives the three-way split and the seven-partition sweep
behind it. What matters for the capacity accounting here is that the totals
in Table~\ref{tab:capacity-bits} are not a property of the model alone. They
are a property of a target set, and a better target set moves all three axes
at once.

Read together with the partition sweep of \S\ref{sec:where}, the prescription
is narrow and it is the same under both prices. Write to the gate and up
projections, across all depths: they lead on absorbed bits per unit of
trainable budget ($130$) and on bits per point of cross-domain perplexity
given up ($911$), and they sit inside the MLP, which is where the composable
form of the knowledge appears. Do not write to \texttt{down\_proj} ($72$ and
$210$ on the same two measures, the worst non-degenerate partition), and do
not spend the budget on the early or middle third of the network, whose
recall under this budget---$5.8\%$ and $6.0\%$---does not separate from the
marginal-guess floor even though their cells are the most saturated of any
partition. That last pair is the one result here that runs against the
localization literature~\cite{rome,memit}, and \S\ref{sec:where} argues it is
a difference of task---where an association can be found and edited is not
where a new one is most cheaply written---rather than a contradiction.

\begin{table}[t]\centering\small
\setlength{\tabcolsep}{4pt}
\begin{tabular}{lrrrrr}
\toprule
facts & recall & absorbed & LAMBADA ratio & saturation & min \\
\midrule
1{,}000 & $66.7\!\pm\!8.6$\% & 667 & 1.48 & 3.68\% & 11 \\
3{,}000 & 61.4\% & 1,843 & 2.09 & 41.87\% & 80 \\
10{,}000 & 35.4\% & 3,538 & 1.82 & 78.41\% & 128 \\
30{,}000 & 3.7\% & 1,106 & 1.94 & 90.28\% & 601 \\
100{,}000 & 3.6\% & 3,627 & 2.20 & 92.06\% & 599 \\
\bottomrule
\end{tabular}

\caption{Capacity with CellFill ($r{=}64$, Qwen3-1.7B, synthetic facts).
``absorbed'' is recall times facts presented; ``saturation'' is the
fraction of fill coordinates with $|\tanh|>0.99$. The $10^3$ row is the
mean over three seeds (std shown for recall); the $10^4$ row is a single
run.}
\label{tab:capacity-cellfill}
\end{table}

On the projection paths, from $10^3$ to $10^4$ presented facts, absorbed
facts grow $315\!\to\!1913$ for A+ ($\alpha\!\approx\!0.78$) and
$612\!\to\!4130$ for projected dense ($\alpha\!\approx\!0.83$), with
cell-space use of $0.74$--$1.03\%$: that regime is
optimization-limited (training loss $1.43$ at $10^4$ facts, 24 exposures),
not capacity-limited---five orders of magnitude below the $kN_t$ shipped-bit
ceiling of
Prop.~\ref{prop:capacity} and far below the $\sim$2-bit/parameter
full-precision ceiling of~\cite{physics33}. At $10^4$ facts the projected
dense path absorbs $4{,}130$ facts ($4130\times11.9\approx49$ kbit of
probed attribute entropy) under exact invariance, on a single consumer GPU
in 3.3 hours. Two points cannot resolve a knee, so we claim only that none
is visible in this range.

The cross-domain measurement at $10^4$ facts, which we added after the
audit described in \S\ref{sec:rehearsal}, is the most severe result in this
paper. Projected dense training absorbs $41.6$ kbit while WikiText
perplexity stays \emph{below} the 4-bit anchor ($11.73\!\to\!11.69$)---and
LAMBADA rises from $28.70$ to $129.74$, a factor of $4.5$. A practitioner
watching in-domain perplexity would conclude nothing had gone wrong. The
recipe matters enormously here: rank 64 healing with $10\%$ rehearsal
absorbs $73.4$ kbit at the same scale---$76\%$ more knowledge---for
$64.43$ rather than $129.74$ ($2054$ bits per point against $412$). At
$10^4$ facts, method choice is worth a factor of five in efficiency, far
more than it is worth at $10^3$.

We also observed a reproducibility issue worth recording: two runs with
identical configuration and seed at $10^4$ facts differed by 6.4 recall
points ($41.3\%$ vs.\ $34.9\%$), far beyond probe-sampling error. Dense
training at this scale is not reproducible to better than a few points on
our stack, so the capacity exponents above should be read with that
uncertainty.
The $10^5$-fact point is archived (\texttt{exp22\_100k}) and is reported
here for what it is: with the clip-merge path, rank 64 and 12 epochs, the
healed model recalls $5.44\%$ and the unmerged adapter $6.02\%$, both below
the $6.5\%$ guessing floor, with the adapter's LAMBADA at $503$---the
optimization fails before projection enters. The interference knee lies
between $10^4$ and $10^5$ for that path. CellFill reaches the same
territory: Table~\ref{tab:capacity-cellfill} runs to $10^5$ facts, where
recall is $3.6\%$ at $92.1\%$ saturation---below the guessing floor, with
the link almost entirely against the walls.

\subsubsection{Absorption is set by the parameterization, and priced in
capability}\label{sec:bfrontier}

\begin{figure}[t]\centering
\includegraphics[width=0.72\linewidth]{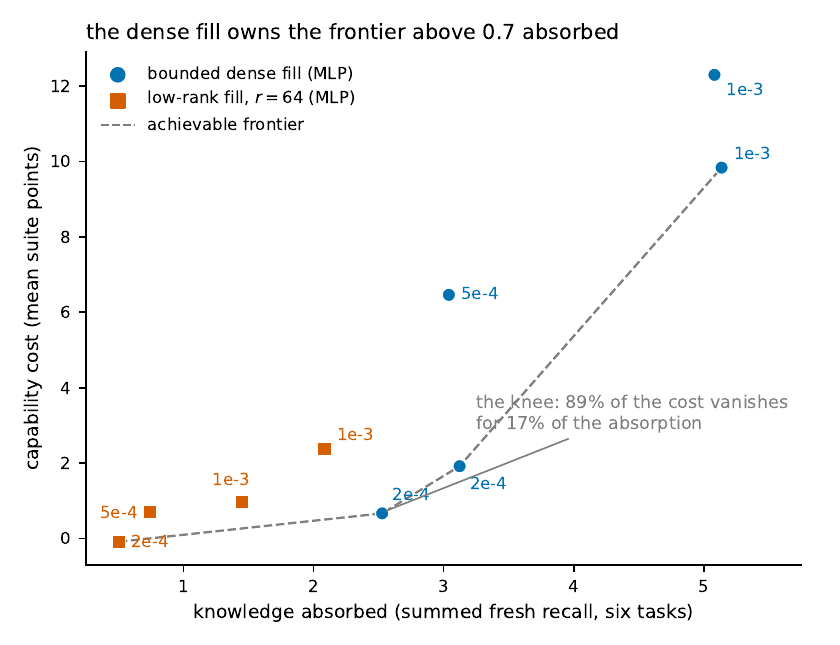}
\caption{\textbf{The parameterization frontier.} Eight arms, one recipe: six
tasks on Qwen3-1.7B, MLP placement, rehearsal with a held-out half, each
point one sequence (the $10^{-3}$ endpoints are two seeds each); capability
cost is the anchor-minus-final mean over ARC-e, ARC-c, HellaSwag and
WinoGrande at 600 items per suite. The bounded
dense fill and the rank 64 fill do not trade along one line. The dense curve
has a knee between $5\times10^{-4}$ and $2\times10^{-4}$, where $89\%$ of the
capability cost vanishes for $17\%$ of the absorption, and every low-rank
arm, including both seeds at its own working rate, lies inside the frontier.
The knee carries a second seed, run later, that lands beside the first.
Invariance violations are zero on all nine arms. Generated from the
archived products by \texttt{scripts/fig\_bfrontier.py}.}
\label{fig:bfrontier}
\end{figure}

The capacity ceiling above is an optimizer ceiling, and the optimizer works
through a parameterization. Replacing the rank 64 factors with a bounded
dense tensor---$M\odot\tanh(sZ)$, one free parameter per constrained MLP
weight, the cells and the invariance check untouched---moves the ceiling by
more than the rank ladder ever did. Six-task sequences letter-identical up to
that one flag absorb $5.08$--$5.14$ summed fresh recall against
$1.45$--$2.08$ for rank 64 (two seeds each, \texttt{exp132}), with no
plasticity collapse: the dense arm's sixth task learns $90$--$119\%$ of its
first, the rank arm's $40$--$49\%$.

The price column keeps the result honest. At the shared learning rate the
dense arm pays $9.8$--$12.3$ points of mean suite accuracy against
$1.0$--$2.4$; absorption is bought, not conjured. But the trade is not
linear. Sweeping the dense arm's rate down (Fig.~\ref{fig:bfrontier})
finds a knee: at $2\times10^{-4}$ it still absorbs $2.53$---more than
either rank seed at the rank arm's own best rate---while its capability cost
falls to $0.7$ points, within the suite noise. The fair-play sweep on the
rank arm does not restore the balance: its absorption halves for about one
point of relief, and at $2\times10^{-4}$ it absorbs $0.50$, having nearly
stopped learning. Above $0.7$ summed absorption the achievable frontier is
dense at every point we measured.

Low rank, on this evidence, is not the safer choice but the smaller one: at
matched capability cost the dense parameterization writes three times the
knowledge. The curve points are single runs and the statement carries its
scale (1.7B, MLP placement); the champion point's second seed and an 8B
confirmation on real facts are in flight (\S\ref{sec:limitations}).

\subsubsection{A second limit: what the model already knows cannot be
written}\label{sec:prior}
The capacity limit of \S\ref{sec:capacity-xdom} is a limit on writing.
There is a second one, of a different kind, on what writing achieves.

We measured what a released model already knows about a set of Python
libraries before any fill, by ranking a correct call against distractor
signatures under the released weights, and then injected each library
separately with one fill per library at identical settings on Qwen3-4B-Base.
Nine arms were launched and the eight of them that completed are in
Table~\ref{tab:prior}. Across the
seven libraries whose prior accuracy fell below $28\%$, the gain was
$+10.7$ points with a standard deviation of $1.0$---flat, and independent
of where in that range the library started. On \texttt{requests}, the one
library the model already knew at $31.1\%$, the same recipe returned
$+2.7$.

\begin{table}[t]\centering\small
\begin{tabular}{lrrrrr}
\toprule
library & prior & after & gain & chance & end sat.\\
\midrule
httpx       & 12.7\% & 22.7\% & $+10.0$ & 11.6\% & 0.0013\%\\
json        & 14.3\% & 23.8\% & $+9.5$  & 9.6\%  & 0.0000\%\\
dataclasses & 18.8\% & 31.2\% & $+12.5$ & 9.6\%  & 0.0000\%\\
rich        & 21.8\% & 33.1\% & $+11.3$ & 11.5\% & 0.4089\%\\
pydantic    & 23.6\% & 33.1\% & $+9.6$  & 12.1\% & 0.0481\%\\
click       & 23.7\% & 35.3\% & $+11.5$ & 12.2\% & 0.0285\%\\
argparse    & 24.1\% & 34.5\% & $+10.3$ & 11.1\% & 0.0000\%\\
\midrule
requests    & 31.1\% & 33.8\% & $+2.7$  & 15.3\% & 0.0009\%\\
\bottomrule
\end{tabular}
\caption{One fill per library, Qwen3-4B-Base, rank 64, 24 epochs, scored by
log-probability ranking against distractor signatures. The gain is flat
below a prior of roughly $28\%$ and collapses above it. End saturation is
the share of the fill sitting against the cell walls at the last epoch.
Baseline and served are measured on the same card for each row, because
these probe sets are small enough that the cross-machine spread is
comparable to the smallest effect here: \texttt{json}'s prior reads $9.5\%$
on one card and $14.3\%$ on another for the same corpus and model. The
placement of a library relative to the $28\%$ threshold should be read with
that spread in mind.}
\label{tab:prior}
\end{table}

The collapse is not the limit of \S\ref{sec:capacity-xdom} seen from
another side, and we checked: end-of-run saturation on \texttt{requests}
was $0.0009\%$, three orders below \texttt{rich}'s $0.4089\%$, and its
training loss fell by $2.43$, the second largest drop in the set. The cells
were nowhere near full, the link was still passing gradient, and the fill
demonstrably moved. It simply did not change the answer. Across the eight libraries
of Table~\ref{tab:prior}, and computed from the values printed there so a
reader can check them, saturation correlates $+0.07$ with prior accuracy and
$+0.24$ with gain---neither carries the effect---while prior accuracy against
gain is $-0.57$.

Nor is it a common ceiling. Injection did not carry every library to one
level: the spread across libraries narrowed by only $17\%$ (standard
deviation $5.5$ to $4.6$ points), and post-injection accuracy relative to
each library's own chance line still ranges from $1.96\times$ to
$3.26\times$.

Two consequences follow, one for the method and one for anyone using it.
For the method, a falling training loss is not evidence that an injection
worked---\texttt{requests} optimized well and gained almost nothing, and no
quantity available during training distinguishes it from the seven that
succeeded. For the user, the return on a fill is set by what is missing
rather than by how much is presented: \texttt{dataclasses} gained the most
in the set, $+12.5$, from the smallest corpus of $18$ facts. Since prior
accuracy is measurable in minutes and a fill costs hours, the ordering is
worth measuring before it is worth training.

\subsubsection{What is knowable before training}\label{sec:knowable}

The sixth link of the argument of \S\ref{sec:introduction} is that how much a
release will absorb can be estimated before anything is written into it. That
is true in a weaker sense than the phrase suggests, and the archive is
unusually clear about which sense, because two of the four estimators we
tried failed and the failures were kept.

What works is a scan and a rate. The scan is the released model's own prior
accuracy on the material about to be written: injecting eight Python
libraries one at a time, the gain was flat at $+10.7$ points wherever that
prior fell below $28\%$ and collapsed to $+2.7$ on the one library the model
already held at $31.1\%$, with the cells barely touched and the training loss
falling normally throughout---so a successful optimization is no evidence of
a successful injection, and the prior is what tells them apart
(\S\ref{sec:prior}). The rate is the room's decay: each fold consumes a
constant fraction of what the last one left, $\beta$ between $0.78$ and
$0.85$ over four matched sequences at 1.7B and lower at 8B
(\S\ref{sec:loop8b}), which turns a first task's absorption
into a bound on the whole run before the run starts (Prop.~\ref{prop:decay},
\S\ref{sec:sequential}). Both are cheap---the scan is a few hundred probes on
the released file, the rate is one constant per recipe---and both are
\emph{target-specific}: they are measured on the corpus and the release at
hand.

What does not work is anything read off the artifact alone. The
diagonal-Fisher budget of Prop.~\ref{prop:budget} is the natural candidate,
and it tracks the wrong magnitude by up to $383\times$; we report it as a
scaling law rather than a numerical certificate. The data-processing bound on
how many bits a fill of given rank can carry is correct and loose by five
orders of magnitude, which makes it a framing and not a tool. Neither is a
near miss to be tightened later, and we have not found a third.

The honest statement of this link is therefore: capacity is predictable per
target from a scan that costs less than the write, and is not yet computable
from the released file's own statistics. A formula in the second sense would
be a better paper than this one; the two estimators above are what we can
defend.

\subsection{Validate: is the minor version still shippable?}\label{res:validate}
A minor version is not finished when the fill converges. What was written
has to be usable, and the model that carries it has to still be the model
the release was evaluated as. Both are gates rather than metrics: an
operator who cannot answer them cannot ship, and an update that fails
either is discarded by subtraction at no cost to the artifact.

\subsubsection{Injected facts compose above the model's own two-hop rate}\label{sec:usable}
Two facts written into the cells composed at inference more often than two
facts the same model had learned in pretraining: $80.8\%$, $94.5\%$ and
$89.6\%$ against $42.5\%$, $52.6\%$ and $34.4\%$ on Qwen3-1.7B, Qwen3-4B and
Qwen3-8B (Table~\ref{tab:twohop}). The injected chain is
indication$\to$drug$\to$ingredient; in the corpus arm that the table reports,
both of its single hops are stated in the forward direction and the
composition itself is stated nowhere. The native chain is a person's country
of birth and that country's capital, read on the release before any update.
Each rate is taken over the items whose two single hops that model answered,
and that conditioning is also the boundary: on Gemma-4-E2B and on the three
models at 27B and above, both hops were known for $2$ to $7$ of the $73$
injected items, too few for those rows to carry a composition claim. Recall
on paraphrased and reversed probes shows the facts are stored as
associations rather than strings, not that they can be \emph{used}: combined
with one another, or exercised in a form the training text never took. Two
tests ask that, and both need a ruler before their numbers mean anything.

\textbf{Composition against the model's own ceiling.} The corpus was built to
state, for each drug, an indication and an ingredient in separate relations,
so that a probe for indication$\to$ingredient could only be answered by
chaining two injected facts. It does not: $324$ of the $432$ medicine
sentences name the ingredient and the indication together, and one paraphrase
states the composed relation outright (``berdazimer, sold as Zelsuvmi, is
indicated for molluscum contagiosum''). The consequence is worked through
below and it is fatal to this arm as a composition test, but the ceiling the
arm was measured against is worth stating first, because it is the part that
survives. Read against $100\%$, the first
corpus arm's $6.8\%$ at 1.7B is a failure. Read against the model it is
something else. On facts the same model learned in pretraining---a person's
country of birth and that country's capital---it knows both hops for $252$
of $264$ items and chains them for $42.5\%$ of those, answering the rest
with a city associated with the person, usually the birth city: the model's
own compositionality gap~\cite{press2022}, measured on the exact model under
test with both probe sets conditioned identically on the single hops being
known. That ceiling is not monotone in scale, which is why it is measured
per model rather than assumed: on the same 264 items it is $42.5\%$ at 1.7B,
$52.6\%$ at 4B and $34.4\%$ at 8B, while all three answer the second hop
perfectly.

Conditioning the injected side the same way exposed a defect in the corpus
rather than in the model. The injected chain is indication$\to$drug$\to$
ingredient. On every served model the second hop is answered for $97$--$100\%$
of items; the \emph{first} hop---which drug treats this indication---is
answered for $4$--$8\%$. The corpus never states that relation in that
direction: every training sentence and every paraphrase puts the drug or the
ingredient first, so indication$\to$drug is a reversal the model was asked
to infer. The composition prompt itself is answered for $6.8\%$, $12.3\%$
and $16.4\%$ of items across 1.7B, 4B and 8B---slightly \emph{above} the
first hop, which is possible only because indication and ingredient share a
training sentence, so the ``composition'' has a one-hop route the chain
does not. Conditioned on both hops being known, the eligible items number
three to six per model, which is no measurement.
This is the asymmetry the corpus's own construction note warns
about---withholding a phrasing makes a corpus artificially hard rather than
honestly hard---and the remedy is one more training sentence per drug,
stating the first hop in the forward direction, while the composition itself
remains unstated anywhere in the corpus. That arm is the one the table is
built to report (Table~\ref{tab:twohop}); the numbers above are kept in
the text so a reader can see what a composition rate measured without its
single-hop conditions is worth. One more number belongs here: the served
models of the first corpus arm still compose their \emph{native}
facts---$50.0\%$, $51.4\%$ and $67.0\%$ at the better phrasing, against
$42.5\%$, $52.6\%$ and $34.4\%$ before the update, up at 1.7B and 8B and down
$1.2$ points at 4B---so whatever the fill costs, it is not the ability to
chain.

\begin{table}[t]\centering\small
\setlength{\tabcolsep}{4pt}
\begin{tabular}{lrrrrrr}
\toprule
 & \multicolumn{3}{c}{native (people$\to$country$\to$capital)} & \multicolumn{3}{c}{injected (indication$\to$drug$\to$ingredient)} \\
\cmidrule(lr){2-4}\cmidrule(lr){5-7}
model & hops known & two-hop & $|$both & hops known & two-hop & $|$both \\
\midrule
Qwen3-1.7B & 252/264 & 41.7\% & 42.5\% & 73/73 & 80.8\% & 80.8\% \\
Gemma-4-E2B & -- & -- & -- & 2/73 & 5.5\% & 100.0\% \\
Qwen3-4B & 249/264 & 51.1\% & 52.6\% & 73/73 & 94.5\% & 94.5\% \\
Qwen3-8B & 241/264 & 33.7\% & 34.4\% & 67/73 & 86.3\% & 89.6\% \\
Qwen3.8-27B & 242/264 & 97.0\% & 97.5\% & 7/73 & 27.4\% & 57.1\% \\
Gemma-4-31B & -- & -- & -- & 4/73 & 19.2\% & 100.0\% \\
Qwen3-32B & 196/264 & 78.4\% & 79.6\% & 3/73 & 21.9\% & 66.7\% \\
\bottomrule
\end{tabular}

\caption{Composition, conditioned on both hops being known: native facts
from pretraining (people$\to$country$\to$capital), measured on the release
before any update, against the injected indication$\to$drug$\to$ingredient
chain that the corpus never states, measured on the served model. ``hops known'' counts items whose two single hops the model
answers correctly; ``$|$both'' is the two-hop rate over those items and is
the number to compare across the two halves. Each set is read at the better
of two phrasings (continuation, question) for that model. Rows whose eligible
count is in single digits are printed for completeness and are not read: at
$n{=}2$ a rate of $100\%$ is two items, and this paper's own standard
(\S\ref{sec:usable}) is that three to six eligible items are no
measurement.}
\label{tab:twohop}
\end{table}

\textbf{Using an API in code.} A second corpus is a library's public API,
extracted by introspection (\texttt{inspect.signature} on the installed
package, so the corpus cannot drift from the library and a reviewer
regenerates it with one command): 67 symbols of \texttt{cyclopts}, a
library whose pinned version postdates the models' training data and whose
recitation probes scored at the floor on the 1.7B base in our selection
run. The recitation probe asks, in prose,
what the $n$th parameter of a symbol is named. The usage probe asks for the
same fact as a line of code---finish \texttt{obj = cyclopts.App(} so that
only its second parameter is set---and is scored by parsing the completion
and binding it against the real signature: a positional guess lands on the
first parameter, so for the second and third parameters---two thirds of the
tasks---the only way through is the name, in a keyword. The stdlib is the positive control (\texttt{csv}, \texttt{argparse},
\texttt{json}), and it is where the test finds its floor. Base models score
$0/56$ at 1.7B, $1/56$ at 4B and $4/56$ at 8B on libraries they certainly saw
in pretraining, even with two worked examples in the prompt: they do not
follow the instruction, emitting a plausible full call with invented or
extra keywords, or no parseable call at all, instead of the one requested
parameter. None of the three sizes can
discriminate, so none of their \texttt{cyclopts} numbers is evidence about
injected knowledge, and we report the control rather than the ratio. The
injection runs only where the control passes. This is a limit of the test
and not of the models: a base model that cannot follow ``set only the second
parameter'' is being asked the wrong question. A second scoring rule asks
the same question without the instruction: the true parameter name is ranked
among the symbol's own names plus four borrowed from other symbols by the
log-probability of \texttt{name="v")} after the open call. Under it the
1.7B stdlib control rises from $0/56$ to $17.9\%$ against a $10.6\%$ chance
floor---discriminating, weakly---and the injection runs at 4B and 8B scored
both ways (Table~\ref{tab:apiusage}).

The corpus itself turned out to matter more than the size. Stated as three
signature listings per symbol, the library was memorized (training loss
$0.33$) and not extractable: the recitation probe, which asks for the
$n$th parameter by position, scored $2.7\%$ at 4B. Stating the order in
prose as well---an enumeration and a ``takes $a$ first, then $b$'' chain
per symbol, never the probe's own wording
(\texttt{experiments/augment\_api\_corpus.py})---raises the same probe
to $35$, $31$ and $31\%$ at 1.7B, 4B and 8B: the knowledge-augmentation
effect of~\cite{physics33}, met on a real corpus. Under the log-probability
rule the served models then rank the true name at $28.6$, $37.9$ and
$37.4\%$ against $19.2$, $14.3$ and $14.8\%$ released---$2.7$, $3.6$ and
$3.5\times$ the $10.6\%$ chance floor; the stdlib control moves the other way at the two
smaller sizes ($25.0\to17.9\%$ at 4B), which is the forgetting this fill
costs on a library it was never trained on. The generate rule stays at the
floor for every served model below 27B: a base model that cannot follow
``set only the second parameter'' cannot show what it knows that way, and
the table reports both.

\begin{table}[t]\centering\small
\setlength{\tabcolsep}{4pt}
\begin{tabular}{lrrrrrrrr}
\toprule
 & \multicolumn{4}{c}{generate (instruction)} & \multicolumn{4}{c}{log-prob rank (chance 10.6\%)} \\
 & \multicolumn{2}{c}{stdlib} & \multicolumn{2}{c}{cyclopts} & \multicolumn{2}{c}{stdlib} & \multicolumn{2}{c}{cyclopts} \\
model & released & served & released & served & released & served & released & served \\
\midrule
1.7B & 0.0\% & 1.8\% & 0.5\% & 0.5\% & 17.9\% & 10.7\% & 19.2\% & 28.6\% \\
4B & 1.8\% & 0.0\% & 3.3\% & 1.1\% & 25.0\% & 17.9\% & 14.3\% & 37.9\% \\
8B & 7.1\% & 0.0\% & 8.8\% & 0.0\% & 28.6\% & 26.8\% & 14.8\% & 37.4\% \\
27B & 21.4\% & -- & 19.2\% & -- & -- & -- & 21.4\% & 38.5\% \\
G31 & 0.0\% & -- & 0.0\% & -- & -- & -- & -- & -- \\
32B & 17.9\% & -- & 6.0\% & -- & -- & -- & -- & -- \\
\bottomrule
\end{tabular}

\caption{The API usage test under two scorers. \emph{generate}: a call
completed in code so that only the $n$th parameter is set, parsed and bound
against the real signature; \emph{log-prob rank}: the true name ranked
among the symbol's own names plus four borrowed ones, without the
instruction. ``stdlib'' is the control the model learned in pretraining;
``cyclopts'' the injected library, verified absent. ``served'' is the fill
trained on the augmented corpus. Where the generate control fails, its
cyclopts number measures instruction following rather than knowledge and is
not read.}
\label{tab:apiusage}
\end{table}

\subsubsection{Is the served model still a good model?}\label{sec:downstream}
Everything above prices knowledge against perplexity, on the argument that
LAMBADA is held out from rehearsal. Perplexity is still a proxy, and the
question a deployment asks is whether the model it serves is intact. We run
the standard multiple-choice suite through lm-eval-harness on the served
weights (Table~\ref{tab:downstream}).

\begin{table}[t]\centering\small
\setlength{\tabcolsep}{4pt}
\begin{tabular}{llrrrr}
\toprule
model & arm & ARC-c & ARC-e & HellaSwag & WinoGrande \\
\midrule
Qwen3-1.7B & released fp32 & 45.1 & 68.4 & 66.5 & 64.5 \\
 & 4-bit anchor & 44.7 & 69.4 & 64.7 & 62.9 \\
 & served & 46.0 & 71.4 & 57.9 & 60.1 \\
 & \emph{update cost} & $\mathbf{+1.3}$ & $\mathbf{+2.0}$ & $\mathbf{-6.8}$ & $\mathbf{-2.8}$ \\
\midrule
Qwen3-4B & 4-bit anchor & 51.1 & 75.7 & 72.3 & 69.1 \\
 & served & 54.0 & 76.4 & 66.2 & 63.8 \\
 & \emph{update cost} & $\mathbf{+2.9}$ & $\mathbf{+0.8}$ & $\mathbf{-6.1}$ & $\mathbf{-5.4}$ \\
\midrule
Qwen3-8B & 4-bit anchor & 54.0 & 78.4 & 78.0 & 71.0 \\
 & served & 61.0 & 82.7 & 70.9 & 68.7 \\
 & \emph{update cost} & $\mathbf{+7.0}$ & $\mathbf{+4.3}$ & $\mathbf{-7.1}$ & $\mathbf{-2.3}$ \\
\midrule
Qwen3.8-27B & released fp32 & 58.3 & 73.1 & 82.8 & 75.7 \\
 & 4-bit anchor & 58.4 & 76.8 & 82.6 & 76.3 \\
 & served & 66.0 & 86.9 & 80.1 & 73.4 \\
 & \emph{update cost} & $\mathbf{+7.6}$ & $\mathbf{+10.1}$ & $\mathbf{-2.5}$ & $\mathbf{-2.9}$ \\
\midrule
Qwen3-32B & 4-bit anchor & 61.5 & 83.2 & 81.9 & 72.8 \\
 & served & 63.8 & 84.9 & 77.8 & 73.6 \\
 & \emph{update cost} & $\mathbf{+2.3}$ & $\mathbf{+1.7}$ & $\mathbf{-4.1}$ & $\mathbf{+0.9}$ \\
\bottomrule
\end{tabular}

\caption{Standard multiple-choice benchmarks on the served model, against
the 4-bit anchor it preserves and, at 1.7B and 27B where we also ran it, the
dense release (fp16 at 1.7B, bf16 at 27B).
\emph{update cost} is served minus anchor: the part attributable
to the update rather than to quantization. Accuracies are
length-normalised where the task provides it.}
\label{tab:downstream}
\end{table}

Two costs are usually conflated here and only the second is ours.
Quantization alone costs $1.8$ points of HellaSwag and $1.6$ of WinoGrande
at 1.7B before any update happens; a table that reported only
released-versus-served would charge that to the method.

Against the anchor it preserves, the served model is \emph{better} on both
reasoning benchmarks at every size and worse on HellaSwag at every size, with
WinoGrande falling at four sizes of five and rising $0.9$ at 32B.
The split is not noise dressed up as a finding: it lines up with the
perplexity result, where LAMBADA---also a continuation task---degrades while
in-domain perplexity does not. Whatever the update spends, it spends on
next-token coherence rather than on the ability to answer a question, and
two independent metric families say so.

The direction the split moves in is what one operating point could not show,
and the five model sizes do not all say the same thing. Across 1.7B, 4B and
8B the HellaSwag cost is flat ($-6.8$, $-6.1$, $-7.1$) and the
ARC-challenge gain grows ($+1.3$, $+2.9$, $+7.0$); the two larger rows break
the first pattern and dent the second, with HellaSwag costing only $-2.5$ at
27B and $-4.1$ at 32B, and ARC-challenge gaining $+7.6$ then falling back to
$+2.3$. WinoGrande shows no trend at any size ($-2.8$, $-5.4$, $-2.3$,
$-2.9$, $+0.9$). Every cell is a single run and the 27B row is on a self-quantized anchor
(the 32B is a published release), so the honest reading is narrow: the cost
never exceeds seven and a half points of HellaSwag and the gain is positive on
ARC-challenge at every size, while the shape in between is not resolved by
five points.
We do not have an account of why the reasoning benchmarks should improve at
all, and we are wary of one---rehearsal is a plausible confound, and these
runs share a rehearsal fraction.

\subsection{Consolidate and rehearse: running the cycle}\label{res:lifelong}
\subsubsection{Rehearsal: fresh or fatal, and what it does not buy}
\label{sec:rehearsal}
A fixed replay buffer repeated across 24 epochs is itself memorized and
sharpens the model onto its own rehearsal set: test perplexity explodes from $24.6$ to $172$ for the unmerged adapter,
and from $22.6$ to $63.8$ after clip-merge. Redrawing rehearsal snippets every epoch removes
in-domain forgetting entirely: the merged model reaches WikiText PPL 10.78,
below its own 4-bit anchor (11.71), while absorbing 17.4\% of the facts; a
10\% replay fraction dominates 30\% on both axes (recall
$17.4\!\to\!25.1\%$, PPL $10.78\!\to\!10.31$).

\textbf{This in-domain gain is a rehearsal effect, not a property of the
dequantization gap, and it does not survive a domain shift.} Two controls
establish this and we state them prominently because an earlier draft of
this work drew the opposite conclusion. First, the \emph{unmerged} adapter
---which is not constrained to the cells at all---scores essentially the
same in-domain perplexity as the box-constrained merge (10.299 vs.\ 10.271
on seed 1; the gap stays under $0.03$ across the three 10\%-rehearsal
seeds and never exceeds $0.29$ in any fresh-rehearsal run), so the
improvement cannot be attributed to in-cell storage. Second, rehearsal is
drawn from WikiText-train while the perplexity metric is WikiText-test:
on held-out LAMBADA every method \emph{degrades} relative to the anchor
(28.67), and along the projected paths monotonically in how much it absorbs
(Table~\ref{tab:frontier}); CellFill $r{=}64$ degrades less than the
projected path that absorbs less than it does, which is the one place the
trade-off is beaten rather than traversed. The learning--forgetting trade-off is real; the
free lunch was an artifact of measuring rehearsal on its own domain.

\textbf{The constructive form of the same finding.} If rehearsal drawn
from the metric's domain hides the cross-domain cost, rehearsal drawn from
somewhere else should reduce it. Table~\ref{tab:rehearsal} changes nothing
but the rehearsal source on the real corpus and the published 1.7B anchor.

\begin{table}[t]\centering\small
\setlength{\tabcolsep}{4pt}
\begin{tabular}{lrrrr}
\toprule
rehearsal drawn from & recall (507) & WikiText & LAMBADA ppl$\downarrow$ &saturation \\
\midrule
WikiText (default) & $94.7\!\pm\!1.1$\% & $11.85\!\pm\!0.09$ & $46.0\!\pm\!5.9$ & 1.63\% \\
Pile & $95.3\!\pm\!1.2$\% & $13.64\!\pm\!0.36$ & $35.2\!\pm\!0.5$ & 1.74\% \\
C4 & 94.7\% & 13.26 & 35.3 & 1.85\% \\
mixture & $95.4\!\pm\!0.2$\% & $12.16\!\pm\!0.08$ & $35.4\!\pm\!0.3$ & 1.72\% \\
\bottomrule
\end{tabular}

\caption{Rehearsal distribution, real corpus, published NF4 Qwen3-1.7B,
$r{=}64$, rehearsal fraction $0.1$ throughout. ``mixture'' draws equally
from the three sources. Cells with $\pm$ are over three seeds.}
\label{tab:rehearsal}
\end{table}

Recall does not move---every source lands within the seed spread of the
default---while the cross-domain cost drops by nearly a quarter: LAMBADA
from $46$ with WikiText rehearsal to $35$ with any source that shares
nothing with either metric, and the in-domain ``gain'' of the default goes
with it (WikiText $11.85\to12.16$--$13.64$). Nearly two thirds of the
excess over the anchor that this paper has been calling the cross-domain
cost of knowledge was the cost of rehearsing on WikiText. The mixture is
the operating point we would now recommend: it keeps WikiText within
$0.32$ of the default and takes the LAMBADA cost from $1.61\times$ the
anchor to $1.23\times$.

\subsubsection{Four sequential updates, one artifact: forgetting is priced,
not eliminated}
\label{sec:sequential}
Forgetting is priced rather than eliminated once updates arrive in
sequence: each of four updates consumed a constant fraction of the room the
previous one left ($\beta = 0.825\pm0.002$ in this sequence), and the first
task retained $31\%$ of its recall after the three that followed. A constant
fraction makes lifetime capacity finite and measurable: without consolidation
it is $a_1/(1-\beta) = 292\%$ of a first-task equivalent at this sequence's own
$\beta$, and $231$--$339\%$ over the band the matched sequences span, $90\%$ of
it spent within twelve updates, and no constraint is violated anywhere on the
way. Every result
so far comes from a single injection; the sequence here is four disjoint
500-fact tasks arriving in order, each written into the same residual with
CellFill ($r{=}64$), with the artifact verified after every task. What
``disjoint'' does and does not mean here should be said plainly, because it
bounds every sequential claim in this paper: the tasks share no entity, and
they are slices of one draw from one generator, so they share a template, a
vocabulary and a distribution. These sequences measure interference and
capacity under repeated writing. They do not measure distribution shift, and
nothing in this paper does. This
sequence and the one in \S\ref{sec:rehearse-old} run on a 4-bit file we
quantized ourselves; only the cycle of \S\ref{sec:loop} runs on a published
release. Between
tasks the accumulated in-cell position becomes the next task's anchor and the
remaining room shrinks to the distance from there to the cell walls, so the
budget is literally consumed as the sequence proceeds.

\begin{center}\small
\begin{tabular}{lccccrrcc}
\toprule
after & \multicolumn{4}{c}{recall on task} & \multicolumn{2}{c}{perplexity} & mean room & artifact \\
 & T0 & T1 & T2 & T3 & WikiText & LAMBADA & remaining & violations \\
\midrule
T0 & 50.9\% & --- & --- & --- & 11.21 & 39.42 & $3.01\times10^{-3}$ & 0 \\
T1 & 33.5\% & 42.7\% & --- & --- & 11.39 & 41.52 & $2.49\times10^{-3}$ & 0 \\
T2 & 19.9\% & 17.5\% & 35.8\% & --- & 11.45 & 41.54 & $2.05\times10^{-3}$ & 0 \\
T3 & 15.7\% & 13.1\% & 18.5\% & 29.7\% & 11.62 & 43.57 & $1.69\times10^{-3}$ & 0 \\
\bottomrule
\end{tabular}
\end{center}

Four things happen at once, and only the first is good news. The artifact
survives the whole sequence bit-for-bit---zero violations after every one of
the four tasks, over $1.4\times10^9$ weights each time---so the guarantee is
not a single-shot property. Cross-domain ability degrades monotonically and
we report it here rather than only the recall columns: LAMBADA runs
$39.42\to43.57$ across the sequence, which against the $28.67$ anchor is a
$37\%$ cost for the first task rising to $52\%$ after the fourth, while
in-domain perplexity moves only $11.21\to11.62$---the same divergence between
the two metrics that \S\ref{sec:rehearsal} attributes to rehearsal sharing a
corpus with WikiText. We have no unconstrained control sequence to price that
$52\%$ against, so it is a measurement of this method's cost over four tasks
and not a comparison. Old tasks are forgotten: T0 falls
$50.9\to33.5\to19.9\to15.7\%$ as the three later tasks arrive, the largest
drop at the first of them. And, distinctively for this
setting, \emph{plasticity itself decays}: each task learns less than the one
before ($50.9\to42.7\to35.8\to29.7\%$) as the writable room falls by $44\%$.
The second effect is ordinary catastrophic forgetting; the third is specific
to learning in a bounded space.

\paragraph{The decay is geometric, and it is predicted rather than fitted.}
Proposition~\ref{prop:decay} says a bounded update consumes a constant
\emph{fraction} of the remaining room, because the displacement is scaled by
the room itself. The measured ratios are
$\bar r_{k}/\bar r_{k-1} = 0.8271,\,0.8254,\,0.8238$: geometric to three
decimal places, $\beta = 0.825\pm0.002$. New-task absorption follows the same
law, $0.8377,\,0.8391,\,0.8305$, and is proportional to the room available
before the task ($a_k/\bar r_{k-1} = 142,\,144,\,145$ for $k\ge2$).

A sharper claim is available here and measurement refutes it. If
$\beta=1-\mathbb{E}|t|$ were an identity, the implied
$\mathbb{E}|t|=1-\bar r_k/\bar r_{k-1}$ would have to be constant across
tasks, and it is ($0.1729,\,0.1746,\,0.1762$). That is not independent
evidence: those three numbers are one minus the three ratios above, so their
constancy and $\beta$'s are the same statement. The test requires
$\mathbb{E}|\tanh|$ measured from the fill actually applied at each fold,
which we now log: $0.4429,\,0.4617,\,0.5027,\,0.5180$ in the second
archived run, whose own ratios are $0.8173,\,0.8205,\,0.8314$---neither
equal to the implied value ($2.8\times$ apart) nor constant (it rises $17\%$
as the room falls). The identity is therefore false as an identity, exactly
as Proposition~\ref{prop:decay} states: it proves $r'\ge r(1-|t|)$, and the
data confirm the inequality is strict and loose, $0.82$ against the $0.51$
the naive law predicts.

The slack is not the NF4 cells giving room back through their asymmetry,
which is the natural explanation and is wrong. A simulation the frozen
artifact makes possible without a GPU
(\texttt{scripts/sim\_fold\_slack.py}, archived as
\texttt{results/sim\_fold\_slack.json}) settles it. Folding $2\times10^6$ weights
with the four $\mathbb{E}|\tanh|$ values logged in the second archived
sequence and random signs gives ratios in $0.80$--$0.84$ from the second
fold on \emph{whether the cells are NF4, with codes drawn from a Gaussian
weight distribution, or equal-width and perfectly symmetric}; the level
table leaves a footprint only at the first fold ($0.61$--$0.63$ against
$0.56$--$0.57$), and both archived sequences sit inside the symmetric-cell
band. The relaxation belongs to the folding operator, not to the grid:
after one fold every position is off-centre, and a displacement of random
sign then has roughly even odds of moving toward the far wall, so
$\min(a,b)$ loses less than $|t|$ of the room in expectation. It also answers
why $\beta$ holds steady while
$\mathbb{E}|\tanh|$ rises $17\%$: in the simulation the saturated $\min$
makes the ratio roughly a third as elastic in $\mathbb{E}|\tanh|$ as the
naive law, and its exact value depends on the \emph{distribution} of $|t|$
and not only its mean, which is why we quote a band rather than a constant.
Because the operator
does not care what grid the cells came from, the law should transfer to
uniform integer grids from the second fold on; we have not run a sequence
on one.

\paragraph{A control in which the room never decays at all.} Folding is one of
two ways to write a sequence into the cells, and the other one isolates what
the decay costs. Under \emph{folding} the accumulated position becomes the next
task's anchor, so the reachable interval shrinks to the distance from there to
the walls. Under the \emph{fixed-cell} parametrization each cell keeps the
centre and radius the release gave it for the whole sequence, and the trainable
position is an absolute coordinate inside it, so the whole cell stays reachable
at every task and the room never decays---at the price of the tasks sharing one
coordinate rather than accumulating into it. Run as one more arm of the six-task \texttt{exp70} sequence, it starts where
that sequence's folding arm starts ($54.6$ against
$54.8\%$ on the first task, so the two parametrizations have the same
plasticity to begin with) and ends at $27.1\%$ against $18.2\%$. Removing
geometric decay outright therefore recovers about a quarter of what the
sequence loses and leaves three quarters of it in place, which is the first
evidence in this paper that the room is not the main thing a long sequence runs
out of. \S\ref{sec:loop} takes that further.

What survives is the empirical law: the room decays geometrically, and that
is what makes lifetime capacity finite \emph{under sequential folding}. The
four-task sequence gives $\beta$ in $0.82$--$0.83$; a third run, six tasks on
a third host (\texttt{exp70}), gives ratios of $0.813,\,0.827,\,0.846,\,
0.849,\,0.843$, so the band over all three archived sequences is
$0.78$--$0.85$ across the four matched $r{=}64$, $s{=}40$ folding sequences,
rather than the tighter figure a single sequence suggests. The constant is a
property of the recipe as much as of the method: the gentler $r{=}32$,
$s{=}10$ arm barely draws on the room at all and runs at $0.94$--$0.96$, and
the inert-regularized arm at $0.83$--$0.88$. In
that run $\mathbb{E}|\tanh|$ rises from $0.45$ to $0.55$ and plasticity falls
$54.8\to45.8\to34.3\to26.2\to21.9\to18.2\%$---the ratio creeping up as the
fill saturates, which is the direction the simulation's saturated $\min$
predicts, though its top three folds ($0.846$, $0.849$, $0.843$) sit just
above the simulated band's upper end of $0.835$. The first two archived
sequences were run on different hosts (their anchor-only perplexities differ before the first
gradient, $11.3589$ against $11.3625$) and their three ratios drift in
opposite directions, $0.827\to0.824$ and $0.817\to0.831$; we have no
same-environment repeat, and the ``stable to $1.7\%$'' of the earlier draft
was the spread between those two runs, not a within-environment figure.

The lifetime capacity quoted at the head of this subsection follows from
that law rather than from a fit. This is what we mean by the budget being operational: it can be
estimated after two tasks and it predicts when the model will stop being
able to learn.

\paragraph{Consolidation renews the budget, at an explicit price.}
Repeating the sequence with a re-quantization after T1---a deliberate major
version, not silent drift---restores the room and with it the plasticity. The
restoration is measured one write later, because the room is logged after each
task and the major fires between them: the row before the consolidation reads
$2.49\times10^{-3}$ and the row after it $2.98\times10^{-3}$, against the
$3.01\times10^{-3}$ the first task left, so a major version returns at least
$99\%$ of the room even after the next task has drawn on it. What it hands
back at the instant it fires is not logged. The plasticity follows: the fourth task
recovers $33.4\%$ instead of $29.7\%$, and for the first time in the
sequence a later task outperforms its predecessor ($32.8\to33.4\%$). The
price is stated exactly: $19.0\%$ of the 4-bit codes change at that step,
which is precisely why it must be a version bump rather than an update, and
the oldest task's retention drops further ($15.7\to12.2\%$).

Consolidation is not free in the short run either, and the cumulative
absorption shows why. Re-quantizing injects fresh quantization error, so the
task immediately after it does \emph{worse} than it would have without
($32.8\%$ against $35.8\%$), and only the following task pulls ahead
($33.4\%$ against $29.7\%$). Cumulatively the two policies cross between the
third and fourth task ($129.4\%$ vs.\ $126.4\%$, then $159.1\%$ vs.\
$159.8\%$): consolidation takes about two tasks to repay its own cost. A
policy consolidating every $m$ tasks therefore has to be chosen against the
sequence length, and Proposition~\ref{prop:decay} gives the throughput to
compare---$46.5\%$ per task at $m{=}2$ against $39.1\%$ at $m{=}4$, versus
zero in the limit for a policy that never consolidates.

Minor versions keep the artifact and spend the budget; a major version buys
the budget back by giving up the artifact. We report this from a single
sequence on one model, so the constant $\beta$ should be read as a
demonstration that the law is measurable, not as a value that transfers.

\subsubsection{Rehearsing earlier tasks recovers most of what the sequence
forgets}\label{sec:rehearse-old}
The sequence above forgets because each task is trained on its own corpus
alone. Replaying a sample of the earlier tasks' corpora alongside the new one
changes the outcome more than any other intervention we have measured. On the
six-task sequence, with half of each task's training material drawn from the
tasks that preceded it (\texttt{--rehearse-old 0.5}), the first task retains
$84.8\%$ of its recall after five subsequent updates, against $19.3\%$ for the
same sequence without it. The replayed pool is every earlier task's own
sentences, shuffled once and cycled: at the second task each of the $500$
earlier sentences returns about twelve times over the run, at the sixth each
of $2{,}500$ returns about twice. That is deliberately the regime
\S\ref{sec:rehearsal} warns against for a \emph{generic} buffer---this
material is the material we want memorized---and the generic replay those
runs also carry is drawn fresh every epoch as it is everywhere else in this
paper.

\begin{center}\small
\setlength{\tabcolsep}{3.2pt}
\begin{tabular}{lcccccccc}
\toprule
 & \multicolumn{6}{c}{retention after the sixth task (\% of what the task
 first learned)} & \multicolumn{2}{c}{perplexity} \\
\cmidrule(lr){2-7}\cmidrule(lr){8-9}
sequence & T0 & T1 & T2 & T3 & T4 & T5 & WikiText & LAMBADA \\
\midrule
no rehearsal & 19.3 & 23.6 & 34.6 & 43.5 & 58.4 & 100 & 12.45 & 50.0 \\
rehearsal $0.5$ & \textbf{84.8} & \textbf{100.7} & \textbf{97.2} & \textbf{91.5} & \textbf{92.4} & 100 & 12.42 & 50.7 \\
\quad at $r{=}32$, $s{=}10$ & 82.4 & 95.2 & 92.5 & 97.1 & 82.6 & 100 & 10.26 & 33.4 \\
\bottomrule
\end{tabular}
\end{center}

Three things about this deserve to be stated separately, because only the
first is unambiguously good.

\textbf{The recovery is large and it is free in perplexity.} Summed over all
six tasks, recall at the end of the sequence is $1.859$ against $0.757$, a
factor of $2.45$, at a WikiText perplexity of $12.42$ against $12.45$ and a
LAMBADA of $50.7$ against $50.0$. Whatever rehearsal costs, it is not paid in
the cross-domain metrics that the rest of this paper uses to price knowledge.
The T1 entry above exceeds $100\%$: that task ends the sequence marginally
better recalled than when it was written, so replay is not only slowing decay
but partly restoring what interference had taken.

\textbf{It is paid in plasticity instead.} New-task absorption falls from
$54.8$, $45.8$, $34.3$, $26.2$, $21.9$, $18.2\%$ without rehearsal to
$54.8$, $38.1$, $33.6$, $30.5$, $23.7$, $18.5\%$ with it. Half of each task's
budget now goes to material the model has already seen, and the second task
pays the largest share of that. The exchange is favourable here only because
what is bought (five tasks' retention) is larger than what is spent (one
task's absorption), and we have not established that the exchange stays
favourable at longer sequence lengths or smaller rehearsal fractions.

\textbf{Against a ceiling the deployment setting cannot reach.} Training the
same model on all $3{,}000$ facts at once, under identical hyperparameters
and seed, recalls $61.2\%$ (\texttt{exp73}). The six-task sequence with
rehearsal averages $31.0\%$ over the same facts and without rehearsal
$12.6\%$: replay recovers roughly half of the joint-training optimum where
sequential updating alone recovers a fifth. We report the joint number as the
ceiling it is and not as a competitor, because it is unavailable in the
setting this paper is about. It requires every earlier corpus to still be
held at every later update, and it produces a new artifact each time, which
is the cost the frozen file exists to avoid. Rehearsal needs only a sample of
the earlier material, which an operator who wrote it has by construction.

\paragraph{One weight left its cell, and the gentler recipe does not.} Both
$r{=}64$, $s{=}40$ arms above report a single constraint violation at the
sixth fold, one weight of $1{,}409{,}286{,}144$ checked, after five clean
folds. We report it rather than round it to zero: the guarantee in this paper
is an integer identity, and an identity that fails once has failed. The third
row of the table is the same sequence at rank $32$ and $s{=}10$, the setting
\S\ref{sec:capacity-xdom} identifies as the one that keeps the link in its
gradient-carrying range, and it holds the retention profile with zero
violations at every one of the six folds, at a WikiText of $10.26$ and a
LAMBADA of $33.4$ against the anchor's $28.67$. It absorbs about half as
much ($0.983$ summed against $1.859$). The violation is therefore a property
of driving the link deep into saturation over repeated folds and not of
sequential folding itself, which is consistent with the saturation mechanism
of \S\ref{sec:sequential}; a deployer who wants the identity to hold across
a long sequence should pay for it in absorption rather than in margin.

\subsubsection{The cycle run as a loop}\label{sec:loop}
\begin{figure}[t]\centering
\includegraphics[width=\linewidth]{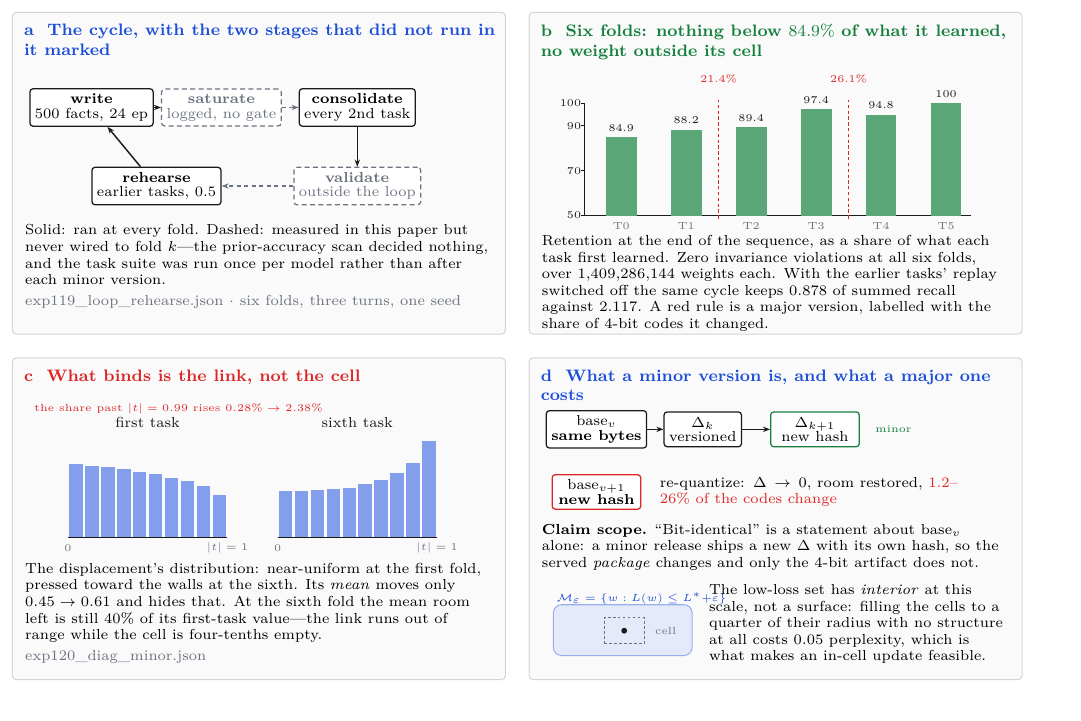}
\caption{The cycle as it ran. \textbf{a}, The five stages, with the two that
were measured in this paper but never wired into the loop drawn dashed---the
prior-accuracy scan decided nothing, and the task suite was run once per model
rather than after each minor version. \textbf{b}, Six folds on a published
release (unsloth NF4 Qwen3-1.7B, the run of \S\ref{sec:loop}; the 8B
replication is tabulated in \S\ref{sec:loop8b}): what each task retains of
what it first learned, the two major versions and the share of 4-bit codes
each changed, and no weight outside its cell at any fold. \textbf{c}, Why a later task absorbs less. The
displacement's distribution is near-uniform at the first fold and pressed
against the cell walls at the sixth, while its mean barely moves and the cell
is still four-tenths empty. \textbf{d}, What a release is under the contract,
and the geometric fact it rests on: the low-loss set has interior at this
scale rather than being a surface.}
\label{fig:lifecycle}
\end{figure}

Everything to this point measures one stage of the deployed life at a time:
a write, a saturation, a consolidation, a round of rehearsal. An operator
does not get to run them separately. This experiment runs them as the loop
the title names---write, saturate, consolidate, rehearse, validate,
repeat---on six 500-fact tasks, replaying half of each task's material from
the tasks before it, and consolidating after every second task, so the cycle
turns three times (\texttt{exp119}, $r{=}64$, $s{=}40$, one seed;
Figure~\ref{fig:lifecycle}). Unlike
the two sequences above, which run on a 4-bit file we quantized ourselves,
this one runs on unsloth's published NF4 release of Qwen3-1.7B---so it is
the paper's only sequential result on an artifact somebody else shipped, and
the codes it returns at every fold are that vendor's.

\begin{center}\small
\setlength{\tabcolsep}{3.2pt}
\begin{tabular}{lcccccccccc}
\toprule
after & \multicolumn{6}{c}{retention (\% of what the task first learned)} &
\multicolumn{2}{c}{perplexity} & room & viol. \\
\cmidrule(lr){2-7}\cmidrule(lr){8-9}
 & T0 & T1 & T2 & T3 & T4 & T5 & WikiText & LAMBADA & remaining & \\
\midrule
T0 & 100.0 & --- & --- & --- & --- & --- & 11.17 & 39.5 & $3.01\times10^{-3}$ & 0 \\
T1 & 98.0 & 100.0 & --- & --- & --- & --- & 11.51 & 42.0 & $2.33\times10^{-3}$ & 0 \\
\multicolumn{11}{l}{\footnotesize\emph{consolidate: $21.41\%$ of the 4-bit codes change}} \\
T2 & 82.8 & 95.1 & 100.0 & --- & --- & --- & 12.73 & 44.2 & $2.61\times10^{-3}$ & 0 \\
T3 & 84.9 & 91.5 & 94.1 & 100.0 & --- & --- & 12.72 & 44.3 & $2.04\times10^{-3}$ & 0 \\
\multicolumn{11}{l}{\footnotesize\emph{consolidate: $26.14\%$ of the 4-bit codes change}} \\
T4 & 84.2 & 89.0 & 94.4 & 98.0 & 100.0 & --- & 14.78 & 54.0 & $2.42\times10^{-3}$ & 0 \\
T5 & \textbf{84.9} & 88.2 & 89.4 & 97.4 & 94.8 & 100.0 & 14.46 & 57.3 & $1.88\times10^{-3}$ & 0 \\
\bottomrule
\end{tabular}
\end{center}

\textbf{The loop is clean at every fold, and we cannot show that
consolidating is why.} Every one of the six folds re-quantizes to the shipped
codes over all $1{,}409{,}286{,}144$ constrained weights, with no exception.
It is tempting to read that against the one weight that left its cell in
\S\ref{sec:rehearse-old} and conclude that a major version relieves the
saturation which caused it, and an earlier version of this paper did. The
archive does not support it. Counting every six-task arm at $r{=}64$,
$s{=}40$: two of the five that never consolidate had a violation, and neither
of the two that consolidate did---$0/2$ against $2/5$, which is $p=0.48$ by
Fisher's exact test and no signal at all. Nor does displacement order the
failures. Sorted by the mean $|\tanh|$ each arm ends at, the two violations
sit at $0.549$ and $0.614$ with clean arms interleaved on both sides, at
$0.446$, $0.579$ and $0.753$; \S\ref{sec:loop} makes the same point from the
controller side. The account that survives is the one the introduction gives:
with $|\tanh|<1$ and the room recomputed against the frozen walls at every
fold, escape is impossible in exact arithmetic, so the single escape was a
floating-point event at the margin---and consolidation does not touch that.
What the loop establishes is that six folds and three major versions can pass
with the identity intact, not that consolidating is what keeps it so.

\textbf{The room is renewed, and one can watch it happen.} Writable room falls
within each pair of tasks and is higher again on the far side of every
consolidation ($2.33\!\to\!2.61\times10^{-3}$ across the first,
$2.04\!\to\!2.42\times10^{-3}$ across the second), so the geometric draw-down
of \S\ref{sec:sequential} restarts instead of running out. As in
\S\ref{sec:sequential} the far-side figure already includes a task's write,
so it understates what the major returned. Retention is flatter than any sequence we
have run: no task ends below $84\%$ of what it first learned, and summed
recall at the end is $2.117$ against $1.859$ for the same six tasks with
rehearsal and no consolidation.

\textbf{The price of a major version rises with depth, and the cycle is not
free.} The first consolidation changed $21.41\%$ of the 4-bit codes and the
second $26.14\%$: the $19.0\%$ reported in \S\ref{sec:sequential} is the
first consolidation of a shorter sequence, not a constant, and an operator
budgeting for a lifetime of releases should expect the per-consolidation cost
to grow. Held-out ability is where the rest of the bill lands. LAMBADA ends
at $57.3$ against the anchor's $28.57$ and against $50.7$ for the same six
tasks rehearsed but never consolidated, and WikiText at $14.46$ against
$12.42$. Consolidating twice bought renewed room, a flatter retention profile
and six clean folds, and paid for them in cross-domain perplexity. We report
one seed and one cadence: how the exchange moves with the consolidation
interval is the parameter a deployment would actually tune, and we have not
swept it.

\textbf{Consolidation and rehearsal do different jobs, and neither substitutes
for the other.} Running the same six tasks, the same cadence, the same anchor and the same
seed with the replay of \emph{earlier tasks} switched off isolates each. Both
arms keep the generic WikiText replay that every run in this paper uses
($10\%$ of each epoch); what differs is only whether the earlier tasks' own
material comes back. The invariant does
not need rehearsal: both arms return the shipped code at every one of the six
folds, so the identity does not depend on replay. Whether it depends on
consolidating is a separate question and the counts do not answer it
(above: $0/2$ against $2/5$, $p=0.48$). The knowledge does need replay: summed recall at the end is
$2.117$ with it and $0.878$ without, a factor of $2.41$, and the first task
ends at $84.9\%$ of what it learned against $16.7\%$.

\begin{center}\small
\setlength{\tabcolsep}{3.6pt}
\begin{tabular}{lcccccc}
\toprule
 & \multicolumn{5}{c}{retention of tasks T0--T4 after the sixth} & summed \\
\cmidrule(lr){2-6}
six tasks, consolidating after T1 and T3 & T0 & T1 & T2 & T3 & T4 & recall \\
\midrule
rehearsal $0.5$ & \textbf{84.9} & \textbf{88.2} & \textbf{89.4} & \textbf{97.4} & \textbf{94.8} & \textbf{2.117} \\
no rehearsal    & 16.7 & 17.3 & 24.1 & 26.8 & 68.8 & 0.878 \\
\bottomrule
\end{tabular}
\end{center}

The trade the unconsolidated sequence reported---that replay is paid for in
plasticity---does not survive unchanged into the cycle. Without rehearsal the
model absorbs more immediately after each consolidation ($46.2$ against
$40.8\%$ at T2, $39.9$ against $34.7\%$ at T4) and \emph{less} by the end
($23.6$ against $27.9\%$ at T5). What rehearsal does cost here is cross-domain
ability---LAMBADA $57.3$ against $52.6$, both far above the $28.57$
anchor---and a more expensive major version,
because moving more weight changes more codes at the next re-quantization
($21.4$ and $26.1\%$ against $19.3$ and $21.3\%$).

\textbf{Most of the plasticity loss is not the room running out.} Absorption
falls to $53\%$ of the first task's with rehearsal and $45\%$ without, while
the writable room never drops below $62\%$ and rises at every consolidation, so
neither arm is short of space. Removing the replay load does not stop the decay
either: the two arms differ by $1.4$ points of mean absorption across the six
tasks, and the unrehearsed one declines faster. The cleanest test is not
consolidation at all but the fixed-cell control of \S\ref{sec:sequential},
which holds each cell's centre and radius constant for the whole sequence so
that room \emph{never} decays and no re-quantization noise is injected. Its
first task absorbs what the ordinary arm's does ($54.6$ against $54.8\%$), and
its absorption still falls by half ($54.6\to27.1\%$). Removing geometric decay
outright recovers about a quarter of the ordinary sequence's loss
($a_6/a_1=0.497$ against $0.332$) and leaves the rest standing.

The paper's own proportionality gives the same conclusion on one anchor and
refuses it on another, which is itself worth reporting. Absorption per unit of
room available, $a_k/r_{k-1}$, is flat over four tasks ($141.8$, $143.9$,
$144.8$), the observation \S\ref{sec:sequential} builds the decay law on. Over
six tasks on a file we quantized ourselves it falls,
$151.0\to139.3\to128.6\to127.2\to124.4$; over six tasks of the same recipe and
the same policy on unsloth's published release it \emph{rises},
$141.8\to157.1\to173.6\to160.7\to188.0$, and that arm ends at
$a_6/a_1=0.527$ against the self-quantized one's $0.332$. The decay is
therefore not a constant of the method: two anchors of the same base model,
one recipe, no replay in either, differ by a factor of $1.6$ in how much
plasticity survives six tasks. We have no account of it, we did not expect it,
and it means the ``less per unit of room'' reading above holds for the
self-quantized anchor and not for the published one.

\textbf{What binds is the link, not the cell.} Separating the candidates needs
a run that keeps the \emph{distribution} of the displacement rather than its
mean, which the sequences above do not; the last experiment does
(\texttt{exp120}: the same six tasks, no consolidation, no replay of earlier
tasks, the fold instrumented). Two things come out of it, and only the second
was expected.

Interference between consecutive tasks is not the mechanism. The cosine
between task $k$'s displacement and task $k{-}1$'s stays between $-0.010$ and
$-0.016$ over all five comparisons, and the fraction of coordinates whose sign
disagrees runs $0.5031$ to $0.5051$---chance to three decimals---without rising
as the sequence deepens. Destructive interference between neighbours would show
a clearly negative cosine and a conflict fraction above one half, and neither
appears. This bounds opposition between \emph{adjacent} tasks only: each task
is compared with its predecessor, not with everything already written, and the
second comparison is the one we would make next.

Saturation is the mechanism, and the mean was hiding it. $\mathbb{E}|\tanh|$
rises only $0.450\to0.579$ across the six folds, under a third; over the same
folds the
share of coordinates past $|t|{=}0.99$ rises $0.28\to2.38\%$, a factor of
$8.4$, and the $|t|>0.95$ tail $2.4\to10.2\%$, both accelerating at the last
fold. The distribution changes shape rather than shifting: near-uniform at the
first task, clearly right-skewed at the sixth. And at that sixth fold the mean
room remaining is still $40\%$ of its first-task value, so the link is running
out of range while the cell is still four-tenths empty. That is why removing
geometric decay recovered only a quarter of the loss: the fixed-cell control
keeps the whole cell reachable forever and drives its link \emph{harder}
($\mathbb{E}|\tanh|$ $0.441\to0.753$), so it inherits the same limit by a
different route. The prediction runs the right way at the other end too---the
$r{=}32$, $s{=}10$ recipe that \S\ref{sec:capacity-xdom} identifies as keeping
the link in its gradient-carrying range moves only $0.147\to0.159$ over six
tasks and retains the most plasticity of any six-task arm we have run
($a_6/a_1=0.555$ against $0.332$ for $r{=}64$, $s{=}40$ on the same anchor).
Sequential capacity is limited by the same thing single-injection capacity
is (\S\ref{sec:capacity-xdom}): the optimizer under the bound, not the space
inside it.

\textbf{What a deployment should monitor is not the remaining room.} The
obvious controller for this cycle watches the room and consolidates when it
runs low: define $\phi=M/M_0$ as the mean room remaining against its value
after the first task---a different quantity from the reserved radius $\rho$ of
Prop.~\ref{prop:budget}---and consolidate when $\phi$ falls past a threshold.
Across the ten archived sequences that rule behaves badly, for a reason the
data makes plain. Where nothing consolidates, $\rho$ falls monotonically and
tracks plasticity closely (Pearson $+0.98$ against normalized absorption over
six tasks, $+1.00$ over four). Where consolidation is running, $\phi$ rises at
every major version and becomes a sawtooth: its minimum is $0.62$ in the
rehearsing cycle above, $0.72$ in the arm without rehearsal and $0.81$ in the
four-task sequence, against $0.33$ to $0.79$ across the seven that never
consolidate. A threshold at $0.40$ therefore fires in exactly one of the ten---the
rehearsing sequence that never consolidates, at its last task---and never in a
sequence that is already consolidating, because consolidating is what puts the
room back. The controller that decides to
consolidate on low room switches itself off the first time it acts.

The obvious replacement does not work either. The link's mean displacement
$\mathbb{E}|\tanh|$ climbs through consolidations where $\phi$ oscillates
($0.450\to0.608$ rehearsing, $0.450\to0.529$ without), which makes it look like
the depletion signal $\phi$ is not, and within a single sequence it correlates
$-0.97$ with normalized absorption. It fails on three checks. It is not in fact
monotone---in the arm without earlier-task replay it \emph{falls} between the
fifth and sixth tasks ($0.535\to0.529$) while absorption drops $41\%$, and it
is non-monotone throughout the inert-regularized arm of \S\ref{sec:fusion}.
No threshold on it separates a fold that violated the invariant from one that
did not: the two folds that did violate were at $0.550$ and $0.614$, while the
consolidating cycle reached $0.608$ and the fixed-cell control $0.753$ with
every fold clean---so a threshold at $0.55$ fires through the whole of a clean
sequence and one at $0.62$ never fires in the sequence that failed. And
across arms it orders nothing. The comparison has to be made among sequences
that never consolidate, because a consolidation zeroes the fill and a
consolidating arm's final displacement therefore describes its last cycle
rather than its run; six of the archived six-task arms qualify. Over those six
the rank correlation between final displacement and normalized late absorption
is $-0.03$. The most displaced is the fixed-cell control of
\S\ref{sec:sequential} at $\mathbb{E}|\tanh|=0.753$, and it ends at $49.7\%$
of its first task's absorption---third of the six, behind an arm at $0.579$
that ends at $52.7\%$. The \emph{least} displaced, the $r{=}32$, $s{=}10$
sequence of \S\ref{sec:rehearse-old} at $0.159$, ends at $55.5\%$ and is the
best of them; the worst, at $31.1\%$, is the inert-regularized arm at $0.446$,
second-least displaced of the six. Neither scalar we log is a controller. What an
operator can currently do is schedule on a fixed interval, which is what every
consolidating run here did, and we report the two candidates so that a reader
does not have to rediscover that they fail.

\textbf{Consolidation has a break-even, and the gentle recipe is below it.}
\S\ref{sec:capacity-xdom} identifies $r{=}32$, $s{=}10$ as the setting that
keeps the bounded link inside its gradient-carrying range, and the saturation
account above predicts that such a recipe should gain little from a major
version, having little saturation to relieve. Run through the same cycle it
gains \emph{less than nothing}. Its link stays where the repair intends---the
share past $|t|{=}0.99$ never exceeds $0.0016\%$, three orders of magnitude
below the standard recipe, and its room never falls below $90\%$ of its
starting value---and consolidating twice takes the first task's retention from
$82.4$ to $59.5\%$, a drop of $22.9$ points against a between-seed spread of
five to eight elsewhere in this paper, while raising per-task plasticity to
$a_6/a_1=0.678$, the highest of any arm we have run. Summed recall falls too,
$0.983$ to $0.875$, but that difference is about $1.8$ recall points per task
and we do not lean on it: at one seed and one cadence it is inside the range
\S\ref{sec:limitations} says single-seed comparisons cannot resolve. The
retention drop is the load-bearing number here.
The cycle keeps that model able to learn and stops it keeping what it learned.

\begin{table}[t]\centering\small
\setlength{\tabcolsep}{3.6pt}
\begin{tabular}{llcrrrrrc}
\toprule
recipe & policy & replay & summed & $a_6/a_1$ & WikiText & LAMBADA & major cost & viol. \\
\midrule
$r64$, $s40$ & minor       & ---  & 0.757 & 0.332 & 12.45 & 50.0 & --- & 1 \\
$r64$, $s40$ & minor       & $0.5$ & 1.859 & 0.338 & 12.42 & 50.7 & --- & 1 \\
$r64$, $s40$ & consolidate & ---  & 0.878 & 0.449 & 14.21 & 52.6 & $19.3$, $21.3\%$ & 0 \\
$r64$, $s40$ & consolidate & $0.5$ & \textbf{2.117} & 0.531 & 14.46 & 57.3 & $21.4$, $26.1\%$ & 0 \\
\midrule
$r32$, $s10$ & minor       & $0.5$ & 0.983 & 0.555 & 10.26 & 33.4 & --- & 0 \\
$r32$, $s10$ & consolidate & ---  & 0.401 & 0.655 & 10.22 & \textbf{31.6} & $1.2$, $1.3\%$ & 0 \\
$r32$, $s10$ & consolidate & $0.5$ & 0.875 & \textbf{0.678} & 10.31 & 31.7 & $1.2$, $2.3\%$ & 0 \\
\bottomrule
\end{tabular}
\caption{Seven operating points: two recipes, two policies, replay on or off,
six $500$-fact tasks each on a published NF4 release. ``summed'' is recall
over all six tasks at the end of the sequence, $a_6/a_1$ the last task's
absorption as a share of the first's, ``major cost'' the share of 4-bit codes
each consolidation changed, and ``viol.'' the number of weights that left a
cell over $1{,}409{,}286{,}144$ checked per fold. No configuration dominates:
the standard recipe with consolidation absorbs most, the gentle recipe costs
least across domains and needs the cheapest major versions, and replay moves
summed recall by a factor of $2.2$ to $2.5$ in every cell.}
\label{tab:operating}
\end{table}

Read down Table~\ref{tab:operating}'s replay column and the same factor appears in every recipe and
under both policies: $0.757\to1.859$, $0.878\to2.117$, $0.401\to0.875$, a
factor of $2.2$ to $2.5$ each time. The violation column is not a second such split, and we say so because the
table invites the reading. Both unconsolidated arms at $r{=}64$, $s{=}40$
\emph{in this table} put a weight outside its cell at the sixth fold and all
four consolidating arms are clean, but the archive holds three more
unconsolidated arms at that recipe and all three are clean, one of them at a
\emph{larger} final displacement ($0.579$) than either consolidating arm. Over
all seven, $0/2$ consolidating and $2/5$ unconsolidated arms failed, which is
$p=0.48$. We report the violations where they happened and claim nothing about
what prevents them. Replay carries the knowledge; consolidation carries the
artifact; the two do not substitute.

The mechanism is the one the rest of this subsection argues. Consolidation
burns the fill and injects one fresh quantization error in exchange for room.
Where the fill is large relative to that error---the standard recipe ends at
$\mathbb{E}|\tanh|=0.61$ with its room drawn down to $62\%$---the exchange is
favourable. Where the fill is small, it is not: at $\mathbb{E}|\tanh|=0.176$
the re-quantization noise is comparable to the knowledge the fill carries, and
a major version overwrites rather than renews. A deployer should read the
table as a set of operating points and not as a ranking: for absorption, the
standard recipe with consolidation; for cross-domain cost and for cheap major
versions, the gentle recipe \emph{without} one, which over six tasks never
needed a major version at all. The price of a major version is itself a
property of the recipe, and it varies here from $1.2\%$ of the 4-bit codes to
$26.1\%$.

\textbf{What a major version costs on a task metric.} Everything above is
recall and perplexity. The stage the cycle was missing is the one that asks
whether the model a major version produces is still the model the release was
evaluated as, and it needs a task metric measured \emph{inside} the sequence.
We ran it: the same six-task cycle at both recipes, with ARC-easy,
ARC-challenge, HellaSwag and WinoGrande scored on the live served model after
folds $0$, $1$, $3$ and $5$ and on both sides of every consolidation, and with
four epochs of replay on the earlier tasks' own material after each major
before the version is judged. At $600$ items per task the binomial standard
error is about $2.0$ points, which is the resolution of everything below.

\begin{center}\small
\setlength{\tabcolsep}{4pt}
\begin{tabular}{lrrrrrrr}
\toprule
 & \multicolumn{4}{c}{end of run, points from the released anchor} &
   summed & major cost & viol. \\
\cmidrule(lr){2-5}
recipe & ARC-e & ARC-c & HellaSwag & WinoGr. & recall & (\% of codes) & \\
\midrule
$r{=}32$, $s{=}10$ & $+0.0$ & $-0.3$ & $+0.5$ & $-0.5$ & 0.684 & $1.2$, $1.1$ & 0 \\
$r{=}64$, $s{=}40$ & $-1.8$ & $-4.5$ & $-2.8$ & $-6.2$ & 2.050 & $21.4$, $24.3$ & 0 \\
\bottomrule
\end{tabular}
\end{center}

The two arms answer the question in opposite directions and the recipe is what
chooses. After six tasks and three major versions the gentle recipe is
\emph{indistinguishable from the file it started as} on all four metrics---the
largest deviation is half a point, a quarter of the measurement's own
error---and it has kept almost nothing of what it was taught. The standard
recipe keeps three times the knowledge and pays for it with two to six points
of task accuracy, the largest of them three standard errors. Both hold the
artifact at every one of the six folds. Capability retention and knowledge
retention are not two faces of one property here; they are a trade, and the
slope of the link sets the exchange rate.

\textbf{What the recovery pass does is trade knowledge back for capability.}
Immediately after the second major version at $r{=}64$, $s{=}40$, before any
replay, the served model has lost $5.2$ points of ARC-challenge and $3.5$ of
WinoGrande against the anchor; four epochs on the earlier tasks' material
returns it to $-2.5$ and $-1.0$. Across both majors the pass is worth about a
point and a quarter on each metric. It is not free: summed recall ends at
$2.050$ with the pass against $2.117$ for the same cycle without it
(\S\ref{sec:loop}), so the replay that repairs the function costs a little of
what was written. At the gentle recipe the same trade is a bad one---summed
recall falls from $0.875$ to $0.684$ for a capability the recipe was holding
anyway---which is the break-even of the previous paragraph seen from the other
side.

Two things this measurement does not settle. It is one seed and one
consolidation interval, and the intermediate points are noisy at this suite
size: the gentle recipe dips to $-3.2$ on ARC-easy after its second major and
returns to $0.0$ by the sixth task, and we would not read either number alone.
And the recall immediately after a consolidation, before the recovery pass,
was not logged, so the pass's effect on knowledge is bounded rather than
measured---one line of instrumentation we will not have until the next run.

\textbf{What this loop is not.} Three of the five stages ran; two did not,
and calling the result a closed loop without saying which would overstate it.
\emph{Saturate} was measured, not used: the prior-accuracy scan of
\S\ref{sec:prior} decides whether a write is worth issuing at all, and it
was run as eight standalone fills on a different model, never as a gate on a
task inside this sequence. The only saturation signal that did cross into the
loop is $\mathbb{E}|\tanh|$, logged per fold, and nothing was conditioned on
it. \emph{Validate} now runs on a task metric inside the sequence, as the
paragraphs above report, but it still gates nothing: no minor version here was
accepted or rejected on the strength of a score, and nothing was rolled back. A loop that closed all five stages would
read the saturation before deciding to write, validate on task accuracy after
writing, and roll the fill back by subtraction when the gate failed---each of
which this paper measures separately and none of which it wires to the next.
That wiring, and a sweep of the consolidation interval, are what stand
between these six tasks and an operating discipline that runs unattended.
\S\ref{sec:lifecycle} states the discipline the measurements do support.

\subsubsection{The same loop at 8B}\label{sec:loop8b}

The cycle above runs on a 1.7B release, and the objection that a claim about
language models cannot rest there is the right one. This is the identical
recipe---six 500-fact tasks, half of each earlier task's material rehearsed
alongside the new one, a consolidation after every second task, four epochs
of recovery after each---on unsloth's published NF4 release of Qwen3-8B-Base,
$6.946\times10^{9}$ constrained weights (\texttt{exp124}, $r{=}64$, $s{=}40$,
seed 0 tabulated below at $7.3$ hours on one 4090; a second seed
follows). It is the first sequence in this paper
above 1.7B, and running it is what exposed the storage floor of
\S\ref{sec:floor}: the first attempt reported $19{,}791$ violated weights at
its first fold and died at the second. That reported figure is a lower bound
rather than the count---the guard that produced it tested the fp32 value
before the fp16 store, so the rounding it existed to catch happened after the
test---and the controlled reproduction of that configuration violates
$37{,}443$ (\S\ref{sec:floor}). The numbers below are from the run with the
floor in place.

\begin{center}\small
\setlength{\tabcolsep}{3.6pt}
\begin{tabular}{lcccccccccc}
\toprule
after & \multicolumn{6}{c}{retention (\% of what the task first learned)} &
\multicolumn{2}{c}{perplexity} & room & viol. \\
\cmidrule(lr){2-7}\cmidrule(lr){8-9}
 & T0 & T1 & T2 & T3 & T4 & T5 & WikiText & LAMBADA & remaining & \\
\midrule
T0 & 100.0 & --- & --- & --- & --- & --- & 8.30 & 33.4 & $1.63\times10^{-3}$ & 0 \\
T1 & 97.4 & 100.0 & --- & --- & --- & --- & 8.26 & 33.1 & $1.15\times10^{-3}$ & 0 \\
T2 & 94.9 & 99.2 & 100.0 & --- & --- & --- & 9.40 & 41.1 & $1.47\times10^{-3}$ & 0 \\
T3 & 98.0 & 108.3 & 105.5 & 100.0 & --- & --- & 9.40 & 44.8 & $1.02\times10^{-3}$ & 0 \\
T4 & 91.5 & 96.6 & 95.0 & 110.5 & 100.0 & --- & 11.09 & 55.6 & $1.52\times10^{-3}$ & 0 \\
T5 & 94.0 & 101.2 & 98.5 & 119.2 & 102.1 & 100.0 & 10.77 & 56.4 & $9.79\times10^{-4}$ & 0 \\
\bottomrule
\end{tabular}
\end{center}

A second seed of the identical recipe reproduces the picture
(\texttt{exp124\_loop\_8b\_s1}): no task ends below $92.8\%$ of what it
first learned---the floor of the two-seed pair, at task~0---and across
both seeds five of the ten earlier-task ends land above $100\%$. Fresh
learning falls $0.818$ to $0.447$; summed recall ends at $3.750$ against
$3.707$ learned; the integer check returns the vendor's code on all
$6.9\times10^{9}$ constrained weights at every fold of this run too. Its
within-pair room ratios are the $\beta$ series already quoted
($0.716/0.695/0.640$ against seed~0's $0.704/0.693/0.642$), and the two
consolidations rewrite $34.0$ then $34.6\%$ of the codes against seed~0's
$33.4$ and $34.6\%$. Table~\ref{tab:v2bench} gathers the pair beside the
v2 lottery ring.

Read the last row first. After six tasks and two major versions, every
earlier task holds between $94\%$ and $119\%$ of what it first learned, and
the artifact returned its vendor's codes on all $6.946\times10^{9}$
constrained weights at every one of the six folds. Nothing was forgotten
catastrophically and nothing was forgotten quietly: the count is over every
weight, not a sample.

\begin{table}[t]\centering
\begin{tabular}{lrrrr}
\toprule
\multicolumn{5}{l}{\emph{Six-task loop, 8B, NewFacts v2 (500 facts/task)}} \\
seed & summed fresh & retention floor & $\beta_{1..3}$ (within-version) & violations \\
\midrule
s0 & 3.48 & 94.0\% & 0.704/0.693/0.642 & 0 \\
s1 & 3.71 & 92.8\% & 0.716/0.695/0.640 & 0 \\
\midrule
\multicolumn{5}{l}{\emph{v2 lottery ring, 8B (exact-bit cells)}} \\
cell & bits/fact & min fresh & mean fresh & violations \\
\midrule
lottery ring (6$\times$60) & 21.72 & 98.3\% & 99.2\% & 0 \\
\bottomrule
\end{tabular}

\caption{\textbf{The NewFacts v2 instrument, tabulated.} Top: the six-task
8B loop under the frozen v2 benchmark, both seeds---summed fresh recall,
the retention floor, the per-version $\beta$ series the text quotes, and
integer-domain violations (every constrained weight, every fold). Bottom:
the v2 lottery ring; the archived corpus-level rate is $21.72$ bits per
fact, while the ring's two record cells are matched in \emph{exact}
information content by construction ($28.1224$ against $28.1138$ bits per
draw, $\log_2$-combinatorial, $0.03\%$; \texttt{experiments/newfacts.py}).
An independent displacement-pair estimator of $\beta$ matches these series
within $0.009$ (\S\ref{sec:loop8b}). Earlier tables report the archived
v1-era instruments; this table reports the frozen v2 benchmark---bit-matched
cells, content-hashed, verified by the loader at import. Both instruments
remain in the record.}
\label{tab:v2bench}
\end{table}

The denominator of those percentages is falling, and reporting the ratio
without it would be misleading. What each task learns when it is new runs
$0.761$, $0.667$, $0.659$, $0.483$, $0.536$, $0.376$: the sixth task absorbs
half of what the first did ($a_6/a_1=0.49$). Retention is therefore high
against a shrinking baseline, and the two effects nearly cancel in the
aggregate---summed recall over all six tasks is $3.539$ at the end against
$3.483$ when each was written. The right summary is that the sequence does
not forget and does lose plasticity, and that the plasticity loss is the
geometry of Prop.~\ref{prop:decay} rather than a failure of the fill: T3,
whose $119\%$ is the largest retention in the table, is also the task that
learned least, having been written immediately after a consolidation zeroed
the fill.

Three things in the table are worth separating. The retentions above $100\%$
are not the method improving a task after the fact; they are rehearsal
recovering what the fill had lost, and they are largest for T3, which was
written immediately after a consolidation had zeroed the fill and so first
learned the least ($0.483$ against $0.761$ for T0). The room does not decay
monotonically---it falls within each pair of tasks and is renewed by the
consolidation between them, $1.15\rightarrow1.47\times10^{-3}$ after the
first and $1.02\rightarrow1.52\times10^{-3}$ after the second---which is
Prop.~\ref{prop:decay} and the major-version step doing what they are for.
The two consolidations changed $33.4\%$ and $34.6\%$ of the 4-bit codes: a
major version is a different file, and calling it one is not a formality.

Taking the ratio inside each pair is the test of Prop.~\ref{prop:decay}'s
constant $\beta$, and this table has been printed without it. The three
within-pair ratios are $1.15/1.63=0.706$, $1.02/1.47=0.694$ and
$0.979/1.52=0.644$; a second seed of the same cycle reproduces them at
$0.715$, $0.695$ and $0.640$, matching to within $0.009$. All fall below the
$0.78$--$0.85$ this paper quotes, and they decline within both runs. That band was measured over four matched
sequences at 1.7B and the scale was not stated with it, which was an error of
the kind \S\ref{sec:limitations} is for; recomputed at the recipe this 8B run
uses, 1.7B gives $0.775$, $0.783$ and $0.775$, a spread of $0.8\%$, against
$0.706$, $0.694$ and $0.644$ here.

Part of the gap is recipe and part is not. Rehearsal alone moves 1.7B from
$0.812$ to $0.778$, because rehearsing writes more material into the same
cells, so a rehearsed sequence should sit low in the band---and 8B sits below
it. What recipe does not explain is the decline \emph{within} the 8B run:
$\beta$ is not a constant there, and a $\beta$ that falls makes the geometric
budget $a_1/(1-\beta)$ of \S\ref{sec:theory} an over-estimate of what remains,
not an under-estimate. The honest statement is narrower than either the law or
its dismissal: \emph{the constant-fraction law is exact at 1.7B over four
sequences and approximate at 8B over two, where $\beta$ itself declines
within each run.}
Which of the two the 27B and 32B sequences resemble is the measurement that
decides whether the budget is a budget or a fit, and it is in flight.

The cost is where it has been throughout, in cross-domain ability rather than
in-domain perplexity. WikiText rises $8.15\rightarrow10.77$ over the whole
sequence, $32\%$ above the anchor; LAMBADA rises $22.28\rightarrow56.35$,
$153\%$. The task suite, run inside the sequence, shows that this cost is
concentrated rather than spread:

\begin{center}\small
\begin{tabular}{lcccc}
\toprule
 & ARC-e & ARC-c & HellaSwag & WinoGrande \\
\midrule
released anchor & 78.7 & 55.8 & 67.7 & 68.8 \\
after T0 & 78.5 & 56.3 & 59.5 & 70.0 \\
after T3 & 79.0 & 57.0 & 55.8 & 65.2 \\
after consolidating at T3, fill zeroed & 77.7 & 55.0 & 54.3 & 66.5 \\
after the recovery pass & 76.3 & 55.8 & 54.5 & 66.8 \\
after T5 & 74.3 & 52.8 & 54.0 & 64.2 \\
\bottomrule
\end{tabular}

\vspace{2pt}
{\footnotesize Metric convention as in Table~\ref{tab:downstream}
(\texttt{acc\_norm} where the task provides it, \texttt{acc} otherwise), $600$
items per task, so these are not row-comparable with the full-set scores
there.}
\end{center}

At $600$ items a suite point has $\sigma\approx2.0$, so ARC-e, ARC-c and
WinoGrande end $4.3$, $3.0$ and $4.7$ points down---one and a half to two and
a half standard errors across six tasks---while HellaSwag ends $13.7$ points
down, near seven.

The row to read for the release contract is the fill-zeroed one. A
consolidation re-quantizes the anchors and zeroes the fill, so what that row
scores is the new 4-bit artifact by itself, with nothing written into it: it
is the file an operator would ship as the next major version. The row below
it is the same artifact after four epochs of recovery have written a fill
back in, and the sequence continues from there. Against the version it
replaces it moves $-1.3$, $-2.0$, $-1.5$ and $+1.3$ points---$0.7$, $1.0$,
$0.8$ and $0.7$ standard errors, one of them upward.

Read alone, that event says a major version costs nothing measurable, and an
earlier draft of this paper said so. It is the wrong reading, and the archive
contains its own correction. Five consolidations across three runs and two
scales carry the same before-recovery measurement:

\begin{center}\small
\begin{tabular}{lrrrrrr}
\toprule
run & codes changed & ARC-e & ARC-c & HellaSwag & WinoGrande & mean \\
\midrule
1.7B, $s{=}10$, T1 & $1.2\%$ & $-2.83$ & $-1.00$ & $+0.67$ & $+0.16$ & $-0.75$ \\
1.7B, $s{=}10$, T3 & $1.1\%$ & $-2.67$ & $-1.16$ & $-0.33$ & $+1.00$ & $-0.79$ \\
1.7B, $s{=}40$, T1 & $21.4\%$ & $-0.34$ & $\pm0.00$ & $+2.34$ & $-2.17$ & $-0.04$ \\
1.7B, $s{=}40$, T3 & $24.3\%$ & $-2.83$ & $-3.50$ & $+0.50$ & $-0.33$ & $-1.54$ \\
8B, $s{=}40$, T3 & $34.6\%$ & $-1.33$ & $-2.00$ & $-1.50$ & $+1.33$ & $-0.87$ \\
\bottomrule
\end{tabular}
\end{center}

Every one of the five is negative, at $-0.80\pm0.24$ points of mean suite
accuracy ($t=-3.36$ on four degrees of freedom). The events are clustered two
to a run and each mean averages four correlated suites, so we do not quote a
$p$-value as though they were five independent draws; what the consistency
supports is a small negative effect rather than none. \emph{A major version
costs about eight tenths of a point of mean suite accuracy, and the cost does
not scale with how much of the file it rewrites}---$1.2\%$ of codes changed
costs as much as $34.6\%$. That is a better result than no cost at all: it is
a price an operator can budget, and it is not bought down by consolidating
more gently.

Against the vendor's original release the 8B artifact is down $0.5\sigma$ on
ARC-e, $0.4\sigma$ on ARC-c and $1.2\sigma$ on WinoGrande, and $6.7\sigma$ on
HellaSwag---and that last number is not the consolidation's doing. Writing
the first task alone, before any consolidation exists, already costs
HellaSwag $8.2$ points while leaving the other three within a standard error
($-0.2$, $+0.5$, $+1.2$). The cost is incurred by writing and it is paid
once, largely at the first task; re-quantizing to a new released artifact
adds a fluctuation on top of it. The loss is also
front-loaded and cumulative rather than a consolidation charge: $8.2$ of it
is gone after the first task alone, it keeps sinking through T3, and the
recovery pass does not bring it back ($54.3\rightarrow54.5$). Two readings
predict the same thing and we cannot yet separate them: HellaSwag is the most
narrative of the four and therefore the furthest from the WikiText the
rehearsal draws on, and it is also the most sensitive to general language
modelling, which the $32\%$ perplexity rise directly degrades. Both predict
the same shape at larger scale, which is what the 27B and 32B sequences will
test.

\paragraph{What the bound buys: the same recipe without it.}
Everything above prices what the cell bound costs. What it buys is measured by
removing it and changing nothing else. The control arm replaces
$M\odot\tanh(sBA^{\!\top})$ with the raw $BA^{\!\top}$, so the same $252$
wrapped layers of the same 8B release are trained with invariance
\emph{measured} rather than asserted, and the obvious objection---that an
unconstrained update is being judged at a rate chosen for a constrained
one---is answered by giving it four rates of its own. One task, $500$ facts,
six epochs, $r{=}64$, everything else as in \S\ref{sec:setup}:

\begin{center}\small
\begin{tabular}{lccrr}
\toprule
learning rate & recall & WikiText ppl & weights outside their cells & \\
\midrule
$1\times10^{-5}$ & $0.027$ & $7.71$ & $14{,}634{,}220$ & $0.21\%$ \\
$5\times10^{-5}$ & $0.065$ & $7.29$ & $28{,}401{,}256$ & $0.41\%$ \\
$2\times10^{-4}$ & $\mathbf{0.152}$ & $7.37$ & $92{,}884{,}472$ & $1.34\%$ \\
$5\times10^{-4}$ & $0.142$ & $9.93$ & $782{,}103{,}992$ & $11.26\%$ \\
\midrule
bounded, $10^{-3}$, 24 epochs & $0.761$ & $8.30$ & $0$ & $0.00\%$ \\
\bottomrule
\end{tabular}
\end{center}

At none of the four rates does the unconstrained arm both learn and stay
inside the cells. At the gentlest, where it recalls $2.7\%$ and has learned
almost nothing, $14.6$ million weights have already left the cells they were
quantized from; at the rate that recalls most, $92.9$ million have. The
released codes are not recoverable from any of these artifacts, which is the
property the paper is about, and it is lost before any of them becomes useful.
The last row is not epoch-matched---it is task~0 of \texttt{exp124} at $24$
epochs---so its recall is not comparable to the four above it and is given for
the count in the final column, which is zero by construction at any epoch.

The six-task unconstrained sequence has since completed, at its own best
rate and on the same card as \texttt{exp124}
(Fig.~\ref{fig:boundprice}; \texttt{cycle\_free8b}). On the function axis the
two arms are close, and the comparison is reported the way it came out: the
unconstrained arm absorbs slightly more ($3.63$ against $3.48$ summed fresh
recall), keeps more late-sequence plasticity ($a_6/a_1$ $0.62$ against
$0.49$), ends ten LAMBADA points better ($46.6$ against $56.4$), and loses
HellaSwag alike ($14.5$ against $13.7$ points)---the write damage is
indifferent to the bound. On the artifact axis the two arms are not comparable at all:
$376$~million weights leave the released cells in the unconstrained arm's
first task and $2.01$~billion ($29\%$ of the file) by its sixth, growing
between major versions even though it consolidates on the same cadence,
while the bounded arm re-quantizes to the vendor's codes at every fold of
every task. Each arm is one run. The bound's price at 8B is therefore real
and honestly stated---a few percent of absorption, a tenth of plasticity,
ten points of LAMBADA---and what it buys is the only thing the unconstrained
procedure cannot have at any rate: a file that is still the release.

\begin{figure}[t]\centering
\includegraphics[width=\linewidth]{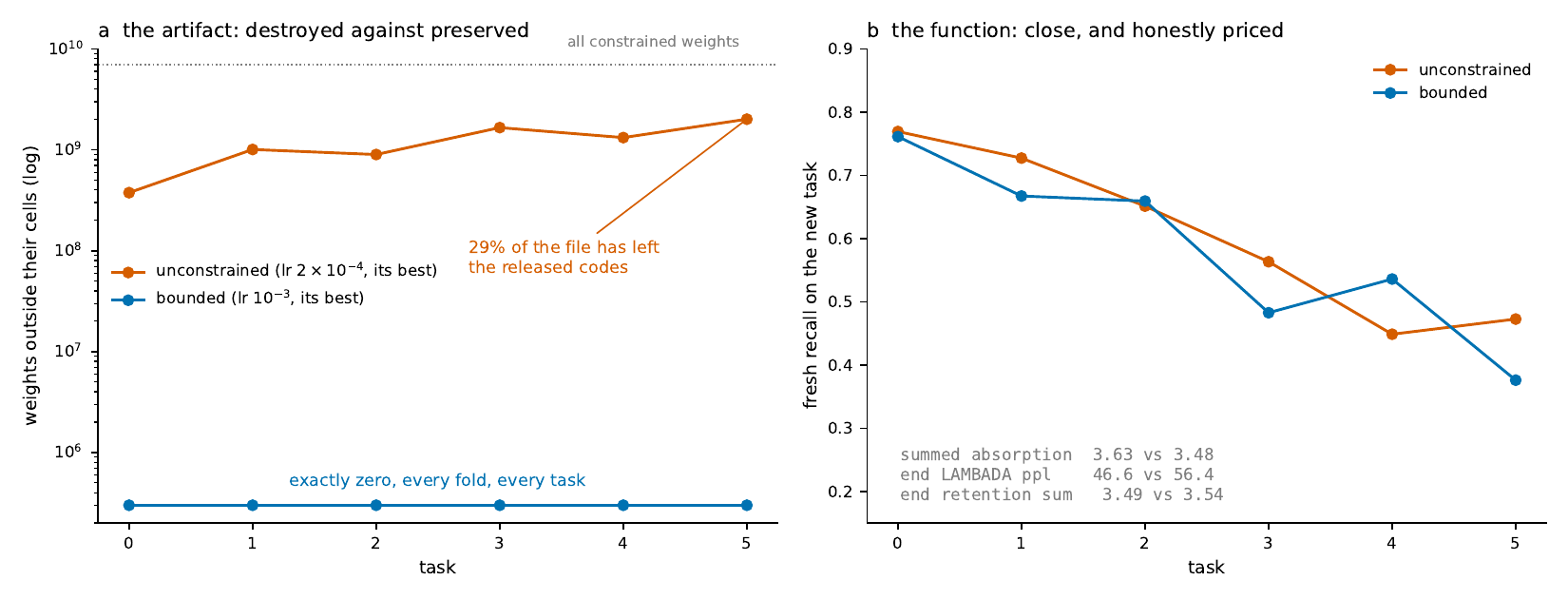}
\caption{\textbf{What the bound buys, and what it costs, at 8B.} Same six
tasks, same card, each arm at its own best learning rate. \textbf{a}, Weights
outside their released cells: the unconstrained arm leaves the artifact at
the first task and never returns; the bounded arm returns the vendor's codes
at every fold. \textbf{b}, Fresh recall per task: the functional gap is
small and priced in the text. One run per arm. Generated by
\texttt{scripts/fig\_boundprice.py}.}
\label{fig:boundprice}
\end{figure}

\subsubsection{Nine consolidations: the drift is bought by writing}
\label{sec:driftlaw}

\begin{figure}[t]\centering
\includegraphics[width=\linewidth]{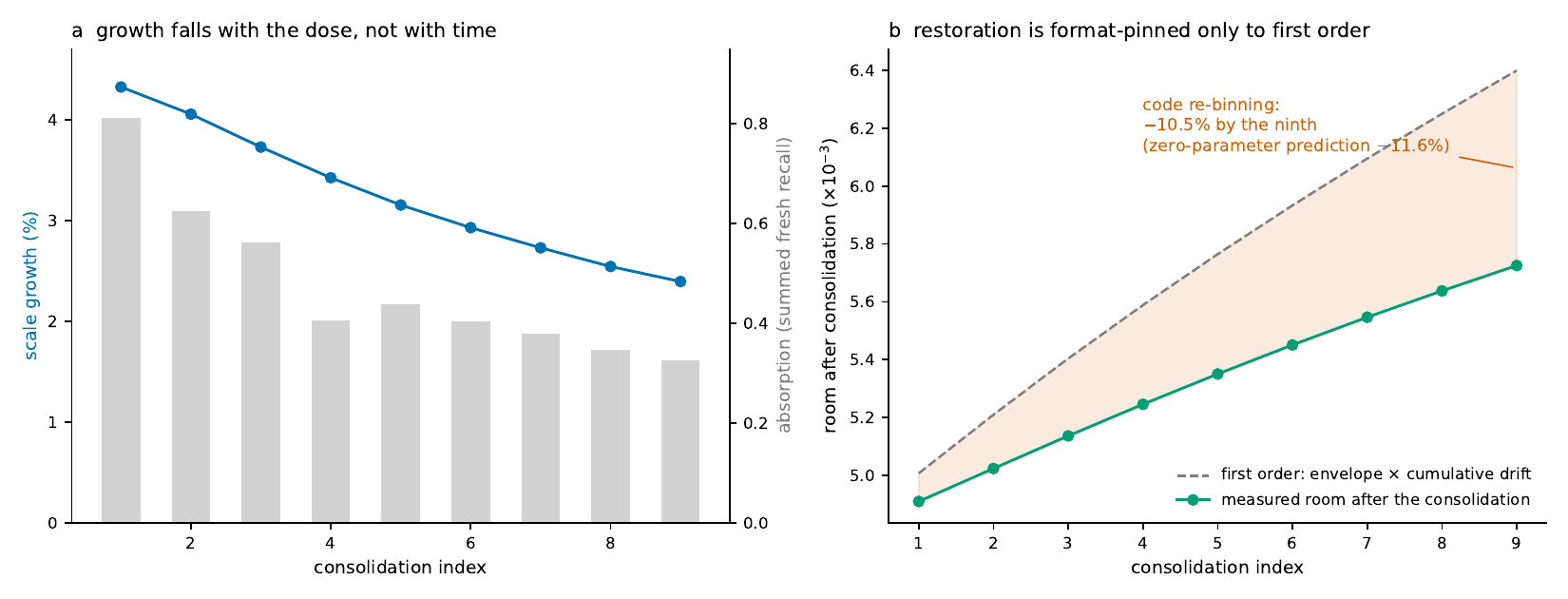}
\caption{\textbf{Nine consolidations in one sequence} (20 tasks, Qwen3-1.7B,
consolidation every second task, \texttt{exp131}). \textbf{a}, The scale
growth at each consolidation falls in step with the absorption of the
version before it; the decline is the image of the dose collapsing, not an
approach to a drift equilibrium. \textbf{b}, The room each consolidation
restores, against the first-order prediction of
Prop.~\ref{prop:consproj}(iv). The growing deficit is code re-binning:
inflating the scales $\times1.334$ over fixed weights re-bins mass into the
central narrow cells and lowers $\mathbb E[H_c]$ by $11.6\%$, a
zero-parameter prediction of the measured $10.5\%$. Generated by
\texttt{scripts/fig\_nineburn.py}.}
\label{fig:nineburn}
\end{figure}

\begin{figure}[t]\centering
\includegraphics[width=0.72\linewidth]{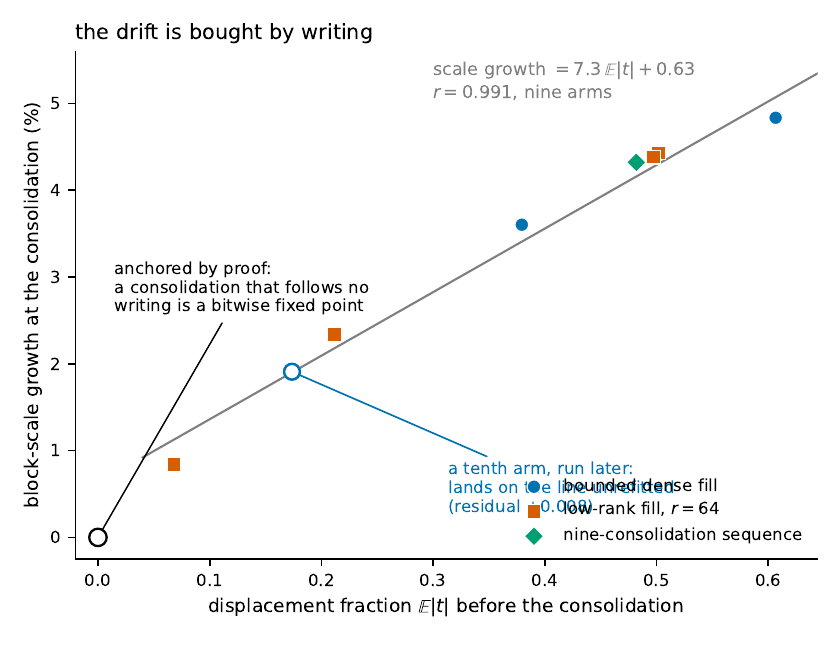}
\caption{\textbf{The drift law.} The block-scale growth at the first
consolidation, against the mean displacement fraction $\mathbb E|t|$ the fill
had written, for nine arms spanning both parameterizations, four learning
rates and a nine-fold range of displacement. One line fits them at
$r=0.991$. The zero point is not fitted: by Prop.~\ref{prop:consproj}(i) a
consolidation that follows no writing is a bitwise fixed point, so the curve
is anchored at the origin analytically and the positive intercept of the
linear fit is curvature, not a write-free drift. Generated by
\texttt{scripts/fig\_driftlaw.py}.}
\label{fig:driftlaw}
\end{figure}

Whether renewal can repeat without limit turns on one quantity the earlier
cycles never recorded: what happens to the block scales. A consolidation
re-derives each scale from the current weights, and a fill that pushed the
block's largest weight outward drags the scale up with it. The twenty-task (1.7B, one-seed)
sequence of Fig.~\ref{fig:nineburn} instruments this: nine consolidations,
each recording its scale statistics, its restored room, and the absorption
of the version it closes. Violations are zero at every fold and every
consolidation.

Three findings. First, the drift is real and it compounds: the mean scale
grows $4.3\%$ at the first consolidation and $2.4\%$ at the ninth,
$\times1.334$ cumulatively, and the sequence reaches no equilibrium.
Second, the within-run decline of the drift rate attributes to the dose,
not to time. Inside a single run every quantity is monotone in the
consolidation index, so any of them ``explains'' any other; the attribution
comes instead from varying the driver across arms at a fixed consolidation
index (Fig.~\ref{fig:driftlaw}): nine arms put their first consolidation on
one line, drift $=7.3\,\mathbb E|t|$ at $r=0.991$, all at 1.7B, with the origin supplied
by the projection identity rather than by the fit. The drift is bought by
writing, at a measured exchange rate, and a version that writes little
drifts little. Third, the room a consolidation restores is format-pinned
only to first order: the measured restoration falls $10.5\%$ short of
envelope~$\times$~cumulative drift by the ninth consolidation, and the
deficit is reproduced with no fitted parameter by re-binning alone
(Fig.~\ref{fig:nineburn}b). Scale drift does not merely inflate the file's
dynamic range; it slowly reshapes which cells the knowledge lives in.

\subsubsection{Several writers, one artifact: two keep their knowledge, four
lose their smallest writers}\label{sec:fusion}
Two writers with disjoint corpora wrote into one release and recalled
$86.6\%$ of their facts after the merge. On the real corpus of
\S\ref{sec:realcorpus} split by domain, two inert fills at reserved width
summed to that figure from $92.9\%$ before the merge, the worst-off writer
keeping $82.7\%$, at a LAMBADA of $29.6$ (Table~\ref{tab:fusion}). Nothing
had to be clamped, because the sum of two reserved fills lies inside the
cells by construction, and the merged model re-quantizes to the released
codes exactly. The alternatives a reader would reach for kept far less of
the same two writers' knowledge, and from a marginally higher $94.4\%$ in
place: $20.5\%$ by averaging the full-width fills, $25.6\%$ under TIES and
$47.3\%$ under the magnitude winner. The full-width sum came closest at
$62.3\%$, and stayed inside the cells only by pulling $208{,}005{,}377$
weights back into them, which is where the rest of its knowledge and its
cross-domain perplexity went (LAMBADA $94.6$). The number of writers is the
ceiling, and the table locates it: four inert domain writers at a quarter
of the width each kept $50.7\%$ but left the worst of them at $2.6\%$, and
eight kept $2.4\%$ with the worst at $0.0\%$.

Every update above has one author. A release, once published, has many: a
hospital, a newsroom and a software vendor each want their own facts in the
same model, at the same time, without seeing one another's corpus. With
bounded fills the arithmetic of combining them is simple. If each of $K$
writers is given $1/K$ of the half-width, every fill satisfies
$|f_k|\le M/K$ and their sum satisfies $|\sum_k f_k|\le M$: the combined
model is bit-identical to the release \emph{by construction}, with no clamp
and nothing lost at the merge, and each contribution remains separately
revocable. Writers who did not coordinate can be averaged instead: a cell is
an interval, the mean of points inside it is inside it, so $\tfrac1K\sum_k
f_k$ is invariant by convexity---at the price of diluting every writer's
fill $K$-fold. Whether the \emph{knowledge} survives either operation the
way the arithmetic does is measured in \texttt{experiments/exp\_fuse.py}:
$K\in\{2,4\}$ writers on disjoint fact sets, each scored on its own facts
before the merge and again after it, against a single writer on the union
and against the lossy alternative a reader would try first (full-width
fills, summed, clamped to the cell).

\begin{table}[t]\centering\small
\setlength{\tabcolsep}{4pt}
\resizebox{\linewidth}{!}{
\begin{tabular}{llrrrrrr}
\toprule
$K$ & merge & own recall & fused recall & worst party & clamped & WikiText & LAMBADA \\
 & & before merge & & & weights & & \\
\midrule
1 & single writer & 55.8\% & 56.1\% & 56.1\% & 0 & 11.65 & 39.4 \\
2 & reserved $1/K$, sum & 40.3\% & 20.3\% & 18.7\% & 0 & 12.13 & 80.1 \\
2 & full width, average & 56.4\% & 24.9\% & 23.3\% & 0 & 10.70 & 35.1 \\
2 & full width, sum, clamp & 56.4\% & 30.6\% & 29.7\% & 275,590,505 & 13.05 & 73.7 \\
2 & disjoint matrices, full width, sum & 37.0\% & 19.2\% & 15.1\% & 0 & 12.37 & 76.6 \\
4 & reserved $1/K$, sum & 29.2\% & 4.9\% & 4.1\% & 0 & 32.08 & 548.9 \\
4 & full width, average & 57.3\% & 13.9\% & 13.2\% & 0 & 10.24 & 32.1 \\
4 & full width, sum, clamp & 57.3\% & 16.0\% & 14.5\% & 250,027,597 & 17.18 & 206.3 \\
4 & disjoint matrices, full width, sum & 25.0\% & 5.3\% & 4.0\% & 0 & 17.45 & 213.6 \\
2 & real corpus, by domain: reserved $1/K$, sum & 92.1\% & 67.1\% & 60.8\% & 0 & 11.81 & 106.2 \\
2 & real corpus, by domain: full width, average & 94.4\% & 20.5\% & 20.1\% & 0 & 10.72 & 40.8 \\
2 & real corpus, by domain: full width, sum, clamp & 94.4\% & 62.3\% & 50.9\% & 208,005,377 & 13.08 & 94.6 \\
2 & real corpus, by domain: full width, magnitude winner & 94.4\% & 47.3\% & 22.2\% & 0 & 12.13 & 68.3 \\
2 & real corpus, by domain: full width, TIES & 94.4\% & 25.6\% & 18.8\% & 0 & 10.71 & 35.9 \\
2 & real corpus, by domain: disjoint matrices, full width, sum & 94.0\% & 69.4\% & 55.7\% & 0 & 12.20 & 110.8 \\
2 & full width, average & 56.1\% & 26.3\% & 25.3\% & 0 & 10.70 & 35.5 \\
2 & full width, sum, clamp & 56.1\% & 32.7\% & 32.5\% & 274,492,084 & 13.06 & 78.8 \\
2 & full width, magnitude winner & 56.1\% & 33.3\% & 33.1\% & 0 & 12.06 & 53.8 \\
2 & full width, TIES & 56.1\% & 22.4\% & 20.7\% & 0 & 10.70 & 31.5 \\
2 & full width, DARE & 56.1\% & 29.1\% & 28.7\% & 474,387,963 & 11.95 & 48.9 \\
4 & full width, average & 57.3\% & 13.8\% & 12.9\% & 0 & 10.24 & 32.1 \\
4 & full width, sum, clamp & 57.3\% & 16.0\% & 14.5\% & 250,027,597 & 17.18 & 206.4 \\
4 & full width, magnitude winner & 57.3\% & 18.4\% & 16.5\% & 0 & 12.21 & 63.4 \\
4 & full width, TIES & 57.3\% & 17.7\% & 15.6\% & 0 & 10.97 & 37.6 \\
4 & full width, DARE & 57.3\% & 16.3\% & 15.2\% & 425,199,367 & 14.95 & 137.8 \\
2 & full width, $A$ frozen, average & 20.9\% & 8.8\% & 8.6\% & 0 & 10.45 & 33.8 \\
2 & full width, $A$ frozen, magnitude winner & 20.9\% & 13.4\% & 12.4\% & 0 & 11.38 & 47.6 \\
2 & inert (neutral pool), reserved, sum & 37.0\% & 14.4\% & 13.1\% & 0 & 12.33 & 47.2 \\
2 & inert (neutral pool, $3\times$), reserved, sum & 24.5\% & 10.2\% & 8.1\% & 0 & 12.32 & 47.1 \\
2 & inert (WikiText), reserved, sum & 34.4\% & 18.7\% & 16.7\% & 0 & 11.00 & 30.0 \\
2 & inert (neutral pool), full width, average & 45.9\% & 7.2\% & 6.4\% & 0 & 10.61 & 31.0 \\
2 & inert (neutral pool), full width, magnitude winner & 45.9\% & 14.2\% & 12.7\% & 0 & 11.81 & 37.7 \\
2 & inert (neutral pool), full width, TIES & 45.9\% & 6.3\% & 5.8\% & 0 & 10.58 & 29.9 \\
2 & real corpus: inert, reserved, sum & 92.9\% & 86.6\% & 82.7\% & 0 & 10.86 & 29.6 \\
4 & real corpus: inert, reserved, sum & 93.6\% & 50.7\% & 2.6\% & 0 & 11.18 & 33.7 \\
4 & real corpus: reserved, sum & 94.8\% & 0.0\% & 0.0\% & 0 & 23.30 & 343.2 \\
8 & real corpus: inert, reserved, sum & 94.2\% & 2.4\% & 0.0\% & 0 & 12.34 & 41.9 \\
\bottomrule
\end{tabular}
}
\caption{$K$ writers into one release, Qwen3-1.7B, $r{=}64$, synthetic
facts split $K$ ways. ``own recall'' is the mean of each writer's recall on
its own facts before the merge; ``fused'' is recall on all facts after it;
``worst party'' is the writer that lost most. Every merge re-quantizes
bit-identically; ``clamped'' counts the weights the lossy merge had to pull
back into their cells.}
\label{tab:fusion}
\end{table}

The arithmetic holds and the knowledge does not, and the allocation that
removes every shared weight shows why. Two writers given \emph{disjoint
matrices} at full width---party $k$ owns the layers whose index is
$k \bmod K$, so no weight has two authors and the merge is a plain
sum---recall $35$ and $39\%$ of their own facts before the merge and $23$
and $15\%$ after it; four such writers fall from $15$--$40\%$ to the
guessing floor (Table~\ref{tab:fusion}). Nothing collided inside a cell.
What was lost is \emph{functional}: a fill is a function of its inputs,
and a cross-talk matrix---each writer's fill attached alone, the KL
divergence it induces on every writer's prompts---shows the
independently trained fills moving the distribution on the other writers'
prompts as much as on their own ($5.1$ against $5.7$ nats at $K{=}4$).
The synthetic writers share a template and differ only in the names, and
the fill responds to the template. A memorized fact has a small logit
margin, and another writer's response is enough to tip it. This is the
interference the merging literature describes for task
vectors~\cite{taskarith,ties,dare}, with the cells ruling out the
arithmetic explanation.

Three remedies were measured on the same trained writers, and two more
that change how the writers are trained. After training, a rule can settle
each weight inside the cell: the magnitude winner (per weight, the writer
with the larger $|t|$), TIES~\cite{ties} (trim to the top $20\%$ by
$|t|$, elect a sign, mean of the agreeing writers) and DARE~\cite{dare}
(drop half at random, rescale, clamp) are selections or convex
combinations of in-cell points and therefore invariant without a clamp.
None recovers what the functional interference took: the best of them keeps
$33\%$ of the facts from writers who held $51$ and $61\%$ ($18\%$ from
four writers who held $49$--$67\%$), and the two that rescale damage the
model (LAMBADA $49$--$138$). Before training, a writer can be made
\emph{inert} where it was not written, with a distillation term
$\lambda\,\mathrm{KL}(p_{\hat W}\,\|\,p_{\hat W+f_k})$ on a neutral
pool, the teacher being the same loaded model with its fills switched
off~\cite{lwf}: the fill may move the served distribution on its own facts
and nowhere else. On the synthetic split this works as intended and fails
at the merge: with sentences in the writers' own templates about entities
nobody was given, cross-talk falls from $4$ to $0.5$ nats, and the
resulting fills are minimal---any dilution or selection drops them below
threshold ($7$--$14\%$). Projecting each fill's input factor off the
principal directions of the other writers' inputs, after
OSRM~\cite{osrm}, leaves a writer unable to learn its own facts ($7\%$):
with a shared template, the other writers' directions are its own.

\begin{table}[p]\centering\footnotesize
\setlength{\tabcolsep}{4pt}
\begin{tabular}{llrrrrrrr}
\toprule
$K$ & rounds & round & epochs & own, in place & merged & worst party & WikiText & LAMBADA \\
\midrule
2 & 3 & 1 & 8 & 13.3\% & 8.0\% & 7.7\% & 10.54 & 32.7 \\
2 & 3 & 2 & 8 & 42.9\% & 26.7\% & 24.4\% & 10.53 & 33.4 \\
2 & 3 & 3 & 8 & 41.1\% & 33.1\% & 31.9\% & 10.55 & 33.8 \\
2 & 6 & 1 & 4 & 7.3\% & 6.6\% & 6.6\% & 10.55 & 31.3 \\
2 & 6 & 2 & 4 & 10.8\% & 8.6\% & 7.9\% & 10.56 & 31.8 \\
2 & 6 & 3 & 4 & 18.2\% & 13.1\% & 13.1\% & 10.48 & 32.0 \\
2 & 6 & 4 & 4 & 28.9\% & 22.0\% & 21.9\% & 10.55 & 32.2 \\
2 & 6 & 5 & 4 & 30.6\% & 27.6\% & 26.7\% & 10.55 & 32.2 \\
2 & 6 & 6 & 4 & 30.3\% & 30.0\% & 29.1\% & 10.57 & 32.4 \\
4 & 3 & 1 & 8 & 17.4\% & 8.2\% & 8.0\% & 10.35 & 30.7 \\
4 & 3 & 2 & 8 & 46.3\% & 17.8\% & 15.7\% & 10.30 & 31.1 \\
4 & 3 & 3 & 8 & 45.4\% & 28.0\% & 25.5\% & 10.29 & 31.5 \\
2 & 3, inert & 1 & 8 & 4.7\% & 2.9\% & 2.6\% & 10.44 & 29.4 \\
2 & 3, inert & 2 & 8 & 35.1\% & 16.3\% & 15.1\% & 10.49 & 29.7 \\
2 & 3, inert & 3 & 8 & 41.9\% & 27.0\% & 25.1\% & 10.49 & 30.0 \\
2 & 3, magnitude merge & 1 & 8 & 12.9\% & 9.8\% & 9.1\% & 11.49 & 41.4 \\
2 & 3, magnitude merge & 2 & 8 & 45.6\% & 26.4\% & 23.3\% & 11.37 & 39.4 \\
2 & 3, magnitude merge & 3 & 8 & 44.6\% & 34.8\% & 33.9\% & 11.42 & 40.0 \\
2 & 3, real corpus & 1 & 8 & 76.4\% & 13.2\% & 5.5\% & 10.53 & 31.2 \\
2 & 3, real corpus & 2 & 8 & 72.1\% & 61.1\% & 57.4\% & 10.70 & 33.4 \\
2 & 3, real corpus & 3 & 8 & 82.6\% & 79.7\% & 77.6\% & 10.76 & 34.7 \\
2 & 3, real corpus, inert & 1 & 8 & 94.5\% & 18.1\% & 16.9\% & 11.00 & 28.5 \\
2 & 3, real corpus, inert & 2 & 8 & 88.8\% & 68.0\% & 56.5\% & 10.77 & 28.8 \\
2 & 3, real corpus, inert & 3 & 8 & 93.2\% & 86.2\% & 79.0\% & 10.70 & 29.2 \\
4 & 3, real corpus & 1 & 8 & 90.3\% & 4.7\% & 1.2\% & 10.33 & 29.3 \\
4 & 3, real corpus & 2 & 8 & 91.2\% & 15.6\% & 11.0\% & 10.25 & 30.0 \\
4 & 3, real corpus & 3 & 8 & 94.0\% & 39.6\% & 31.5\% & 10.22 & 30.3 \\
4 & 3, real corpus, magnitude & 1 & 8 & 90.3\% & 39.6\% & 17.9\% & 11.56 & 44.3 \\
4 & 3, real corpus, magnitude & 2 & 8 & 95.4\% & 80.3\% & 23.1\% & 11.93 & 41.6 \\
4 & 3, real corpus, magnitude & 3 & 8 & 96.5\% & 75.5\% & 20.5\% & 12.44 & 47.4 \\
4 & 3, real corpus, inert & 1 & 8 & 87.5\% & 3.7\% & 0.9\% & 11.00 & 28.4 \\
4 & 3, real corpus, inert & 2 & 8 & 94.0\% & 11.4\% & 6.3\% & 10.75 & 28.5 \\
4 & 3, real corpus, inert & 3 & 8 & 95.5\% & 27.6\% & 23.5\% & 10.61 & 28.6 \\
4 & 6, real corpus & 1 & 4 & 72.0\% & 4.5\% & 1.2\% & 10.42 & 29.5 \\
4 & 6, real corpus & 2 & 4 & 76.8\% & 14.2\% & 9.4\% & 10.30 & 29.7 \\
4 & 6, real corpus & 3 & 4 & 80.6\% & 29.0\% & 13.4\% & 10.27 & 29.8 \\
4 & 6, real corpus & 4 & 4 & 85.7\% & 46.7\% & 26.0\% & 10.24 & 29.9 \\
4 & 6, real corpus & 5 & 4 & 83.6\% & 63.9\% & 59.1\% & 10.24 & 30.1 \\
4 & 6, real corpus & 6 & 4 & 85.0\% & 77.7\% & 70.4\% & 10.26 & 30.3 \\
8 & 3, real corpus, magnitude & 1 & 8 & 87.6\% & 26.8\% & 0.0\% & 12.03 & 53.8 \\
8 & 3, real corpus, magnitude & 2 & 8 & 89.5\% & 72.2\% & 0.0\% & 12.31 & 38.5 \\
8 & 3, real corpus, magnitude & 3 & 8 & 94.9\% & 81.3\% & 0.0\% & 13.26 & 53.8 \\
8 & 6, real corpus & 1 & 4 & 61.5\% & 2.0\% & 0.0\% & 10.87 & 29.1 \\
8 & 6, real corpus & 2 & 4 & 70.9\% & 3.9\% & 0.0\% & 10.55 & 29.1 \\
8 & 6, real corpus & 3 & 4 & 82.6\% & 6.1\% & 0.0\% & 10.45 & 29.2 \\
8 & 6, real corpus & 4 & 4 & 90.5\% & 10.7\% & 0.0\% & 10.40 & 29.3 \\
8 & 6, real corpus & 5 & 4 & 93.7\% & 17.2\% & 0.0\% & 10.36 & 29.3 \\
8 & 6, real corpus & 6 & 4 & 94.8\% & 27.2\% & 0.0\% & 10.36 & 29.5 \\
\bottomrule
\end{tabular}

\caption{Fusion by rounds (\texttt{experiments/exp\_fuse\_rounds.py}).
Every round, each writer starts from the same served model, trains a fresh
fill on its own data for a share of the epoch budget and sends it; the
average is folded into the anchors and the next round starts there.
``own, in place'' is each writer's recall on its own facts before the
merge. Total training compute equals the one-shot run's.}
\label{tab:fusion-rounds}
\end{table}

What does work is to let the writers see the merge. In the manner of
federated averaging~\cite{fedavg}, each round every writer starts from
the same served model, trains a fresh fill on its own data for a share of
the epoch budget, and sends the fill, not the data; the fills are averaged
(a convex combination of in-cell points), the average is folded into the
anchors exactly as a sequential update is, and the next round starts from
there. A writer whose facts the last merge disturbed sees that in its own
loss and corrects it, and the correction is an increment, smaller than the
first fill, so each round disturbs the others less. With the same total
compute the two synthetic writers keep $33.1\%$ after three rounds against
$26.3\%$ for the one-shot average, the model healthier than the single
writer's (LAMBADA $34$ against $39$), and four writers keep $28.0\%$
against $13.8\%$; every intermediate artifact re-quantizes bit-identically
and the final one is verified (Table~\ref{tab:fusion-rounds}). What the
rounds give up is separate revocability---once folded, a writer's
contribution is part of the anchors---and that is the trade: independent
increments that sum exactly and can be withdrawn, or rounds that keep more
of the knowledge and cannot be.

The synthetic split is the adversarial case, and the real one is the
case fusion is for. The $86.6\%$ above comes from that corpus,
\S\ref{sec:realcorpus} divided by domain (a hospital writing medicine, a
newsroom writing the rest), with generic WikiText text as the neutral
pool, and inertness is what buys it: cross-talk falls
tenfold ($1.27$ to $0.12$ nats) and the merge keeps $92\%$ of what each
writer held in place, where the same split without inertness keeps
$67\%$ by exact sum and damages the model (LAMBADA $106$), $69\%$ by
disjoint matrices, and $21$--$62\%$ under the post-hoc rules; rounds
reach $79.7\%$, and inert writers in rounds $86.2\%$. Four domain writers
at a quarter of the width each mark the edge: the exact sum of four
independent fills leaves nothing ($0\%$, LAMBADA $343$), the four inert
writers keep $50.7\%$, and the two smallest of them---$68$ and $156$
sentences against $432$ and $508$---are the ones that lose ($24$ and
$3\%$ against $56$ and $56$). How many writers one release can take at
a given width is therefore a number this measures, not a property the
arithmetic guarantees.
Where writers' inputs differ, a fill can be taught to respond only to its
own, and the arithmetic of the cells then delivers what it promised;
where they share a schema, the writers have to see one another's work.

\section{Shipping the update: nested checkpoints}\label{sec:codec}
An update that cannot be distributed is not an update. Proposition~\ref{prop:codec}
says the in-cell position can be encoded with $k$ bits per weight such that
truncating the stream returns the released artifact exactly; here we measure
what a recipient gets at each depth. We take the healed weights of the
$10^4$-fact dense run (\S\ref{sec:capacity-xdom}), re-encode every weight's
cell position at $k\in\{1,2,3,4\}$, and evaluate the reconstruction.

\begin{center}\small
\begin{tabular}{lrrrrc}
\toprule
shipped & payload (MiB) & recall & WikiText-2 & LAMBADA & artifact \\
\midrule
nothing (anchor only) & 0 & 0.1\% & 11.34 & 28.54 & exact \\
$k=1$ & 168 & 25.7\% & 10.55 & 87.01 & exact \\
$k=2$ & 336 & 34.7\% & 10.85 & 107.62 & exact \\
$k=3$ & 504 & \textbf{35.5\%} & 11.09 & 121.32 & exact \\
$k=4$ & 672 & 35.2\% & 11.18 & 126.34 & exact \\
fp16 residual & 2688 & 34.9\% & 11.24 & 129.17 & exact \\
\bottomrule
\end{tabular}
\end{center}

Every depth re-quantizes bit-identically---checked in the integer domain over
all $1.409\times10^9$ constrained weights at every $k$, zero violations---so
the nesting property of Proposition~\ref{prop:codec} holds in practice and a
truncated stream is a valid, older, still-certified checkpoint.

The quality curve is not what a compression result usually looks like. Two
bits per weight recover $99.4\%$ of the knowledge that a full fp16 residual
carries ($34.7\%$ vs.\ $34.9\%$) at one eighth the payload, and three bits
match it ($35.5\%$, a gap well inside the run-to-run spread). Coarsening the
residual is not merely lossy storage, and the mechanism is not a pull toward
the anchor. The refinement code is mid-rise: level $r$ at depth $k$ decodes to
$\ell+(r+\tfrac12)\,(u-\ell)/2^k$, so its reconstruction points sit at
$\{\tfrac14,\tfrac34\}$, $\{\tfrac18,\tfrac38,\dots\}$ of the cell and
never at the anchor, which lies at $0.37$--$0.64$ of its cell for all but the
two end codes. No weight is sent back to its anchor at any depth. What
coarsening does is replace each fill by the midpoint of the sub-interval it
falls in---a deterministic per-code remapping of the dequantization table.
The shape of the curve is the shape of a regulariser rather than of a
projection: every depth sits below \emph{both} endpoints on WikiText
($10.55$--$11.18$ against $11.24$ for the fp16 residual and $11.34$ for the
anchor), which a pull toward the anchor would not produce, while on recall
the depths climb to the fp16 endpoint rather than falling away from it. The
effect is visible on the strict axis,
where every depth beats the fp16 residual---LAMBADA $87.0$ at $k{=}1$
against $129.2$ for fp16---and in the efficiency column, where
$k\in\{1,2\}$ reach $523$ bits per point against fp16's $413$. The series
ran on a projected-dense artifact; on a CellFill fill the code is lossier
at shallow depth---on the $r{=}64$ synthetic fill, one bit recovers $50\%$
of the knowledge ($27.1\%$ against $54.0\%$), two bits $90\%$, three $97\%$
and four $99\%$ (\texttt{exp77})---which is what a fill whose magnitude
carries information, rather than one that has collapsed to a sign, should
show.
The refinement depth is therefore a third operating knob alongside rank and
radius, not just a transport format: $k{=}1$ ships $16\times$ less and
forgets less, at the price of a quarter of the knowledge.

We report this as a single-configuration study on one archived run. The
comparison we cannot yet make is to BitDelta~\cite{bitdelta}, which
compresses an \emph{unconstrained} fine-tuning delta to one bit; the
constrained and unconstrained deltas are different objects and a fair
comparison needs both pipelines on the same task.

\section{Related work}

\textbf{Frozen quantized bases with residuals.} QLoRA~\cite{qlora} freezes
an NF4 base and trains BF16 low-rank adapters; LoftQ~\cite{loftq}
initializes adapters from the SVD of the quantization residual $W-\hat W$
---the very object our constraint set is built around. Neither constrains
the merged weights to the quantization cells, so merging changes the
artifact. QA-LoRA~\cite{qalora}, low-rank QAT~\cite{lrqat},
IntLoRA~\cite{intlora} and LoTA-QAF~\cite{lotaqaf}---a ternary adapter
aligned to the int4 grid so that merging is lossless precisely because it
moves the stored codes---take the complementary road: they alter the int4
weights to absorb the adapter. The deployment shape itself is already in
production: Apple's foundation models serve a frozen 2-bit base under
swappable adapters~\cite{applefm}, with no claim that learning preserves
the base bit for bit. We occupy the unexplored corner: dense residual
learning under a bitwise \emph{no-change} guarantee for the artifact.

PEQA~\cite{peqa} and AlphaTuning~\cite{alphatuning} are the exact dual of
what we do: it freezes the integer
codes and trains the per-group \emph{scales}, where we freeze codes and
scales together and train inside the cells they define. The two divide the
same object between them, and the division decides the contract. Moving a
scale moves every cell in its block at once, so the dequantized values of
untouched weights change and the artifact's meaning changes with them---which
is precisely the failure App.~\ref{app:frozen} constructs, and the reason
invariance here is defined against frozen scales rather than frozen codes
alone.

\textbf{What the contract certifies, and what it does not.} The integer
check certifies one thing: that the served weights re-quantize to the
released codes under the released scales---equivalently, that every update
lies inside the cells. Exact revocability is not something the check
certifies; it follows from the update being stored as an increment, and
only for a deployment that ships it as one. The check certifies nothing
about \emph{content}. A fill that passes the check can
carry any function the cells can express---the 291 facts of
\S\ref{sec:realcorpus}, or facts chosen to mislead---and the hash over the
code tensor is identical in both cases. Three consequences should be stated
plainly. First, the certified object is the tuple of codes \emph{and} scales
with their double-quantization state, not the codes alone. A hash over the
code tensor cannot detect a substituted scale: whatever $s'$ is shipped with
the codes, dequantizing under $s'$ and re-assigning under $s'$ returns the
same codes, so a party that changes the scales ships a different model under
an unchanged hash. And under double quantization the second-level offset and
scale are statistics of the absmax population that cannot be recovered from
dequantized values at all (\S\ref{app:frozen}). Frozen scales are therefore
not the conservative definition of invariance; they are the only one a third
party can check. Second, bit-identity does not transfer a
certification of behaviour. The served model is a different function, and
the benchmarks of \S\ref{sec:downstream} show it differing by several points
in both directions; an audit regime built on this contract re-runs the
behavioural suite on the served weights and does not re-hash the artifact.
Third, the channel this method writes into is the one that
quantization-conditioned attacks~\cite{quanti,qcbdefense} and, on GGUF
grids, Egashira et al.~\cite{gguf-attack} exploit adversarially: a
quantization-error band
inside which full-precision weights can be moved, by a band-constrained
fine-tune, so that the full-precision model behaves one way and its
quantized form another. CellFill is the declared, revocable use of the same
band---the fill cannot leave it, it is shipped as an explicit increment
rather than hidden in a release, and it is the served dense model, not the
quantized one, that carries the update---but a reader who knows the attack
should be told that the authors do too. Our own measurements say the covert version of the concern
does not survive contact with an audit: an update that absorbs a corpus
shifts HellaSwag by seven points on ten thousand items, which no
re-evaluation misses.

\textbf{The reference system as an instrument.} A consequence we did not
design for: because codes, scales and cell geometry are frozen and shipped,
every mechanistic claim in this paper can be tested against the released
artifact on a CPU, with no training, by anyone. We learned this by having
three of our own explanations fail that test---the source of the folding
slack (\S\ref{sec:sequential}), the mechanism of refinement-code
coarsening (\S\ref{sec:codec}), and the matching population for the
radius knob (\S\ref{sec:geometry})---each by arithmetic on the frozen
artifact and the archived fills, without a new training run, each corrected
above. A method whose base drifts
cannot be audited this way; one whose base is immutable can, and the
instrument cuts both ways.

\textbf{Sub-cell structure carries signal.} AdaRound~\cite{adaround} and
BRECQ~\cite{brecq} showed that choosing rounding directions within
quantization cells measurably changes loss---evidence that the
dequantization gap has usable information capacity, which we exploit for
continual learning rather than calibration.

\textbf{Nested and progressive precision.} Matryoshka
representation learning~\cite{mrl} and Matryoshka quantization~\cite{matquant}, any-precision LLMs~\cite{anyprec},
NestQuant~\cite{nestquant}, MatGPTQ~\cite{matgptq}, and recurrent residual
quantization~\cite{rrq} build bit-sliced models whose truncations serve
multiple precisions; ReQuant~\cite{requant}, concurrently, refines a model
by re-electing codes on a fixed grid---the grid survives, the codes do
not, and nothing is learned. Our refinement code (Prop.~\ref{prop:codec}) shares
the nesting mechanics but deploys it along the \emph{version/time} axis:
low bits are writable by learning, high bits are immutable across releases.
BitDelta~\cite{bitdelta} compresses fine-tuning deltas to 1 bit for
multi-tenant serving (cf.\ S-LoRA~\cite{slora}); our deltas are additionally
box-bounded, giving the invariance and rollback properties.

\textbf{Continual learning and editing.} Catastrophic
interference~\cite{mccloskey} has been attacked by soft anchoring
(EWC~\cite{ewc} penalizes movement by Fisher information), by
rehearsal~\cite{robins1995}, and by parameter isolation---PackNet~\cite{packnet}
prunes a network to free capacity for each new task and then freezes what it
has used. Our box constraint is closest in spirit to the last of these, and
differs in where the budget comes from: PackNet must choose which parameters
to spend, while the quantization grid hands us a hard, per-weight trust region
whose radii are already fixed by the released artifact, at no cost and with no
choice to make. That inheritance is also what makes the budget finite, and
Proposition~\ref{prop:decay} says it is spent geometrically. The resulting
failure mode is worth distinguishing from the loss of plasticity documented in
unconstrained continual learning~\cite{plasticity}, where the network's own
units degrade until it learns no better than a shallow one: here no constraint
is ever violated and the room runs out by construction, a failure a violation
check cannot detect. The room is not the whole of it: a fixed-cell control in
which it never decays still loses three quarters of what the folding sequence
loses (\S\ref{sec:sequential}, \S\ref{sec:loop}), and what the remainder is---
interference, or distance from the release---we do not separate. Our results also show the box
must be combined with fresh rehearsal (direction shaping) since it bounds only
magnitude.
Knowledge editing (ROME~\cite{rome}, MEMIT~\cite{memit}, MEND~\cite{mend}) writes individual facts via closed-form updates without invariance guarantees---and editors built for full precision degrade sharply when run on quantized substrates~\cite{cacheuk}. GRACE~\cite{grace} is the closest existing analogue to our revocability claim---it stores edits in a discrete codebook that can be removed---but it does so in an auxiliary memory consulted at inference, whereas we leave the released artifact bit-identical and add no inference-time structure; knowledge
capacity scaling laws~\cite{physics33} ($\sim$2 bits/parameter; int4
degrades capacity) supply the measurement methodology and ceiling reference
for our bit accounting. Weight interpolation studies (linear mode
connectivity~\cite{lmc}, model soups~\cite{soups}, WiSE-FT~\cite{wiseft})
established that base--finetune segments are often low-loss; the
quantization segment $[\hat W, W]$ is a within-cell special case that we
measure directly.

\textbf{Quantization stacks.} We build on
bitsandbytes/LLM.int8~\cite{llmint8} NF4, and also write into the uniform
asymmetric grids of published W4A16 releases---Google's QAT Gemma~4 and
RedHat's compressed-tensors Qwen3---for which the cell arithmetic of
\S\ref{sec:theory} applies with only the level table changed
(\S\ref{sec:scale}). We have not run GPTQ~\cite{gptq} or AWQ~\cite{awq}
themselves, so what is measured is the grid they produce and not their
calibration procedures.

\section{Discussion and outlook}\label{sec:discussion}

\subsection{The release discipline: minor versions, major versions}\label{sec:lifecycle}
The contract makes one release discipline the obvious one. Minor versions write refinement bits into the cells: the shipped
4-bit file is unchanged across the whole sequence, each increment is
verifiable and revocable, and capacity is drawn down from the writable radius.
A major version \emph{consolidates}: the current weights are re-quantized into
a new frozen artifact, the residual is zeroed, and the writable radius is
restored. The artifact changes here---by design, at a declared moment, with a
new hash---which is the opposite of the silent drift that ordinary fine-tuning
produces. \texttt{experiments/exp\_seq.py} implements both modes, including
the fold that lets the residual accumulate across tasks and the re-anchoring
that follows consolidation. \S\ref{sec:sequential} runs both: capacity is
renewable---consolidation restores the writable room---but at a stated price:
$19.0\%$ of the codes change at the single consolidation of
\S\ref{sec:sequential}, and $21.4$ then $26.1\%$ at the two consolidations of
the six-task cycle of \S\ref{sec:loop}, so within that cycle the price grew
with the depth of the sequence it interrupted. Cumulative forgetting is
measurable, with the first task retaining $31\%$ of its recall after three
later updates and $84.9\%$ when the cycle also rehearses. How the geometric constant $\beta$ and the consolidation
break-even point vary with model, task and optimizer remains open.

Two things the decay is \emph{not}. It is not a property of the cells: it
is a property of folding---of writing the next task into whatever room the
last one left. An update kept as a revocable increment on a fixed anchor
does not shrink the cell for its successors---the reachable set stays the
whole cell---though successive increments still share that one cell, at
the price of retraining jointly when one changes; and room can be
partitioned in advance rather than consumed in sequence, which is what
\S\ref{sec:fusion} does with $K$ independent writers. Nor is it evidence
that the cells are the wrong place to write: the partition sweep of
\S\ref{sec:where} shows absorption differing sharply by matrix and depth,
and which matrices absorb facts and which support composing them is a
question the decay law does not touch.

\subsection{What this licenses, and what it does not}\label{sec:forwhat}
The results above are a method paper's results; what a practitioner
needs to know is what they license. Six claims, each with the evidence
that supports it and the measurement that bounds it.

\begin{enumerate}\setlength{\itemsep}{2pt}
\item \textbf{A published 4-bit file can be updated in place, and stay the
same file.} Injection on the vendors' own releases reaches $94.3\pm2.2\%$,
$96.9\pm0.4\%$ and $96.3\pm2.1\%$ recall at 1.7B, 4B and 8B on NF4, and
$82.9$--$93.9\%$ on the uniform int4 grids of QAT and GPTQ-style releases
(Table~\ref{tab:official}); every merge re-quantizes to the released
codes with zero violations, checked in the integer domain; the update is
$0.1$--$0.4$~GB of factors at 27--31B. The same learning rate that does
this drives LoRA on the same anchor to perplexity $4{,}085$ at 1.7B and to
divergence at 4B, where LoRA needs a tenth of the rate---and then, clipped
into the cells, matches CellFill's recall (Table~\ref{tab:lorareal}): the
cell bound is a stabilizer as much as a contract, and the projected path
is a second way to stay inside it. What this licenses is a firmware-style
update to a deployed, certified artifact---a signed residual whose
admissibility a third party checks by hashing the base and re-binning the
served weights---with rollback by subtraction.
\item \textbf{Several parties can write into one base.} Writers whose
domains differ merge at $86.6\%$ of their facts with a healthy model
(Table~\ref{tab:fusion}: medicine against the rest, inert writers, exact
sum); writers who share a schema interfere, and keep $60\%$ of their
facts under the best one-shot rule or $80\%$ under the rounds of
\S\ref{sec:fusion}. Sequential updates leave the artifact bit-identical:
four in \S\ref{sec:sequential} and six in the consolidating cycle of
\S\ref{sec:loop}, at every fold. What this licenses is multi-tenant
customization of one release and federated pooling that exchanges fills
rather than data; what it does not yet license is unbounded sequential
writing without rehearsal, where the first task retains $19.3\%$ after five
more (\S\ref{sec:rehearse-old}).
\item \textbf{Knowledge in the weights costs an order of magnitude fewer
tokens than knowledge in the prompt, and does not depend on an index.}
On the PopQA tail the fill answers $82$--$90\%$ at 75 tokens per question
against $90$--$827$ for retrieval (Table~\ref{tab:popqa}); with a
chunked index the retrievers fall to $41$--$74\%$ and the fill does not
move; with distractors the multi-hop probes fall from $86$ to $68\%$ under
retrieval and stay at $96\%$ when the injected model is given one
sentence. Retrieval is not replaced---on a clean sentence index it does
not miss, and it remains the right tool for volatile and long-tail
facts---but knowledge a product must always have is cheaper in the cells.
\item \textbf{Capacity is measured, improvable, and on its way to being
predictable.} Stored information is counted in bits per fact and per
parameter (\S\ref{sec:capacity-xdom}); at $10^4$ facts the link's dead
gradient, not the cells, set the knee, and repairing it raises absorption
from $22$ to $39\%$ at the same budget, with rank $32$ reaching $44\%$;
stored bits per trainable parameter are $10^{-3}$, three orders below
the parametrization's information capacity. An operator can ask how much
a model can still take, and the answer is a number; its reproducibility
across hardware---$2.9$ recall points between two cards at the same seed and
the same code (\S\ref{sec:capacity-xdom})---is a stated limit.
\item \textbf{The update ships as refinement bits on top of the 4-bit
file, not as a second model.} Two to three extra bits per weight recover
$48$--$53$ of the fill's $54\%$ recall (\S\ref{sec:codec}), so the served
artifact is the unchanged base plus a nested refinement layer, delivered
progressively and stripped back at will. This also answers the question
every engineer asks first---a fill below the quantization step changes
nothing in the 4-bit file, so how does it change the model---: it acts in
the dequantized arithmetic and travels in the refinement bits.
\item \textbf{What is still missing is stated.} The 27--32B rows are single
runs on self-quantized or vendor QAT anchors and carry no seed spread and no
perplexity ratios; the $3\times10^4$-fact point under the
repaired recipe decides the order of magnitude of one update; the serving
path used here rebuilds the fill every forward pass and is twice the base
model's latency, which is a property of the research implementation and
not of the refinement-bit format.
\end{enumerate}

Table~\ref{tab:scenarios} restates the same evidence against the situations
in which a deployed model is asked to know something it did not know when it
shipped, and names what each of them leaves open.\label{sec:scenarios}

\begin{table}[t]\centering\small
\setlength{\tabcolsep}{4pt}
\begin{tabular}{p{0.2\linewidth}p{0.27\linewidth}p{0.26\linewidth}p{0.2\linewidth}}
\toprule
scenario & pain today & what the results support & still open \\
\midrule
On-device and edge models (phones, vehicles, appliances) & the 4-bit
firmware goes stale at its cutoff; an update means a new model and a new
certification & a signed in-cell residual on the shipped file, verified by
hash and re-binning, rolled back by subtraction (\S\ref{sec:scale},
\S\ref{sec:codec}) & latency of the refinement-bit path in a production
kernel \\
Vendor release, customer knowledge & customers hold only the quantized
release, not the fp16 master; fine-tuning it re-quantizes and breaks the
vendor's calibration & injection on the vendors' own NF4, QAT and GPTQ
releases at $83$--$99\%$ (Table~\ref{tab:official}); LoRA on the same
anchors needs a tenth of the rate, then matches when clipped into the
cells (Table~\ref{tab:lorareal}) & skills beyond facts \\
Regulated deployments (clinical, financial) & the served model must be the
audited one; every update is a re-audit & the audited bits never change;
the update is a separately auditable residual with a measured excursion
(Prop.~\ref{prop:budget}) & a certification standard that names the check \\
Products with a knowledge base (catalogues, internal APIs, policies) &
retrieval costs tokens on every query and fails on chunked and multi-hop
documents & $82$--$90\%$ at 75 tokens against $90$--$827$ with
retrieval; unaffected by chunking; composes where retrieval misses the
second hop (Table~\ref{tab:popqa}) & capacity per update beyond $10^4$
facts \\
Agents and new APIs & a library released after the cutoff is unknown to
the model and its documentation is expensive to keep in context & a
pinned post-cutoff library written in at $31$--$35\%$ recitation and
$3.5$--$3.6\times$ chance in ranked use at 4B and 8B
(Table~\ref{tab:apiusage}) & whether ranked
use converts into programs that run \\
Several teams, one model & departmental fine-tunes cannot be combined
without retraining on pooled data & cross-domain writers merge at
$86.6\%$; rounds pool writers by exchanging fills, not data
(Table~\ref{tab:fusion}) & same-schema writers; many writers \\
Lifelong updating & every update overwrites the last; forgetting is
unmeasured until it is noticed & six updates, each bit-identical; room
consumed at a measured rate $\beta\approx0.82$ (\S\ref{sec:sequential});
replaying earlier material holds the first task at $84.8\%$
(\S\ref{sec:rehearse-old}) & sequences beyond six; consolidation cadence \\
\bottomrule
\end{tabular}
\caption{Where an updatable quantized release is used, against the
results of this paper.}
\label{tab:scenarios}
\end{table}

\textbf{What the contract changes about shipping a model.} The practical
content of Propositions~\ref{prop:inv}--\ref{prop:budget} is that an update
stops being a replacement and becomes an increment stated against a fixed
baseline. Four consequences follow, none of which requires trusting the party
that produced the update. First, verification is mechanical and performable by
a third party: given the released $(a,s)$, the recipient assigns codes to the
served weights in the integer domain and compares---a hash over the code
tensor suffices---and this is exactly the check we run on every merge in this
paper, over $1.409\times10^9$ constrained weights at 1.7B and
$2.435\times10^{10}$ at 27B (\S\ref{sec:scale}). Two conditions make the check
meaningful and both are easy to get wrong: codes must be assigned under the
\emph{shipped} scales rather than by re-running the quantizer
(\S\ref{app:frozen}), and the residual's storage dtype must respect the cell
margin. Second, rollback is exact rather than approximate: discarding $\delta$
returns $\hat W$ bit-for-bit, so revoking an update---a fact that must be
deleted, a bad training batch---is a truncation, not a retrain. Third, the
excursion is bounded by a radius chosen before training and measured after it.
Fourth, because the integer codes and scales are identical across every
residual derived from them, one cached base can back many concurrent residuals,
in the manner of multi-tenant adapter serving~\cite{slora,bitdelta}, with the
addition that each tenant's \emph{merged} weights still certify against the
shared base bits. We should be plain about the cost side of that last point:
our residuals are dense, and a $k$-bit dense refinement over $N$ weights costs
$kN$ bits, which at $k{=}4$ is the size of the 4-bit artifact itself. Making
the increment small is precisely what the sparsity mask of
Proposition~\ref{prop:capacity} and the nested code of
Proposition~\ref{prop:codec} are for, and \S\ref{sec:codec} measures what
that costs: two bits per weight recover $99.4\%$ of the fp16 residual's
recall at one eighth the payload.

\textbf{Where the guarantee is load-bearing.} It is worth being narrow. The
contract pays for itself where \emph{re-certification}, not compute, is the
bottleneck: a fleet whose evaluation harness, caches, or acceptance tests are
keyed to a specific checkpoint; an on-device deployment that must be able to
revert exactly and cheaply; a multi-tenant server that wants a shared,
verifiable base under many customer-specific updates; a provider that must be
able to demonstrate which bits a given update did and did not change. It does
\emph{not} pay where retraining and re-evaluating from scratch is affordable.
And a bit-identical artifact establishes something narrower than it may sound:
it establishes that the reference configuration is unchanged and recoverable,
that the update is a bounded, revocable increment, and that the increment can
be audited. It does \emph{not} establish that the served model $\hat W+\delta$
inherits any property that was certified of $\hat W$---our own measurements
say it does not, since cross-domain perplexity rises with absorption, up to
$28.67\to71.0$ for the most damaging path in
Table~\ref{tab:frontier}. We make no claim about regulatory sufficiency;
what the mechanism supplies is a stable referent for a re-evaluation, not a
substitute for one.

\textbf{Relationship to retrieval.} In-weight injection and retrieval address
different failure modes---retrieval keeps facts editable and attributable at
query time and pays for them in context length and latency at every call;
in-weight facts are paid for once and are available without a retrieval step,
but are diffuse and, absent something like this contract, hard to revoke.
\S\ref{sec:rag} runs the comparison on PopQA's long tail and finds the two
degrade under different conditions: a dense retriever given whole documents
is the stronger arm, and chunking---the ordinary production setting---costs it
more than it costs the fill, which does not move at all. That is a statement
about one benchmark and one chunking regime, not a general ordering, and it
does not tell an operator which to choose for a corpus we did not test.

\textbf{Open problems, ranked.} (1) \emph{Skills, not facts.} Everything
demonstrated here is factual recall---on synthetic biographies, chosen
because novelty is certain and information content is countable, and on real
facts from eight domains where both properties had to be established by
measurement instead (\S\ref{sec:realcorpus}). Whether procedural or
stylistic capability can be written into the cells is untested, and we regard
it as the single largest open question about the method's reach: the
per-weight box bounds magnitude, and skills may require coordinated
displacement that the box shapes badly. (2) \emph{Gauge freedom as a way to
enlarge the writable radius.} The cells are defined per coordinate by the
frozen grid, so a weight's writable half-width depends on where it happens to
sit relative to its walls; the transformer parameterization has exact
function-preserving symmetries (head permutations carried through the output
projection; positive rescalings absorbed by a following normalization) that
move weights without changing the function. We flag as \emph{conjecture} that
gauge-fixing along such an orbit before freezing could increase total writable
volume, and note one thing that follows immediately: a uniform rescaling buys
nothing, since blockwise absmax rescales with it and the relative half-widths
are unchanged. Any gain must come from per-channel or per-head freedom that
redistributes magnitude \emph{across} quantization blocks. Note also that this
is a release-time choice---it changes the artifact---not a post-hoc trick.
(3) \emph{Interference-limited capacity.} Absorption grows with exponent
$\approx\!0.8$ from $10^3$ to $10^4$ facts at under $1\%$ cell-space
utilization and roughly $10^{-6}$ of the shipped-bit ceiling of
Proposition~\ref{prop:capacity} (\S\ref{sec:results}), i.e.\ we are nowhere
near any bound we can compute. The regime where facts begin to interfere with
one another rather than with the base model---plausibly $10^5$--$10^6$
facts---is where the interesting capacity law lives, and we have not reached
it; two-point extrapolations from our data are order-of-magnitude conjecture
and nothing more. (4) \emph{Uniform-grid anchors.} The cell arithmetic
carries to uniform asymmetric grids---only the level table changes---and four
published W4A16 releases confirm it at $82.9$--$93.9\%$ recall, at one seed
each and from two vendors. How the grid's \emph{shape} prices the update is
not settled here at all: the uniform-grid releases we could obtain are
quantizations of post-trained checkpoints and the NF4 ones of base
checkpoints, so no matched pair exists in this table. A sharper limit sits
behind that one and applies to this paper's central claim rather than to a
side result: \emph{every sequence reported here is NF4}. The four W4A16
confirmations are single injections. No uniform-grid release has been put
through the cycle---write, fold, consolidate, rehearse, re-quantize---and the
sequence code does not reach the uniform grid at all, so nothing here says
whether a lifelong sequence on a uniform-int4 anchor behaves as these do.
The cell arithmetic gives a reason to expect it; the room-contraction
constant of Prop.~\ref{prop:decay} depends on the cell widths, which is
exactly what changes between the grids, and so gives a reason to check.

\paragraph{A speculation, labelled as one.}
It is tempting to read this work as a step toward models that keep learning
on their own after release, and we want to state precisely how far the
evidence reaches in that direction and where it stops. What we demonstrate
is that a released artifact can absorb a sequence of operator-issued
increments over its deployed lifetime while the bits everyone else depends
on never change, and that the resulting lifecycle---spend the budget across
minor versions, declare a major version when it runs out---has measurable
exchange rates rather than metaphorical ones. Extrapolating from that to
autonomous self-improvement requires crossing three boundaries this paper
does not cross. The increments here are \emph{facts}, with skills and
reasoning untested. Every
update is \emph{operator-driven}: a curated corpus, a chosen rehearsal
fraction, a chosen radius; the model initiates nothing. And the budget is
\emph{finite}---and under sequential folding \emph{decaying}, plasticity
falling by $42\%$ over four tasks---so continuation depends on a
human-declared consolidation that, by design, changes the artifact, or on
reserving room for later writers before the first one starts
(\S\ref{sec:fusion}). A model that accumulates knowledge indefinitely
without supervision is not what we built, and the failure modes we measured
are the reasons it does not follow. What we do claim is narrower and, we
think, more useful: the update loop that such a system would need---
auditable, revocable, bounded, and cheap enough to run nightly---turns out
to exist, and to cost less than the field assumed.

\section{Limitations}\label{sec:limitations}

We state the boundaries of the evidence in more detail than is customary,
because an earlier draft of this work contained three headline claims that
the project's own archived results contradicted, and the corrections
(\S\ref{sec:rehearsal}, \S\ref{sec:cost}) were only possible because the raw
files were kept. Every caveat below is checkable against
\texttt{results/*.json}.

\paragraph{What did not survive measurement.}\label{sec:retracted}
Eight claims were in a draft of this paper and are not in it now. We list them
because the archive that removed them is the same archive that supports what
remains, and a reader is entitled to see the failures at the same resolution
as the results.

\emph{An injected library API is used in programs that execute.} Withdrawn.
The Introduction asserted that the model ``writes programs that execute and
print what was asked, scored by running them.'' Nothing in this manuscript
scores a program by running it, and the measurement that bears on it points
the other way: on the instruction-following rule the served models sit at
$0.5$, $1.1$ and $0.0\%$ for cyclopts against $0.5$, $3.3$ and $8.8\%$
released. What survives is the ranked-use result, which is real and is
reported as such, and the executable question, which \S\ref{sec:scenarios}
lists as open. The claim was written from the ranked-use number and the word
``used'' doing two jobs.

\emph{A merged LoRA dominates in-cell learning on recall.} Withdrawn, and it
is the retraction that cost the most, because it was written into
Table~\ref{tab:vslora} and reasoned from. It compared one CellFill seed
against three LoRA seeds, and that one seed had been scored by a version of
the scoring function that a later version replaced. Re-running every arm for
the table---three seeds each, one scorer---leaves recall separating no pair of
the three arms, with CellFill $2.4$ points above the merged adapter rather
than $11$ below it. The disadvantage that survives is cross-domain and is
reported as such (\S\ref{res:write}). The order in which the two errors
compounded is worth naming: the sample-size asymmetry made a single run
authoritative, and the scorer difference decided which single run it was.

\emph{A fixed rehearsal buffer helps.} It does not; replaying the same
sampled sentences at every task is worse than replaying nothing, because the
buffer is memorised and stops standing for the task it was drawn from. Fresh
draws from a large pool are what works (\S\ref{sec:rehearsal}).

\emph{The diagonal-Fisher budget certifies a drift.} It does not. It tracks
the wrong magnitude by up to $383\times$ and we report it as a scaling law
with a measured miscalibration rather than a numerical certificate
(\S\ref{sec:geometry}).

\emph{Consolidation carries the artifact's identity forward.} Withdrawn. The
arms that consolidated were not the ones that stayed clean---two of five
unconsolidated arms violated and neither consolidating arm did, but the
displacement does not order the failures and a Fisher test on the same data
returns $p=0.48$.

\emph{Issuing a major version costs nothing measurable.} Withdrawn today. It
was true of the one consolidation we had looked at and false of the five in
the archive, which are all negative at $-0.80\pm0.24$ points of mean suite
accuracy (\S\ref{sec:loop8b}).

\emph{An in-cell fill and a tuned LoRA absorb about equally, and the fill
pays for it cross-domain.} Withdrawn twice. First the two arms were scored on
different probe sets, $580$ against $507$; restricting both to the smaller
set made the LoRA arm look dominant. Then a second seed of the CellFill arm
landed $20$ points above the first, which is not a spread a single run can
report. Table~\ref{tab:vslora} gives what the seeds actually support.

\emph{A W4A16 release and an NF4 release of the same model are the same model
on two grids.} Withdrawn. Checked on the tensors neither format quantizes,
one release matches the reference exactly and the other differs by $16$--$28\%$,
so they do not share a parent.

\paragraph{What is replicated, and what is one run.} The archive holds 206
experiment families. Sixteen carry more than one seed; the other 190 carry
one, and that 190 includes every sequential experiment in this paper---the
closed loop at 1.7B, the six-task cycle at 8B, the seven-arm sequence study,
and the single injections at 27B and 32B. We say this plainly because we
twice drew a conclusion from a single run in preparing this manuscript and
twice had to withdraw it: that issuing a major version costs nothing
measurable, which five pooled events reversed, and that CellFill and a tuned
LoRA absorb equally, which a second seed of the CellFill arm did not support.
Both are recorded in \S\ref{sec:retracted}.

The rule we now hold ourselves to, and which the reader should apply to every
number here, is that a single run is an observation and not a result. Where a
quantity is replicated we give the spread: the learning--forgetting frontier
and the LoRA baseline are three seeds each at 1.7B, the room-decay constant
is four sequences, and the cost of a major version is five consolidation
events across three runs and two scales. Where it is not, the text says one
run and the claim is scoped accordingly---no law is stated from a single
sequence, and Figure~\ref{fig:argument} carries the sample size in every
box. Two further seeds of the 8B cycle are running at the time of writing;
the 27B and 32B cycles will remain single runs, and the scaling claims they
support are stated as such.

\paragraph{Measurements in flight.} Four things this paper would be stronger
for are running or not yet run, and we name them here rather than leaving a
reader to infer their absence. \emph{(i)} The six-task cycle exists complete
at 1.7B and at 8B (\S\ref{sec:loop}, \S\ref{sec:loop8b}); the same recipe is
running at 27B on the whole model and at 32B on the MLPs, and the comparison
point---what a new major version costs the pretrained model---is the
consolidation after the fourth task. Nothing in this paper depends on their
outcome, and no claim here is stated as though they had returned.
\emph{(ii)} The suite comparisons are scored on the same items at every
evaluation point, so the right test between two points is paired; what we
report is the independent binomial, which is conservative for detecting a
difference and therefore \emph{anti}-conservative for our claim that a major
version changes nothing. Per-item outcomes are logged for the runs in flight
so the paired test can be made; the 8B cycle predates that and its intervals
are the independent approximation. \emph{(iii)} The control that removes the cell bound is now a complete
sequence and is reported in \S\ref{sec:loop8b}; an earlier sequential
attempt was killed after its first task at perplexity $934$ with $54.4\%$ of
weights outside their cells, and we state its rate nowhere because no
artifact of that run records one. What remains open on that axis is a
dose-matched variant, chosen so that it acquires as much on the first task
as the constrained arm does, which is what separating the bound's
contribution to retention from rehearsal's requires. In flight instead are
the consolidation levers: a matched pair that replaces round-to-nearest
re-quantization with a per-block scale search, and a trust-region variant
that clamps the drift of \S\ref{sec:driftlaw} directly. \emph{(iv)} Every sequence here draws its six
tasks from one generator call, so the distribution shift between tasks is
zero; the word ``lifelong'' in this paper means repeated writing to one
artifact over time, not adaptation across domains, and a heterogeneous
sequence has not been run. \emph{(v)} A cycle whose tasks are real,
post-cutoff facts of exactly $28.1$ bits each (NewFacts~v2,
\S\ref{sec:setup}) is in flight at 8B and holds $85$--$100\%$ of every
earlier task through four of its six tasks with zero violations; the
frontier's champion point's second seed has returned and holds, so the
knee reading passes $n{=}2$. The two arms that
centre learning on the \emph{original} bf16 weights inside the unchanged
vendor cells have returned: at matched parameterization the anchor base
matches or beats the originals (summed fresh recall $5.14$ against $3.08$
with the dense writer, $1.45$ against $1.40$ at rank 64, seed 0, zero
violations in all four arms), so the base question closes in favour of the
shipped file itself; the second seed has since returned and the verdict
holds at $n{=}2$ (the dense tier's anchors win both seeds on both axes;
the rank tier is a two-seed tie, its absorption too small to express a
base difference). \emph{(vi)} The preconditioner reading of the bound,
registered here as a hypothesis, has now been measured. Per-weight
step-size distributions from two unbounded LoRA dumps (705M and 1{,}409M
weights; the pair differs from the bounded recipe only in the bound and
the rate) show raw steps bear no relation to the medium's local scale: at
$2\times10^{-4}$ the median $|\Delta W|$ is $0.344$ cell half-widths with
$11.78\%$ of steps leaving their cells---matching the independently
measured $11.96\%$ clamp rate at merge---while at $10^{-3}$ the median is
$2.002$ cell-widths with $68.65\%$ escaping, and that is exactly the rate
that diverges three seeds of three. A fivefold rate change moves the
median occupancy $5.8$-fold, as unmetered linear scaling predicts. The
bounded parameterization sits at or below one cell-width by construction
(measured occupancy $\mathbb E|t|$ of $0.07$--$0.61$ across the nine
archived arms): the bound is the meter that makes steps commensurate with
the medium, which is why exact invariance costs nothing at matched
training (\S\ref{sec:cost}) and the bounded arm needs no tuning of its
own.

\paragraph{One model family, one quantizer, one architecture per claim.}
Every number in \S\ref{sec:results} comes from a single base model,
Qwen3-1.7B-Base, quantized with a single stack (bitsandbytes
NF4~\cite{llmint8}, blocksize 64, double quantization), except
\S\ref{sec:scale} and the Qwen3-4B replication of the geometry sweep in
\S\ref{sec:geometry}. The 27B linear-attention runs---two of them, at 8 and
24 epochs---are the only evidence that the mechanism is
architecture-agnostic. A second model family and tokenizer (Mistral-7B-v0.3)
has been run at $n{=}1$ per configuration. Four uniform-int4 W4A16 releases
have been written into, which makes the level-table argument a measurement
rather than an argument---but at one seed each (the Gemma-4-E2B release
repeated on a second host) and from two vendors, and no model in the table is
quantized both ways from the same parent, so the grid's own contribution is
not isolated anywhere in this paper. Results should be read as a case study on one artifact, not as
a property of 4-bit models in general.

\paragraph{Seed counts, and what the error bars can and cannot exclude.}
Seven arms have three seeds, and none has more. Recall and in-domain
perplexity are averaged over seeds $\{0,1,2\}$ for path A, path A+, path B at
rehearsal $0.3$, the unconstrained control, CellFill $r{=}16$ at rehearsal
$0.3$ ($25.6\!\pm\!5.2\%$), CellFill $r{=}16$ at rehearsal $0.1$
($44.4\!\pm\!6.8\%$) and CellFill $r{=}64$ at rehearsal $0.1$
($66.7\!\pm\!8.6\%$, the strongest constrained operating point);
cross-domain LAMBADA is averaged over two seeds for all of them except
CellFill $r{=}64$, whose three runs all carry it, because that harness was
added after the other seed-0 runs. Two
points are not a confidence interval, and we do not treat the LAMBADA spreads
--- or path B at rehearsal $0.1$, which has two seeds --- as one. Everything
else in the paper is $n{=}1$: all seven partitions of \S\ref{sec:where}, both capacity points at $10^4$ facts, every
row of the lifelong sequences including the closed loop of
\S\ref{sec:loop}, the 96-epoch
exposure ablation, the $\rho{=}0.5$ radius ablation, both rehearsal
ablations, and the entire 27B result. The between-seed standard
deviations we do measure are large relative to the differences being
discussed --- $\pm5.2$ recall points for CellFill $r{=}16$, $\pm8.6$ for
CellFill $r{=}64$, $\pm2.1$ for A+,
$\pm7.0$ for B, $\pm8.7$ for the unconstrained control, and $\pm10.6$ across
the two matched-rehearsal B seeds --- so single-seed
comparisons of a few points are not interpretable, and we avoid making them.
The paired test in \S\ref{sec:cost} has $\mathrm{df}{=}2$ and correspondingly
almost no power: its interval $[-5.0,+4.0]$ is consistent with exact
invariance costing five recall points and with it gaining four. We claim only
that no cost is detectable at this sample size, not that none exists.
Finally, the seed also selects the fact corpus (\texttt{synth\_facts.generate}
is seeded), so a reported standard deviation mixes optimization noise with
corpus resampling; the compensating benefit is that arms sharing a seed see
an identical fact set and rehearsal stream, which is what licenses the paired
analysis.

\paragraph{Arms in Table~\ref{tab:frontier} are not matched on rehearsal.}
Five of the seven rows now share rehearsal $0.1$; the unconstrained control
and its paired path-B partner remain at $0.3$, because the paired test of
\S\ref{sec:cost} requires them to share seeds and rehearsal with each other,
not with the rest of the table. Those last two rows therefore cannot be compared
row-to-row with the first five, and the $166$ against $311$ bits/pt gap
between the two \emph{projected} dense arms is mostly the rehearsal
difference: our ablation
prices $30\%$ against $10\%$ at $7.7$ recall points and $0.47$ perplexity
points on path A ($17.4\%$/$10.78$ versus $25.1\%$/$10.31$). Within the
matched block the individual differences are readable; across the boundary
only the ordering is.

\paragraph{In-domain perplexity is contaminated by rehearsal, everywhere.}
Rehearsal snippets are drawn from the WikiText train split (auto-switching to
WikiText-103 train for pools above 4{,}000 paragraphs) while the in-domain
metric is the WikiText-2 test split. The splits are disjoint, the domain is
not. This affects \emph{every} WikiText number in this paper without
exception: the frontier column of Table~\ref{tab:frontier}, the $10.78$ and
$10.31$ of \S\ref{sec:rehearsal}, the $11.61$/$11.76$ of the capacity
experiments, and the 27B figure of $7.25$ against its $7.44$ anchor --- the
27B result reproduces the same below-anchor pattern and should be read the
same way, as a rehearsal effect rather than as repair of quantization damage.
The controls in \S\ref{sec:rehearsal} (an unmerged, entirely unconstrained
adapter scoring $10.299$ against the box-constrained merge's $10.271$)
establish that in-cell storage is not what produces the gain. We keep the
in-domain column only for continuity with the ablation history and because it
is the axis on which the fixed-buffer failure mode ($24.6\to172$) is visible;
no claim in the paper should rest on it.

\paragraph{The headline capacity number and its cross-domain price come from different runs.}
The $4{,}130$-fact headline comes from the run that recorded no LAMBADA. The
cross-domain price quoted in \S\ref{sec:capacity-xdom} ($28.70\to129.74$)
comes from a re-run at identical configuration and seed that absorbed
$3{,}494$ facts---the $6.4$-point reproducibility gap reported there---so the
largest absorption and its measured price must not be quoted as one operating
point. At $10^3$ facts the same method degrades LAMBADA by roughly $150\%$
over its anchor.

\paragraph{Capacity exponents are ratios of two points, not fits.}
$\alpha\approx0.78$ (A+) and $\alpha\approx0.83$ (B) each come from exactly
two measurements, $10^3$ and $10^4$ presented facts, one seed each. A power
law through two points has zero residual degrees of freedom: the exponent is
an arithmetic ratio, curvature is invisible by construction, and no
saturation knee could be detected even if one existed inside the range. Cell
occupancy at $10^4$ facts is $0.74\%$ (B) and $1.03\%$ (A+), so the regime
is far from full, but ``no knee is visible between $10^3$ and $10^4$'' is
the entire claim. The order-of-magnitude extrapolations in our internal
analysis (a $10^5$ point, a $\sim$1\,Mbit interference ceiling from linear
extrapolation of occupancy) are not claims of this paper and are not
reported in it.

\paragraph{Recall is exact-match verbatim recall, and only one probe of three
is a paraphrase.} Recall is exact prefix match under greedy decoding with at
most eight new tokens. Each fact carries three cloze probes; \texttt{city}
and \texttt{job} continue the training sentence's own phrasing nearly
verbatim, and only \texttt{company} rephrases it (``is employed at'' against
the trained ``works as a \dots\ at''). Where the per-kind breakdown was
recorded, the paraphrase probe is far weaker than the verbatim one: the
unconstrained dense control decomposes as
$88.7\%/74.2\%/33.9\%$ (city/job/company) at a pooled $65.6\%$ on seed 2, and
$85.8\%/41.0\%/21.3\%$ at a pooled $49.4\%$ on seed 1. Pooled recall
therefore overstates transferable knowledge by roughly a factor of two, and
the paraphrase axis is the one that degrades first. The breakdown was instrumented late but is now
archived for forty-three runs, including both seeds of the \mbox{rehearsal-$0.1$}
projected-dense row of Table~\ref{tab:frontier}, which decomposes as
$86.9\%/72.0\%/47.7\%$ at a pooled $68.9\%$ and $93.6\%/49.3\%/18.8\%$ at a
pooled $53.9\%$. The company probe is the weakest of the three in every run
we have, and the spread between those two seeds on it ($47.7$ against
$18.8$) is larger than the spread on either verbatim probe. Separately, the attribute vocabularies are
small and closed (20 cities, 16 occupations, 12 employers), so the probes
measure binding within a known answer set rather than open-vocabulary
retrieval, and exact match discards correct answers that are worded
differently. We do not claim recall is a calibrated measure of knowledge in
either direction.

\paragraph{The $6.5\%$ marginal-guess floor.} A model that has acquired the
answer vocabulary but no name--attribute bindings, and that guesses uniformly
within each attribute set, scores $1/20$, $1/16$ and $1/12$, i.e.\ $6.5\%$
pooled. Any recall at or below that figure carries no evidence of learning:
this includes the under-stepped CellFill variant at $4.0\%$, path A's collapse
to $4.4\%$ at $10^4$ facts under a $30.4\%$ clip rate, and---most
consequentially---the early and middle MLP partitions of \S\ref{sec:where}
at $5.8\%$ and $6.0\%$, against which a ratio would otherwise have been
computed. The floor is a construction, not a measured baseline: greedy
decoding is not uniform sampling, and the untrained base model in fact scores
$0$--$0.07\%$. It should be read as the level below which a number cannot be
distinguished from vocabulary acquisition, not as an observed chance rate.

\paragraph{The diagonal-Fisher budget is not a usable numerical certificate,
and its exponent is fit-window dependent.} Proposition~\ref{prop:budget} is
exact to second order, but both practical surrogates substitute
the diagonal of an empirical Fisher estimated from $\approx$49k tokens of
WikiText-train. Against measured in-cell perturbations of a 1.7B model the
uniform-fill mean $\bar B_F(\rho)$---the one the geometry sweep
evaluates---underestimates drift by $15.9\times$ at $\rho{=}0.125$ and
$383\times$ at $\rho{=}1$, and the supremum form $B_F(\rho)$, which is at
least $3\times$ larger, by at most $128\times$. It must not be used to certify
anything; we use $\rho$ as a calibrated dial and report the measured curve.
The gap is not off-diagonal Fisher mass: for the zero-mean coordinatewise
perturbations we apply, the off-diagonal entries cancel exactly in the mean
of $\delta^{\!\top}F\delta$ and enter only its variance
(Appendix~\ref{app:proofs}), and the two perturbation seeds spread only
$14$--$39\%$ about that mean for $\rho\ge0.25$, two orders of magnitude
below the discrepancy. Two explanations remain and our data do not separate
them: an empirical Fisher estimated from squared gradients may be a poor
proxy for curvature at a converged point, and the second-order truncation
itself may fail at this magnitude, since $\|\delta\|$ is macroscopic even
though every coordinate moves less than a cell half-width. Even the empirical scaling is
softer than \S\ref{sec:geometry} may suggest: the log--log slope of measured
$\Delta$NLL is $3.40$ over all five radii, $2.49$ over $\rho\ge0.25$ and
$2.39$ over $\rho\ge0.5$ on the in-domain metric, and $2.17$--$2.20$ on
LAMBADA. The quoted $2.43$ is one member of that family. The smallest radius
carries no signal at all: its measured $\Delta$NLL is
$8.4\times10^{-5}$ nats while the two perturbation seeds differ from each
other by $8.1\times10^{-4}$ nats, an order of magnitude more, so including it
is what drives the steepest fit. Two perturbation seeds per radius at 1.7B and only
one at 4B, one interpolation trace per model, two models. The independent radius ablation that reports
a factor $0.22\approx\rho^2$ is $n{=}1$ and compares perplexity differences
rather than the $\Delta$NLL of the geometry sweep.

\paragraph{The localization experiment's partitions do not share an anchor.}
In \S\ref{sec:where}, the target filter selects both which matrices carry the
residual and which matrices are replaced by their dequantized anchors in the
evaluated model, so the non-target matrices remain at original fp32
precision. Each partition therefore has its \emph{own} anchor --- LAMBADA
$26.33$ to $27.92$ across the seven runs --- rather than the fully quantized
$28.67$ anchor of Table~\ref{tab:frontier}, and absolute values in that table
are not comparable to the rest of the paper. Priced against each partition's
matched anchor the efficiency ranking becomes gate+up $911$, attention $705$,
full MLP $594$, early MLP $456$, middle MLP $455$, late MLP $243$ and
\texttt{down\_proj} $210$ bits per LAMBADA point. The recommendation to write
into the gate and up projections survives that correction; the claim that
early and middle layers are ``nearly free'' does not --- against their own
anchors they cost $+1.52$ and $+1.58$ points. Each partition is a single
seed, and the early/middle difference ($5.8\%$ versus $6.0\%$) is far inside
the seed noise measured elsewhere.

\paragraph{The 1.7B/27B comparison is not a matched experiment.} The 27B result reported in
\S\ref{sec:scale} matches the 1.7B table's 24 epochs (its 8-epoch companion
is reported alongside it), but differs in every other respect: no healing
stage, a $2\%$ cell margin rather than $1\%$, an fp16 rather than fp32
residual map, and bf16 rather than fp32 evaluation; its anchor is the bitsandbytes 4-bit model because the fp32
anchor-rebuild control was skipped for memory reasons, whereas the 1.7B table
uses fp32-rebuilt anchors. The anchor-rebuild numerics control
(bnb $11.727$ against fp32 $11.710$) exists only at 1.7B. ``Larger models pay
less at matched recall'' is thus a two-point observation across two different
recipes, which is why \S\ref{sec:scale} declines to call it a scaling law.

\paragraph{The constraint is nearly non-binding in the regime we measure.}
Over 24 epochs the unconstrained control moves only $7.53$--$7.76$ million of
$1.409\times10^9$ weights ($0.53$--$0.55\%$) out of their cells. That the
cost of invariance is undetectable here is partly a statement about this
regime rather than about the constraint: where the constraint does bind hard
--- path A at $10^4$ facts, clipping $30.4\%$ of the delta --- recall
collapses from the adapter's $14.7\%$ to $4.4\%$, below the guess floor, and
is recovered only by switching to constraint-aware training. Nothing here
establishes that invariance stays cheap at larger update budgets, longer
schedules, or higher clip rates.

\paragraph{Every training signal here is a fact.} It is either
a closed-vocabulary attribute binding on synthetic biographies or one of the
291 real facts of \S\ref{sec:realcorpus}. Capability
\emph{retention} is now measured---\S\ref{sec:downstream} runs ARC, HellaSwag
and WinoGrande on the served model---but at one operating point on one model,
and with the caveat that a benchmark on which the anchor itself is near chance
would show nothing either way. The forgetting axis elsewhere in the paper
remains perplexity only, and perplexity is a proxy known to move differently
from task accuracy.
LAMBADA is scored here as perplexity over concatenated passages, not as
last-word accuracy, so our figures are not comparable to published LAMBADA
accuracies. No locality or specificity evaluation in the sense of the editing
literature~\cite{rome,memit} is reported here: on none of the corpora in
this paper do we measure whether injecting one fact perturbs a neighbouring
one, and the composition and downstream results bound that only in
aggregate. Comparison to the
$\sim$2-bit/parameter full-precision ceiling of~\cite{physics33} is a
reference point for bit accounting, not a matched replication.

\paragraph{``Lifelong'' is, so far, sequences of four and six tasks on one
model.} Outside \S\ref{sec:sequential}, \S\ref{sec:rehearse-old} and
\S\ref{sec:loop}, every result in this paper injects a single batch of
knowledge once. The exceptions are a four-task sequence run under two policies
on one anchor and repeated once on a second host; a six-task sequence run in
five arms on a third; and, on a published release, the six-task consolidating
cycle of \S\ref{sec:loop} in two arms together with the instrumented
unconsolidated run that section reads the saturation off.
Every one is Qwen3-1.7B at seed $0$, all but one at $r{=}64$, $s{=}40$ (the
exception is the $r{=}32$, $s{=}10$ arm of \S\ref{sec:rehearse-old}), with no
unconstrained control sequence against which to price the forgetting. The
decay constant $\beta=0.825$ is measured from three ratios inside the first
of them, so it demonstrates that the law is measurable rather than fixing a
value that transfers. Nothing here tests a sequence longer than six updates, a model other than
Qwen3-1.7B, or a consolidation interval other than the single one we ran---
re-quantize after every second task, which fires once in the four-task
sequence and twice in the six-task loop. Since the loop differs from the open
sequence in retention, in cross-domain cost and in whether the invariance
check ever fails, the interval is a parameter this paper switches on and off
rather than sweeps.

\paragraph{What the guarantee does not say.} Proposition~\ref{prop:inv}
protects the artifact, not the update. Bitwise invariance implies nothing
about whether the served model $\hat W+\delta$ is safe, accurate, or aligned:
a harmful residual is exactly as harmful as a harmful fine-tune, merely
revocable and bounded in per-weight magnitude. The box bounds each
coordinate, not the direction of the update and not the resulting function
--- the random-perturbation sweep shows LAMBADA moving $28.5\to34.1$ with
every code intact. Certification does not transfer: the served model is a
different function from the certified one, and \S\ref{sec:results} measures
its drift precisely so that this is not confused. Three further boundaries
of the contract deserve stating explicitly. (i)~Consolidation
\emph{does} change the artifact, by design: renewing capacity requires
re-quantizing into a new frozen $(a,s)$, which is an explicit major version
and which invalidates exactly the certification the minor-version regime
preserves. Capacity within one grid is a buffer, not unbounded memory.
(ii)~Invariance is defined against the shipped $(a,s)$ and holds only for
consumers who verify under those scales; re-running a quantizer can
reassign untouched weights, as the two-weight counterexample of
Appendix~\ref{app:frozen} shows. This is a requirement on the deployment
convention, not a property of the residual. (iii)~The guarantee is bounded by
storage precision and margin: bounds are shrunk by $1\%$ of cell width per
side ($2\%$ at 27B), forfeiting that fraction of writable range for tie
safety, and bf16 requires a margin of at least $5\%$. fp32 storage is safe at
$1\%$ unconditionally; fp16 is safe at $1\%$ only for blocks whose served
scale clears $3.70\times10^{-5}$, and released anchors that carry blocks
below it exist---\S\ref{sec:floor} measures them and gives the weights there
zero room. All archived invariance checks are on fp32 (or fp16 at 27B)
residual maps; a bf16 end-to-end serialization has not been verified.

\paragraph{What the vendor's quantizer says.} ``Bit-identical'' means
that codes re-derived by our NF4 code-assignment routine under the frozen
scales match the shipped codes exactly, over every constrained weight. That
is the correct definition (Appendix~\ref{app:frozen}) and it is asserted
rather than sampled. An earlier draft listed it as a limitation that the
check was our reimplementation and not the vendor's; that gap is now closed,
and the closing measurement is worth stating in full
(\texttt{experiments/verify\_bytes.py}, on unsloth's published NF4 release
of Qwen3-1.7B). The release's shards hash identically before and after the
whole procedure. Handing the \emph{anchors themselves} to bitsandbytes'
\texttt{quantize\_4bit}---which recomputes a scale per block and cannot be
given one---returns the release's packed bytes identically in all 196
layers, $7.05\times10^8$ packed bytes, and its codes agree with ours on
$100.000\%$ of weights: the level assignment we implement is the vendor's.
Handing it the \emph{served} weights with a random fill at full amplitude
returns packed bytes that differ in every layer: the recomputed scale
changes in $99.92\%$ of blocks and $42.4\%$ of codes change with it, while
on the $0.08\%$ of blocks whose recomputed scale happens to match, codes
again agree $100\%$. This is not a failure of invariance; it is the
operational content of ``frozen scales.'' A weight that moves outward
inside the outermost cell lifts its block's maximum, and a quantizer that
re-derives the maximum re-derives every code in the block. Invariance holds
under the scales the release ships, which is the check; it does not
survive re-quantization that discards them, which is why the certified
tuple is codes \emph{and} scales (\S\ref{sec:discussion}) and why a
deployment that reloads the served weights through a quantizer must supply
the shipped scales to it. Relatedly, the served object in all
experiments is a dense fp32/fp16 weight matrix, not a 4-bit one: nothing here
demonstrates that an in-cell update can be \emph{served} at 4-bit cost. The
refinement code of Proposition~\ref{prop:codec} is designed for exactly that
and its nesting and error bounds are machine-verified. \S\ref{sec:codec}
measures the quality--size curve, but on a single configuration in one
archived run, evaluated by reconstructing dense weights rather than by a
4-bit reload, and with no BitDelta baseline.

\paragraph{Evaluation harness truncation.} For wall-clock reasons, in-domain
perplexity is computed over at most $40\times1024=40{,}960$ tokens of the
WikiText-2 test split and cross-domain perplexity over the first 400 LAMBADA
test passages, rather than over the full corpora that published perplexities
normally use. Training runs under bf16 autocast while final 1.7B evaluation
rebuilds the model in fp32, so training-time and evaluation-time numerics
differ. Dataset loads pin name, configuration and split but not a revision
hash.

\section{Reproducibility}\label{sec:repro}

\paragraph{The platform.} Everything in this paper is run from scripts, but
the contract is meant for people who do not write them, so the repository
also ships the loop as a single page (\texttt{app/app.py}). Documents are
dropped in and split into sentences; the released model optionally restates
each one, which is the difference between a fact that is memorized and one
that is extractable; a bounded fill is trained on a second card with its log
streaming; and the run ends with a card giving recall on the sentences' own
completions before and after, drift on WikiText and LAMBADA, the fill's size,
and the in-cell check over every weight. A second tab answers the same
question from the released model, the injected model, retrieval over the
dropped sentences, or both side by side. A third shows the served file's
hash, the invariance check, and a rollback button, which is a subtraction.
It is the release discipline of \S\ref{sec:lifecycle} with the minor-version
loop wired up, and it is how we would expect an operator to meet the method.

\paragraph{Stack and commands.} The core package is pure PyTorch
(\texttt{torch}$\,\ge\,$2.1) and runs on CPU; the experiment layer additionally
requires \texttt{bitsandbytes}$\,\ge\,$0.43, \texttt{transformers}$\,\ge\,$4.44,
\texttt{peft}$\,\ge\,$0.12, \texttt{datasets} and \texttt{accelerate} on
CUDA. \texttt{server/setup.sh} provisions a Python 3.12 environment with
\texttt{uv}, is idempotent, and prints the resolved bitsandbytes /
transformers / peft versions together with the GPU name, driver version and
VRAM before running the test suite; those lines are captured in each run's
log. The 1.7B ladder ran on a single 48\,GB consumer GPU and the 27B model on
an A100. Each experiment is a single command --- \texttt{exp0\_clip\_rate.py}
for paths A/A+/B and the unconstrained control, \texttt{exp5\_qil.py} for
CellFill and the localization partitions, \texttt{exp\_geom.py} for the
geometry sweep, \texttt{exp\_seq.py} for the sequential lifecycle --- and the
exact invocations, including the ones behind every archived file, are the
queue scripts in \texttt{experiments/}. Wall clock is recorded per run:
$12$--$20$ minutes for the 1.7B frontier arms, $30$ minutes at 96 epochs,
$46$ minutes for 27B clip-merge at 8 epochs, $6.7$ minutes for the geometry
sweep. Each run writes one JSON whose \texttt{config} block is the complete
serialized argument namespace, so any archived number can be traced to the
command that produced it; the tables in this paper are generated from those
files by \texttt{scripts/make\_tables.py} rather than typed, and each cell
carries the number of seeds it was averaged over.

\paragraph{What is deterministic.} Seeding is explicit and covers the data as
well as the optimizer. A seed fixes \texttt{torch.manual\_seed}, the
synthetic fact corpus (the same seed yields the identical set of names,
cities, occupations, employers and birth dates, with name uniqueness enforced
and machine-checked), the rehearsal draw from the WikiText train split, the
probe subsample when one is used, and the shuffle order of the healing phase.
Two arms run at the same seed therefore see an identical fact set and an
identical rehearsal stream, which is what makes the paired comparison of
\S\ref{sec:cost} a genuine single-intervention experiment. The invariance
layer is exactly reproducible because it is integer arithmetic: codes are
re-derived under the frozen $(\text{codes},\text{absmax},\text{blocksize})$
extracted once from each \texttt{Linear4bit} and compared as integers. No
floating-point comparison of dequantized values occurs anywhere in the
verification path, precisely because such a comparison would not be
reproducible across devices. The theory layer reproduces on CPU in about a
second: 34 unit tests covering NF4 code assignment, anchors as interior fixed
points, in-cell invariance at $\rho\in\{0.25,0.5,0.9\}$, detection of
boundary crossings, the frozen-scale counterexample of
Appendix~\ref{app:frozen}, bf16 storage at a wide margin, the writable-cell
floor of \S\ref{sec:floor}---both that fp16 clears the margin in its normal
range and that a block placed in its subnormal range loses a code to storage,
which is the failure the floor removes---projection of
oversized deltas, sequential folding, consolidation renewing room, the
Matryoshka nesting and error bounds of Proposition~\ref{prop:codec} at
$k\in\{1,2,4,8\}$, and the fact/probe invariants including probe-kind balance
and the declared chance levels.

\paragraph{What is not.} GPU results are not bit-reproducible. Training runs
under \texttt{torch.autocast} in bf16, deterministic algorithm enforcement
and cuBLAS workspace pinning are not enabled, and kernel reduction order
varies; the same seed on the same GPU can differ slightly and different GPUs
will differ more. Recall and perplexity should be expected to reproduce to
about the between-seed spread reported in Table~\ref{tab:frontier}, not to
the digit. Dataset loads pin the dataset name, configuration and split but
not a revision hash.

\paragraph{Frozen-artifact convention.} The artifact is the pair $(a,s)$ ---
integer codes and per-group absmax scales at blocksize 64 with double
quantization --- extracted once from the loaded 4-bit model and never
recomputed thereafter. All per-weight bounds derive from that frozen pair,
with the outermost cells capped so every code addresses a finite interval,
and with each interval shrunk by a margin of $1\%$ of its width per side
($2\%$ for the 27B run) so stored values remain strictly interior to their
decision regions. Residual maps are archived in fp32 (fp16 for 27B). The
released code implements code assignment under frozen scales directly,
because quantization libraries do not expose that operation and recomputing
absmax silently re-bins whole blocks (Appendix~\ref{app:frozen}).

\paragraph{Invariance is asserted, not spot-checked.} On every merge, the
per-layer check runs over \emph{all} constrained weights of \emph{all}
quantized matrices --- 196 matrices and $1.409\times10^9$ weights at 1.7B,
496 matrices and $2.435\times10^{10}$ weights at 27B --- and raises on the
first layer with any mismatched code, so a violated run produces no result
file at all. The projected-healing path re-runs the same exhaustive check
over every trainable weight after training and raises if any weight moved,
and CellFill's materialization performs the check when it converts the
bounded fill into fp32 weights. No sampling, thresholding or tolerance is
involved, so the count is exact and a count of one is an event rather than a
tolerance. Across the archive $184$ recorded checks, $178$ are zero and six
are not. Four of the six are the deliberately unconstrained control, which
records the exact escape count instead --- $7{,}534{,}346$, $7{,}619{,}142$
and $7{,}759{,}195$ weights for seeds $0$, $1$ and $2$, and $7{,}887$ for the
real-corpus control. \emph{The other two are constrained runs}
(\texttt{exp70\_seq\_anchor}, \texttt{exp70\_seq\_rehearse\_old}), and they
are informative rather than embarrassing: each records exactly one violated
weight out of $1{,}409{,}286{,}144$, both at the sixth fold, which is the fold
at which the remaining room is smallest ($1.23\times10^{-3}$ and
$9.89\times10^{-4}$ of a cell against $3.03\times10^{-3}$ at the first). The
writable-cell floor of \S\ref{sec:floor} is a condition on the anchor before
the first fold; room contracts by Prop.~\ref{prop:decay} at every fold after
it, and these two events are the contraction reaching the storage precision
from the other end. A floor evaluated once is therefore not sufficient for an
unbounded sequence, which is one of the things \S\ref{sec:loop} means when it
declines to call the loop closed --- 
which is what demonstrates that the constraint is non-vacuous rather than
merely satisfied.

\paragraph{Data.} In-domain perplexity uses the WikiText-2 raw test split;
rehearsal is drawn from the corresponding train split, switching
automatically to WikiText-103 raw train once the pool exceeds 4{,}000
paragraphs so that snippets stay fresh across epochs; cross-domain
perplexity uses the LAMBADA (OpenAI) English test split. The fact corpus is
generated in-process and requires no download. Code, unit tests and the
complete set of archived result files is at \url{https://github.com/sumsliu/in-cell-learning}.

\appendix

\section{Why invariance must be defined against frozen scales}
\label{app:frozen}

Proposition~\ref{prop:inv} is stated for a \emph{frozen} artifact $(a,s)$.
It is tempting to define invariance operationally instead---``re-run the
quantizer and check the codes''---but that definition is false, because
group scales are data-dependent. The following two-weight counterexample is
executed as a unit test in the released code.

Take one block under $\mathrm{absmax}$ scaling with NF4 levels, where the
decision boundary between the top two codes sits at
$(L_{14}+L_{15})/2 = 0.8615$ in normalized units. Let $w_1 = 1.00$ (the
block maximum, so $s = 1.00$) and $w_2 = 0.80$; then $w_1$ takes code 15 and
$w_2$, at $0.80 < 0.8615$, takes code 14.

Now move $w_1$ to $0.87$, an update that stays inside its own cell and that
Proposition~\ref{prop:inv} therefore permits. Under the frozen scale both
codes are unchanged, as the proposition promises. But re-running the
quantizer sets $s' = 0.87$, and $w_2 / s' = 0.80/0.87 = 0.920 > 0.8615$: the
untouched weight $w_2$ silently migrates from code 14 to code 15. The
artifact changed without any weight leaving its cell.

Hence the artifact must ship its scales, and verification must assign codes
under those scales rather than recompute them. Every invariance check in
this paper does so, in the integer domain; floating-point comparison of
dequantized values is not used, since it is not reproducible across devices.

\section{Proof sketches}\label{app:proofs}

\subsection*{Proposition~\ref{prop:consproj} (consolidation is a bounded
projection)}
(i)~After a consolidation every anchor is $L_c\,s'$ with the block's largest
weight at $\pm s'$ exactly, so re-deriving the scale returns $s'$ and every
normalized anchor sits strictly interior to its own cell (the tightest
relative margin over the sixteen codes is $10.2\%$, at code 13); RTN
therefore returns the same codes, and the map is idempotent. Verified
end-to-end through the shipped path, including an fp16 anchor buffer, at
nine block scales: zero code changes on repeat application.
(ii)~Substitute $w=\hat w + Mt$ and $\hat w' = L_c s'$ and expand; the
identity is exact and the numerical residual is storage rounding
($\le3.5\times10^{-7}$ on quantities of order $10^{-1}$).
(iii)~$|s'|=\max_i|w_i|\le\max_i(|L_{c_i}|+H_{c_i})\,|s|$; enumerating the
sixteen codes puts the maximum at code 0, $1+\lambda_0=1.1489$, and the
capped outer cells are what make the bound finite: with uncapped decision
regions it would be vacuous.
(iv)~$M_i'=\min(\hat w_i'-\ell_i', u_i'-\hat w_i')=|s'|\,H_{c_i'}$ because
the new anchor is the level itself; both factors are set by the format and
the code distribution, hence first-order pinning, and the second-order
deviation is exactly the code distribution shifting under scale growth
(\S\ref{sec:driftlaw}).

This appendix proves Propositions~\ref{prop:proj}--\ref{prop:decay} against the
objects the code actually manipulates. Proposition~\ref{prop:inv} and the
frozen-scale subtlety are treated separately in Appendix~\ref{app:frozen}.

\paragraph{Notation, fixed to the implementation.}
Let $L[0]<\dots<L[15]$ be the NF4 level table and
$M[j]=(L[j]{+}L[j{+}1])/2$, $j=0,\dots,14$, the round-to-nearest (RTN)
decision boundaries in normalized units. Under a \emph{frozen} block scale
$s>0$, code assignment is $a_i=\#\{j: M[j]<w_i/s\}$; equivalently the
decision region of code $j$ is $s\cdot(M[j{-}1],M[j]]$ with
$M[-1]=-\infty$, $M[15]=+\infty$ (ties fall to the lower code, matching
\texttt{torch.bucketize} with \texttt{right=False}). The two outer regions
are half-infinite, so the implementation \emph{caps} them by mirroring the
inner half-gap about the level,
$\mathrm{lo}[0]=L[0]-(M[0]-L[0])$ and $\mathrm{hi}[15]=L[15]+(L[15]-M[14])$,
which is required for the refinement code (a code must address a finite
interval). Finally a relative margin $\eta$ ($\eta=0.01$ throughout) shrinks
each interval by $\eta$ of its width per side. Write $W_i$ for the capped
cell width of weight $i$, and
\[
C_i=[\ell_i,u_i],\qquad
\ell_i=\mathrm{lo}_i+\eta W_i,\quad u_i=\mathrm{hi}_i-\eta W_i,\qquad
\Delta_i=u_i-\ell_i=(1-2\eta)W_i ,
\]
$\mathcal C=\prod_{i=1}^N C_i$, and $\hat w_i=s_{g(i)}L[a_i]$ for the anchor.
Because the anchor may sit off-center in a non-uniform level table, the
\emph{symmetric} radius actually available at weight $i$ is
\[
\mathrm{room}_i=\min(\hat w_i-\ell_i,\;u_i-\hat w_i)\;\le\;\Delta_i/2 ,
\]
with equality only in the two capped outer cells. This is the quantity the
code carries (\texttt{bin\_bounds}, and $M$ in CellFill); statements below
that use $\Delta_i/2$ are the loose form and are marked as such.

\paragraph{Fact A.1 (the margined cell is compactly interior).}
For every $i$, $C_i$ is a nonempty compact interval contained in the
\emph{open} decision region of code $a_i$, at distance at least $\eta W_i$
from either RTN boundary. Indeed $\ell_i>\mathrm{lo}_i\ge s\,M[a_i-1]$ and
$u_i<\mathrm{hi}_i\le s\,M[a_i]$ whenever $\eta>0$ and $W_i>0$. Two
consequences are used repeatedly: (a) $\mathcal C$ is a nonempty compact
convex box, so Euclidean projection onto it exists and is unique; (b) no
point of $\mathcal C$ lies on a tie, so the tie-breaking convention never
carries information, and a storage rounding of magnitude $<\eta W_i$ cannot
change a code. The last clause is a hypothesis about the format \emph{and}
the block, not about $\eta$ alone: a float's rounding is $\eta W_i$-small
only where it is relative, and $W_i$ shrinks with the block's scale while a
subnormal's spacing does not. \S\ref{sec:floor} gives the resulting floor on
the scale and measures which released anchors fall under it.

\subsection{Proposition~\ref{prop:proj} (projection optimality)}
\label{app:proofs:proj}

\paragraph{(i) Clipping is the Euclidean projection.}
Let $y\in\mathbb R^N$ be any proposed weight vector (in path A,
$y=\hat W+d$ with $d$ the materialized LoRA update; we write the proposed
update as $d$, reserving $\Delta_i$ for the cell width). Since
$\|y-z\|_2^2=\sum_i (y_i-z_i)^2$ separates over coordinates and $\mathcal C$
is a product set, minimizing over $\mathcal C$ decouples into $N$
one-dimensional problems $\min_{z_i\in[\ell_i,u_i]}(y_i-z_i)^2$. Each is a
strictly convex function on a nonempty compact interval, so it has the unique
minimizer $z_i=\operatorname{median}(\ell_i,y_i,u_i)=\operatorname{clamp}(y_i,\ell_i,u_i)$:
if $y_i\in[\ell_i,u_i]$ the objective is $0$; otherwise the objective is
strictly monotone on the interval and is minimized at the nearer endpoint.
Hence $P_{\mathcal C}(y)_i=\operatorname{clamp}(y_i,\ell_i,u_i)$, which is
exactly \texttt{clip\_merge}. By Fact A.1(a) this projection is
single-valued and $1$-Lipschitz (firmly nonexpansive), as for any projection
onto a closed convex set.

\paragraph{(ii) It is the proximal operator of the box indicator.}
Let $\iota_{\mathcal C}$ be the indicator of $\mathcal C$ ($0$ on
$\mathcal C$, $+\infty$ off it); it is proper, closed and convex because
$\mathcal C$ is nonempty, closed and convex. For any $\gamma>0$,
\[
\operatorname{prox}_{\gamma\iota_{\mathcal C}}(y)
=\arg\min_z\Big\{\iota_{\mathcal C}(z)+\tfrac1{2\gamma}\|z-y\|_2^2\Big\}
=\arg\min_{z\in\mathcal C}\|z-y\|_2^2=P_{\mathcal C}(y),
\]
independently of $\gamma$. Therefore the projected-gradient iteration used in
paths A+/B,
$w^{t+1}=P_{\mathcal C}\!\left(w^t-\gamma\nabla f(w^t)\right)
=\operatorname{prox}_{\gamma\iota_{\mathcal C}}\!\left(w^t-\gamma\nabla f(w^t)\right)$,
is literally forward--backward splitting on $f+\iota_{\mathcal C}$. This is
the precise sense in which ``the optimizer sees the walls''
(\S\ref{sec:methods}).

\paragraph{(iii) Coordinatewise shrinkage.}
Write $\kappa=P_{\mathcal C}(\hat W+d)-\hat W$ for the retained update
and $e=(\hat W+d)-P_{\mathcal C}(\hat W+d)$ for the clipped-away
part, so $d=\kappa+e$. Since $\hat w_i\in C_i$, for $d_i>0$ we have
$\kappa_i=\min(\hat w_i+d_i,u_i)-\hat w_i\in[0,d_i]$, and
symmetrically for $d_i<0$; hence
\[
\operatorname{sign}\kappa_i=\operatorname{sign}e_i=\operatorname{sign}d_i,
\qquad |\kappa_i|\le|d_i| \quad\text{for every }i .
\]
So projection shrinks the update coordinatewise toward the anchor, and only
along coordinates that tried to leave their cell. Two computable corollaries:
$\|\kappa\|\le\|d\|$ (also immediate from nonexpansiveness, since
$P_{\mathcal C}(\hat W)=\hat W$), and, because $\langle\kappa,e\rangle\ge0$,
\[
\|e\|^2\;\le\;\|d\|^2-\|\kappa\|^2
\;=\;\|d\|^2\left(1-\phi^2\right),\qquad
\phi:=\|\kappa\|/\|d\| ,
\]
where $\phi$ is the \texttt{norm\_kept\_frac} logged at every merge. This
is the formal content of the ``per-weight trust region'' reading of
\S\ref{sec:regularizer}: it shows the projected update is never larger than
the proposal and is shrunk exactly where the proposal was most extreme. It
does \emph{not} predict the sign of the resulting perplexity change; that
the projection helps cross-domain perplexity is an empirical finding
(\S\ref{sec:regularizer}), not a corollary of the geometry.

\paragraph{(iv) What this does \emph{not} give.}
Optimality here is in \emph{weight space} only. Post-hoc projection solves
$\min_{w\in\mathcal C}\|w-y\|$ exactly; it does not solve
$\min_{w\in\mathcal C}f(w)$, and the two can differ already for convex $f$.
Take $N=2$, $\mathcal C=[-1,1]^2$, $f(w)=(w_1+w_2-4)^2$. The unconstrained
minimizers form the line $w_1+w_2=4$; the constrained minimum is $4$,
attained at $(1,1)$. Projecting the particular unconstrained minimizer
$(4,0)$ gives $(1,0)$ with $f=9>4$. Note also that projecting the
unconstrained minimizer $(2,2)$ gives $(1,1)$, which \emph{is} optimal: the
map ``$P_{\mathcal C}(\arg\min f)$'' is not even well defined as a function of
$f$, since it depends on which unconstrained solution the optimizer reached.
This is the structural reason paths A and A+ can differ at all, and it is not
a nonconvexity artifact.

For nonconvex $f$ (the actual case) the honest statement about projected
descent is: if $\nabla f$ is $L$-Lipschitz and $\gamma\in(0,1/L]$, projected
gradient descent is a descent method whose gradient mapping
$G_\gamma(w)=\gamma^{-1}\big(w-P_{\mathcal C}(w-\gamma\nabla f(w))\big)$
obeys $\min_{t<T}\|G_\gamma(w^t)\|^2\le 2\big(f(w^0)-\inf_{\mathcal C}f\big)/(\gamma T)$,
so accumulation points are \emph{stationary} for the constrained problem,
i.e.\ satisfy $-\nabla f(w)\in N_{\mathcal C}(w)$. Nothing here implies
convergence to a global constrained optimum, and we do not claim it. Two
further gaps between that theorem and our runs should be stated plainly:
the gradients are stochastic, and the update is projected AdamW (clamp after
every optimizer step), which is a projection in the Euclidean metric applied
to a step taken in a preconditioned metric---so it is not the proximal step
of the classical analysis, and none of the above rates apply to it verbatim.
Paths A+/B are therefore justified by measurement (\S\ref{sec:results}), not
by a convergence guarantee.

\paragraph{(v) The clip rate as a diagnostic, made quantitative.}
If no coordinate is clipped then $y\in\mathcal C$ and
$P_{\mathcal C}(y)=y$; in particular, if $y$ is a global unconstrained
minimizer that happens to lie in $\mathcal C$, it is also a global
constrained minimizer, and post-hoc projection is exactly free. More
generally, with $\nabla f$ $L$-Lipschitz and $e$ the clipped-away vector of
(iii),
\[
f\big(P_{\mathcal C}(y)\big)-f(y)\;\le\;-\langle\nabla f(y),e\rangle+\tfrac L2\|e\|^2
\;\le\;\|\nabla f(y)\|\,\|e\|+\tfrac L2\|e\|^2 ,
\]
so the loss cost of projection is controlled by the clipped-away mass, which
is measured at every merge (\texttt{clipped\_frac},
\texttt{max\_excess\_halfwidths}, and $\|e\|\le\|d\|\sqrt{1-\phi^2}$).
Since $L$ is not measured, this is a scaling argument that explains the
observed pattern---projection free at $0.14\%$ clip, costly at $3$--$5\%$,
catastrophic at $30.4\%$ (\S\ref{sec:results})---and not a numerical
certificate.

\subsection{Proposition~\ref{prop:codec} (nested refinement code)}
\label{app:proofs:codec}

Fix a weight with margined cell $[\ell,u]$ of width $\Delta=u-\ell>0$, and
let $x=\operatorname{clamp}\!\big((w-\ell)/\Delta,\,0,\,1\big)\in[0,1]$ be its
clamped in-cell position. The codec is
\[
r^{(k)}=\min\!\big(\lfloor 2^k x\rfloor,\;2^k-1\big)\in\{0,\dots,2^k-1\},
\qquad
\hat w^{(k)}=\ell+\big(r^{(k)}+\tfrac12\big)\Delta/2^k ,
\]
for $1\le k\le 8$ (\texttt{pack\_residual}/\texttt{unpack\_residual}).

\paragraph{(i) Nesting.}
We use the floor-of-floor identity: for real $y\ge0$ and integer $m\ge1$,
$\lfloor\lfloor y\rfloor/m\rfloor=\lfloor y/m\rfloor$. Proof: put
$n=\lfloor y\rfloor$ and $q=\lfloor n/m\rfloor$, so $qm\le n\le (q+1)m-1$.
Then $y\ge n\ge qm$ gives $\lfloor y/m\rfloor\ge q$, and $y<n+1\le(q+1)m$
gives $\lfloor y/m\rfloor\le q$. Applying it with $y=2^{k+j}x$ and $m=2^j$,
and using that $\gg j$ is exactly $\lfloor\cdot/2^j\rfloor$ on nonnegative
integers,
\[
\lfloor 2^{k+j}x\rfloor\gg j=\big\lfloor \lfloor 2^{k+j}x\rfloor/2^{j}\big\rfloor
=\lfloor 2^{k}x\rfloor .
\]
This settles $x\in[0,1)$. The boundary case $x=1$---which occurs exactly when
$w\ge u$, i.e.\ when the encoder saturates---must be checked separately
because there the floor overflows and the $\min$ clamp fires: then
$r^{(k+j)}=2^{k+j}-1$, whose binary expansion is all ones, so
$r^{(k+j)}\gg j=2^{k}-1=r^{(k)}$, the clamp having fired at depth $k$ as
well. Hence $r^{(k)}=r^{(k+j)}\gg j$ holds on all of $[0,1]$: the saturating
clamp is compatible with truncation precisely because it saturates to an
all-ones code. A $k$-bit stream is therefore a prefix of every deeper one,
and dropping trailing bits is a valid decode at the shallower depth---down to
$k=0$, which returns the anchor itself and the released artifact unchanged.

\paragraph{(ii) Error bound.}
For $x\in[0,1)$ the clamp is inactive and $r=\lfloor 2^kx\rfloor$ satisfies
$r\le 2^kx<r+1$, i.e.\ $x$ lies in the sub-cell
$[r2^{-k},(r{+}1)2^{-k})$ of length $2^{-k}$, whose center is the
reconstructed position $\hat x=(r+\tfrac12)2^{-k}$. Hence
$|x-\hat x|\le 2^{-(k+1)}$ and
\[
|w-\hat w^{(k)}|=\Delta\,|x-\hat x|\;\le\;\Delta/2^{k+1}
\;=\;(1-2\eta)\,W/2^{k+1}\;\le\;W/2^{k+1},
\]
which is the claimed bound, stated in the raw cell width $W$; the margin only
tightens it. The bound is attained (at $x=r2^{-k}$), so it cannot be
improved. If the encoder saturated ($w\notin[\ell,u]$), the same computation
bounds the distance to the \emph{clamped} position, so the total error is
$\Delta/2^{k+1}+\operatorname{dist}(w,[\ell,u])$; in the training paths this
second term is zero by construction, since every shipped weight is already
the output of a projection or of a bounded $\tanh$ fill.

\paragraph{(iii) Strict interiority, at every depth.}
Since $0\le r\le 2^k-1$, the reconstructed position satisfies
$\hat x\in[2^{-(k+1)},\,1-2^{-(k+1)}]\subset(0,1)$, so
\[
\ell+\Delta/2^{k+1}\;\le\;\hat w^{(k)}\;\le\;u-\Delta/2^{k+1} .
\]
Combining with Fact A.1, the reconstruction clears either RTN decision
boundary by at least $\eta W+(1-2\eta)W/2^{k+1}>\eta W$, a clearance that is
bounded below \emph{uniformly in $k$}. By Proposition~\ref{prop:inv},
re-quantization under the frozen scales returns the original code $a_i$, for
every truncation depth and for every weight. Since the 4-bit codes are never
written by the codec, the 4-bit prefix of the shipped file is bit-identical
to the release by construction rather than by verification.

\paragraph{Machine-verified content.}
\texttt{tests/test\_codec.py} checks, on randomly quantized blocks: nesting
($r^{(4)}=r^{(8)}\!\gg\!4$, $r^{(2)}=r^{(8)}\!\gg\!6$,
$r^{(2)}=r^{(4)}\!\gg\!2$); the error bound for $k\in\{1,2,4,8\}$; strict
interiority for $k\in\{1,4,8\}$; monotone decrease of mean error in $k$; and
integer-domain invariance under frozen scales after every round trip.

\subsection{Proposition~\ref{prop:capacity} (capacity bound)}
\label{app:proofs:capacity}

\paragraph{The channel.}
Let $F$ be the random fact corpus, drawn by the synthetic generator of
\S\ref{sec:setup} by sampling each entity's attributes independently and
uniformly from finite lists. Let $R$ be the publisher's independent
randomness (initialization, data order, rehearsal draw). A training procedure
produces a message $m=T(F,R)$, and the subscriber, holding the public
artifact $\hat W$, forms $W'=f(\hat W,m)$. Because $\hat W$ is fixed before
$F$ exists and $W'$ depends on $F$ only through $m$, the variables form a
Markov chain $F\to m\to W'$.

\paragraph{Message length.}
The message is (mask, $r^{(k)}$): a set $S\subseteq\{1,\dots,N\}$ of
$N_t=|S|$ written coordinates and $k$ refinement bits each. Two regimes must
be distinguished, and this is the first place a reviewer should push.
\emph{(a) Declared mask.} If $S$ is fixed and public---announced in the
release contract, as in every experiment here: paths A/A+/B write all
quantized linear weights, and the localization study of \S\ref{sec:where}
uses masks named in advance by module and depth---then the receiver already
knows $S$ and the index term is \emph{not} paid: $|m|=kN_t$.
\emph{(b) Data-dependent mask.} If $S$ is chosen as a function of the run
(hence of $F$), it must be transmitted; enumerative coding of an
$N_t$-subset costs $\lceil\log_2\binom{N}{N_t}\rceil$ bits, plus
$O(\log N)$ to convey $N_t$ if it is not fixed. Then
$|m|\le kN_t+\log_2\binom{N}{N_t}+O(\log N)$, which is the form stated in
Proposition~\ref{prop:capacity}. The distinction is not cosmetic: for
$N_t\ll N$, $\log_2\binom{N}{N_t}\approx N_t\log_2(eN/N_t)$ can exceed
$kN_t$. The proposition states the worst case; our runs live in case (a).

\paragraph{The bound.}
Assume $m$ is encoded in a prefix-free or fixed-length binary code of length
$|m|$. Then $H(m)\le|m|$, and the data-processing inequality along
$F\to m\to W'$ gives
\[
I(F;W')\;\le\;I(F;m)\;\le\;H(m)\;\le\;|m| .
\]
With $N=1.409\times10^9$ constrained weights and $k=4$, case (a) gives a
ceiling of $5.6\times10^9$ bits.

\paragraph{What is bounded, and what is not.}
This is the second place a reviewer should push, and the honest answer has
two parts.

\emph{The DPI bounds mutual information, not recall.} To convert the
proposition into a statement about the measured quantity we need the converse
direction, Fano's inequality. The cloze probe supplies the subject name and
asks for an attribute, so the decoder is $\hat A=g(W',\text{names})$ and the
relevant quantity is the conditional mutual information $I(A;W'\mid\text{names})$,
where $A$ collects the probed attributes; names and attributes are
independent by construction, and the same chain conditioned on names gives
$I(A;W'\mid\text{names})\le|m|$. If $A$ is uniform on a product of finite
attribute sets---true by construction, $H(A)=11.9$ bits per fact for the
three probed attributes---and the probe errs with probability $P_e$, Fano
gives $H(A\mid W',\text{names})\le h(P_e)+P_e H(A)$, with $h$ the binary
entropy and $h\le1$, and hence
\[
(1-P_e)\,H(A)-1\;\le\;I(A;W'\mid \text{names})\;\le\;|m| .
\]
The left-hand side is exactly the accounting used throughout the paper,
$\text{recall}\times N_{\text{facts}}\times 11.9$ bits: the ``absorbed bits''
column is \emph{not} a Fano lower bound, and an earlier version of this
appendix said it was. Applied per fact, Fano gives
$I\ge(1-P_e)H(A)-h(P_e)$ for each, so the bound on the whole corpus is
$\text{recall}\times N\times11.9-N\,h(P_e)$ bits, which sits below the column
by $13\%$ at the highest-recall arm and $28\%$ at the lowest --- a
recall-dependent gap, not a constant factor, so it does not cancel in the
efficiency ratios either. What the column is, is a recall-weighted count of
probed attribute bits: a consistent measure of absorption across arms, and the
one all our comparisons use. The information-theoretic statement that does
hold is the upper bound, $I(A;W'\mid\text{names})\le|m|$, which is
Proposition~\ref{prop:capacity} and is what bounds capacity.
Three caveats travel with it. Recall is measured by exact-match greedy
decoding, so information the model holds but cannot emit is not counted, and
the estimate is conservative in that direction. The probes are not
independent of one another and the $6.5\%$ vocabulary floor
(\S\ref{sec:setup}) must be subtracted before reading small recalls as
information. And Fano requires the decoder to use only $W'$ and the public
subject list, which is why the probe never conditions on the corpus.

\emph{Novelty is a hypothesis of the theorem, not a conclusion.} If $\hat W$
already depended on $F$, $I(F;W')$ could be large with $|m|=0$ and the
identity would say nothing about learning. The synthetic corpus is generated
after the release and the base model's recall is at the vocabulary floor
(\S\ref{sec:setup}), so the required independence holds by construction; on
real post-release facts it would have to be argued rather than assumed.

\emph{The ceiling is not an estimate.} Nothing above asserts that $kN_t$ bits
are \emph{achievable}; the bound is one-sided. Empirically the gap is five
orders of magnitude---$4{,}130$ facts absorbed at $10^4$ presented,
$\approx49$ kbit of probed attribute entropy, against a $5.6$ Gbit ceiling
(\S\ref{sec:results})---which is the quantitative form of the claim that the
regime is optimization-limited rather than capacity-limited. Finally, the
bound scales with what is \emph{shipped}: a full-precision residual would
give $|m|=32N_t$, eight times weaker at $k=4$. Proposition~\ref{prop:codec}
is what makes Proposition~\ref{prop:capacity} bite.

\subsection{Proposition~\ref{prop:budget} (forgetting budget)}
\label{app:proofs:budget}

\paragraph{Second-order expansion.}
Fix an input distribution $\mathcal D$ and write
$\varphi_x(\delta)=\mathrm{KL}\big(p_{\hat W}(\cdot\mid x)\,\|\,p_{\hat W+\delta}(\cdot\mid x)\big)$.
The model is a softmax head over a network smooth in its weights (SiLU
activations), so $p_W(y\mid x)>0$ and $W\mapsto\log p_W(y\mid x)$ is
$C^3$ on a neighborhood of the compact box $\mathcal C$. Then
$\varphi_x\ge0$ with $\varphi_x(0)=0$, so $\delta=0$ is a global minimum and
$\nabla\varphi_x(0)=0$; explicitly,
$\nabla_\delta\varphi_x(0)=-\mathbb E_{y\sim p_{\hat W}}\big[\nabla_W\log p_W(y\mid x)\big]
=-\nabla_W\sum_y p_W(y\mid x)=-\nabla_W 1=0$,
the zero-mean-score identity. Differentiating once more and using the same
identity,
$\nabla^2\varphi_x(0)=\mathbb E_{y\sim p_{\hat W}}\big[\nabla\log p\,\nabla\log p^{\!\top}\big]=F_x$,
the Fisher information at $\hat W$. Taylor's theorem with remainder,
uniformly on the compact set $\mathcal C$, gives
\[
\mathbb E_{x\sim\mathcal D}\,\mathrm{KL}\big(p_{\hat W}\|p_{\hat W+\delta}\big)
=\tfrac12\,\delta^{\!\top}F\delta+O(\|\delta\|^3),
\qquad F=\mathbb E_{x\sim\mathcal D}F_x ,
\]
which is Proposition~\ref{prop:budget}. The remainder constant is uniform
over the box but is \emph{not measured}; everything below turns on that.

\paragraph{Where $\rho^2$ comes from.}
Two different $\rho^2$ statements are in play and they should not be
conflated.

\emph{Homogeneity (exact, no approximation).} The constraint set at radius
$\rho$ is $\rho\mathcal B$ with $\mathcal B$ the fixed box of half-widths
$\mathrm{room}_i$, and $\delta\mapsto\tfrac12\delta^{\!\top}F\delta$ is
homogeneous of degree two, so for \emph{any} PSD $F$, diagonal or not,
$\sup_{\delta\in\rho\mathcal B}\tfrac12\delta^{\!\top}F\delta
=\rho^2\sup_{\delta\in\mathcal B}\tfrac12\delta^{\!\top}F\delta$.
The $\rho^2$ law is thus a property of the geometry, and survives even when
the diagonal surrogate does not.

\emph{Worst case over the box, diagonal surrogate.} Dropping off-diagonal
terms and using the loose radius $\Delta_i/2\ge\mathrm{room}_i$,
$\sup_{|\delta_i|\le\rho\Delta_i/2}\tfrac12\sum_i F_{ii}\delta_i^2
=\tfrac{\rho^2}{8}\sum_i F_{ii}\Delta_i^2=:B_F(\rho)$,
the expression quoted in \S\ref{sec:theory}.

\emph{Uniform in-cell fill (what is measured).} The geometry experiment sets
$\delta_i=u_i\,\rho\,\mathrm{room}_i$ with $u_i\sim\mathcal U(-1,1)$ i.i.d.,
which is the natural model of ``filling the cells with random content''.
Then $\mathbb E[\delta_i]=0$, $\mathbb E[\delta_i^2]=\rho^2\mathrm{room}_i^2/3$
and $\mathbb E[\delta_i\delta_j]=0$ for $i\ne j$, so
\[
\mathbb E\Big[\tfrac12\delta^{\!\top}F\delta\Big]
=\tfrac12\sum_{i,j}F_{ij}\,\mathbb E[\delta_i\delta_j]
=\frac{\rho^2}{6}\sum_i F_{ii}\,\mathrm{room}_i^2 ,
\]
which is the predictor implemented in \texttt{experiments/exp\_geom.py}.
(The same symmetry kills the first-order term of the \emph{evaluation} loss,
which does not vanish at $\hat W$: $\mathbb E\langle\nabla L,\delta\rangle=0$.
So the second-order prediction applies to the measured $\Delta$NLL as well as
to the KL.)

\paragraph{Why the diagonal surrogate fails empirically, and what that rules out.}
Measured against $\rho$-scaled uniform in-cell fills of Qwen3-1.7B
($1.409\times10^9$ perturbed coordinates, two perturbation seeds,
$\sum_i F_{ii}\mathrm{room}_i^2=2.03\times10^{-3}$), the predictor
underestimates the measured $\Delta$NLL by $16\times$ at $\rho=0.125$,
$196\times$ at $0.25$, $291\times$ at $0.5$, $340\times$ at $0.75$ and
$383\times$ at $\rho=1$; fitting the measured curve over
$\rho\in[0.25,0.75]$ against its value at $\rho=1$ gives exponent $2.43$,
not $2$ (\S\ref{sec:geometry}, \texttt{results/exp\_geom\_1p7b.json}). The
surrogate is therefore not a usable numerical bound, as
\S\ref{sec:theory} states.

The displayed computation, however, rules out one tempting explanation. Under
independent zero-mean coordinatewise perturbations the off-diagonal entries
of $F$ contribute nothing to the \emph{mean} of $\delta^{\!\top}F\delta$;
they enter only its variance. Whether a single realization sits near that
mean is empirical, and two seeds bound it only crudely: the per-seed
$\Delta$NLL, recovered from \texttt{results/exp\_geom\_1p7b.json} as
$\log\mathrm{ppl}_\rho-\log\mathrm{ppl}_{t=0}$, ranges over $39\%$, $16\%$,
$14\%$ and $15\%$ of its mean at $\rho=0.25$, $0.5$, $0.75$ and $1$, and at
$\rho=0.125$ the two seeds do not agree even in sign. (The raw perplexities
agree to within $2\%$ at every radius, but that is not the quantity at issue:
the signal is their small difference from the anchor.) A $14$--$39\%$
fluctuation is still two orders of magnitude short of the
$196$--$383\times$ gap, so high dimension alone does not make off-diagonal
Fisher mass inflate the second-order prediction.
Two candidate explanations remain, and we can separate neither with the
present data. (a) \emph{The truncation itself fails at this magnitude}: the
perturbation is full-dimensional, so $\|\delta\|$ is macroscopic even though
each coordinate moves less than a cell half-width, and the $O(\|\delta\|^3)$
remainder need not be small; the measured exponent $2.43>2$ is direct
evidence of this. (b) \emph{The estimator of $F$ is biased low}: the
implementation squares the gradient of a loss averaged over a $1024$-token
chunk, and for approximately independent zero-mean per-token gradients
$\mathbb E[(\tfrac1n\sum_t g_t)^2]\approx\tfrac1n\mathbb E[g^2]$, an
underestimate of the per-token Fisher diagonal by a factor of order $n=1024$
--- the right order of magnitude for the observed gap. A further, unquantified
gap is that the measured quantity is the change in evaluation NLL, whose
Hessian equals $F$ only at a stationary point, and $\hat W$ is not one:
perplexity falls monotonically along the whole segment from $\hat W$ to
$W_{\mathrm{orig}}$ (\S\ref{sec:geometry}).

\paragraph{Status.}
What Proposition~\ref{prop:budget} delivers is (i) an exact second-order
identity with the Fisher form, (ii) an exact $\rho^2$ homogeneity statement
for that form on the box, and (iii) the observation that the box radii are
supplied by the quantization grid at no cost. What it does not deliver is a
computable a-priori certificate: the diagonal surrogate is off by
$16$--$383\times$ and the measured exponent is $2.43$. We therefore treat
$\rho$ as a calibrated dial and report the measured curve
(Fig.~\ref{fig:geometry}b) rather than a bound.

\subsection{Proposition~\ref{prop:decay} (room contraction under folding)}
\label{app:proofs:decay}
Fix a weight with walls at $\ell<\hat w<u$, write $a=\hat w-\ell$, $b=u-\hat w$,
$W=a+b$, $d=a-b$ and $r=\min(a,b)$, and let the update be $w'=\hat w+rt$ with
$|t|<1$. Folding makes $w'$ the new anchor, so the new distances are
$a'=a+rt$ and $b'=b-rt$; both are positive because $|rt|<r\le\min(a,b)$, so
the folded anchor stays strictly inside the same cell and the code is
unchanged, which is Proposition~\ref{prop:inv} applied to the fold.

\paragraph{The identity.} For any reals $p,q$, $\min(p,q)=\tfrac{p+q}{2}-\tfrac{|p-q|}{2}$.
With $p=a+rt$ and $q=b-rt$ we have $p+q=a+b=W$, independent of $t$ because
the cell walls do not move, and $p-q=(a-b)+2rt=d+2rt$. Hence
\[
r'=\min(a',b')=\frac{W}{2}-\frac{|d+2rt|}{2},
\qquad r=\min(a,b)=\frac{W}{2}-\frac{|d|}{2}.
\]
Both are exact; no approximation has been made.

\paragraph{(i) The lower bound.} $|d+2rt|\le|d|+2r|t|$, so
$r'\ge\frac{W}{2}-\frac{|d|}{2}-r|t|=r-r|t|=r(1-|t|)$. Equality needs
$|d+2rt|=|d|+2r|t|$, that is $d$ and $t$ of the same sign or $d=0$: the
update moves toward the nearer wall, or the cell is symmetric.

\paragraph{(ii) A fold can increase the room.} Take $a=1$, $b=10$,
$t=+\tfrac12$. Then $r=1$, $d=-9$, $W=11$ and
$r'=\tfrac{11}{2}-\tfrac{|-9+1|}{2}=\tfrac{11}{2}-4=\tfrac32>r$. So no
statement of the form $r'\ge\varphi(r)$ can establish that the room decays;
that requires an upper bound, which is (iii).

\paragraph{(iii) Contraction in expectation.} Let $t$ have a distribution
symmetric about $0$ (the fill is initialized at $B=0$ with $A$ random, so the
sign of each coordinate's displacement carries no prior). The map
$x\mapsto|x|$ is convex, so by Jensen
$\mathbb E|d+2rt|\ \ge\ \big|\mathbb E[d+2rt]\big|=|d|$, and therefore
\[
\mathbb E[r']=\frac{W}{2}-\frac{\mathbb E|d+2rt|}{2}\ \le\ \frac{W}{2}-\frac{|d|}{2}=r .
\]
For the equality case take the two-point law $t=\pm\tau$. If $|d|\ge2r\tau$
then $d+2rt$ keeps the sign of $d$ for both branches, $\mathbb E|d+2rt|=|d|$
and $\mathbb E[r']=r$ exactly: a cell asymmetric enough that the update never
changes which wall is nearer loses no room in expectation. If $|d|<2r\tau$ the
two branches straddle zero and the inequality is strict; at $d=0$ it is
extremal, $\mathbb E[r']=\tfrac{W}{2}-r\tau=r(1-\tau)$, which is the edge case
of (i). Contraction is therefore fastest at centred anchors and vanishes at
the asymmetric limit, which is the ordering \S\ref{sec:sequential} measures
across folds.

\paragraph{What is not proven here.} Nothing above gives a rate. Passing from
$\mathbb E[r']\le r$ per weight to a geometric law for the population mean
$\bar r_k$ needs the joint distribution of $(d,r,|t|)$ across weights and how
it evolves under folding, which we do not characterize; the constant
$\beta\in[0.78,0.85]$ is measured, and the simulation of
\S\ref{sec:sequential} rather than this proposition is what attributes it to
the folding operator instead of to the level table.

\section{Hyperparameters}\label{app:hparams}
Every value that appears in any archived \texttt{config} block.
\begin{center}\small
\begin{tabular}{l p{0.68\linewidth}}
\toprule
hyperparameter & values used \\
\midrule
base model & Qwen3.8-27B, Qwen3-4B-Base, Qwen3-8B-Base, Qwen3-0.6B-Base, Qwen3-1.7B-Base, mistralai/Mistral-7B-v0.3, \dots \\
facts & 10, 24, 28, 108, 140, 160, \dots \\
epochs & 6, 8, 12, 24, 96 \\
LoRA rank & 16, 32, 64, 128, 256 \\
learning rate & 0.001, 0.003, 2e-05, 5e-05, 0.0001, 0.0002 \\
batch size & 4, 8, 16, 96 \\
fresh rehearsal fraction & 0.1, 0.3 \\
cell margin & 0.01, 0.02 \\
radius $\rho$ & 0.5, 1.0 \\
healing epochs & 0, 2, 4, 8 \\
healing lr & 1e-05, 2e-05 \\
CellFill gain $s$ & 4.0, 10.0, 40.0 \\
\bottomrule
\end{tabular}
\end{center}

\section{Run manifest}\label{app:manifest}
Archived filenames use \texttt{qil}, the working name under which CellFill was developed (\texttt{exp5\_qil.py}, \texttt{exp10\_qil\_r64\_*}, \texttt{exp16\_qil\_r16\_*}); the method and the files are otherwise identical.
Every number in this paper comes from one of the files below, which are
released with the code; tables and figures are generated from them by
\texttt{scripts/make\_tables.py} and \texttt{scripts/fig\_*.py} rather than
transcribed. The public repository releases exactly the runs this manifest lists;
nothing outside it is cited here. Wall-clock times are for a single RTX 4090 (48\,GB) or one
A100 (80\,GB) for the 8B and larger runs.

\begin{center}\scriptsize\resizebox{\linewidth}{!}{
\begin{tabular}{llp{0.34\linewidth}r}
\toprule
archived file & model & what it measures & min \\
\midrule
\texttt{bench\_anchor} & --- & ARC/HellaSwag/WinoGrande, 4-bit anchor & 10 \\
\texttt{bench\_merged} & --- & ARC/HellaSwag/WinoGrande, served model & 9 \\
\texttt{bench\_released} & --- & ARC/HellaSwag/WinoGrande, fp32 release & 12 \\
\texttt{curvature\_cells\_1p7b} & --- & do the cells sit where the loss is curved? Fisher against room, per weight & 6 \\
\texttt{exp0\_qwen3-1.7b\_r16\_e8} & Qwen3-1.7B-Base & clip rate at 8 epochs & 3 \\
\texttt{exp0b\_qwen3-1.7b\_r16\_e24} & Qwen3-1.7B-Base & no rehearsal, 24 epochs & 7 \\
\texttt{exp0c\_qwen3-1.7b\_r16\_e24\_origbase} & Qwen3-1.7B-Base & merge base = original fp32 & 7 \\
\texttt{exp10\_qil\_r64\_qwen3-1.7b} & Qwen3-1.7B-Base & CellFill r64 & 12 \\
\texttt{exp119\_loop\_plain} & Qwen3-1.7B-Base-bnb-4bit & the same loop with rehearsal off: what consolidation buys and what it does not & 41 \\
\texttt{exp119\_loop\_rehearse} & Qwen3-1.7B-Base-bnb-4bit & the full cycle as a loop: six tasks, rehearsal, consolidating every second task & 49 \\
\texttt{exp120\_diag\_minor} & Qwen3-1.7B-Base-bnb-4bit & six folds with the displacement's distribution kept, not just its mean: what binds a long sequence & 41 \\
\texttt{exp121\_s10\_loop\_rehearse} & Qwen3-1.7B-Base-bnb-4bit & the gentle recipe through the same cycle: where consolidation stops paying & 49 \\
\texttt{exp122\_validate\_s10} & Qwen3-1.7B-Base-bnb-4bit & the gentle recipe with the task suite scored inside the sequence and a recovery pass after each major & 67 \\
\texttt{exp122\_validate\_s40} & Qwen3-1.7B-Base-bnb-4bit & the same, at the standard recipe: where the trade between capability and knowledge lands & 68 \\
\texttt{exp12\_27b\_e8} & Qwen3.8-27B & 27B, 8 epochs & 46 \\
\texttt{exp12b\_27b\_e24} & Qwen3.8-27B & 27B, 24 epochs & 114 \\
\texttt{exp14\_attn} & Qwen3-1.7B-Base & partition: attention & 9 \\
\texttt{exp14\_down} & Qwen3-1.7B-Base & partition: down\_proj & 8 \\
\texttt{exp14\_early} & Qwen3-1.7B-Base & partition: layers 0-8 & 8 \\
\texttt{exp14\_gateup} & Qwen3-1.7B-Base & partition: gate+up & 9 \\
\texttt{exp14\_late} & Qwen3-1.7B-Base & partition: layers 19-27 & 7 \\
\texttt{exp14\_mid} & Qwen3-1.7B-Base & partition: layers 9-18 & 7 \\
\texttt{exp14\_mlp} & Qwen3-1.7B-Base & partition: all MLP & 10 \\
\texttt{exp15\_seq\_consolidate} & Qwen3-1.7B-Base & 4 tasks with consolidation & 21 \\
\texttt{exp15\_seq\_minor} & Qwen3-1.7B-Base & 4 sequential tasks & 21 \\
\texttt{exp16\_qil\_r16\_replay10} & Qwen3-1.7B-Base & CellFill r16, matched rehearsal & 12 \\
\texttt{exp17\_aplus10k\_xdom} & Qwen3-1.7B-Base & capacity 10k, A+ r64, cross-domain & 90 \\
\texttt{exp17\_b10k\_xdom} & Qwen3-1.7B-Base & capacity 10k, B, cross-domain & 187 \\
\texttt{exp18\_b\_replay10\_s0} & Qwen3-1.7B-Base & B dense at rehearsal 0.1, seed 0 & 22 \\
\texttt{exp18\_b\_replay10\_s1} & Qwen3-1.7B-Base & B dense at rehearsal 0.1, seed 1 & 16 \\
\texttt{exp1\_replay30\_qwen3-1.7b\_r16\_e24} & Qwen3-1.7B-Base & stale rehearsal buffer & 10 \\
\texttt{exp1b\_freshreplay30\_qwen3-1.7b\_r16\_e24} & Qwen3-1.7B-Base & fresh rehearsal 30\% & 10 \\
\texttt{exp1c\_replay10\_qwen3-1.7b\_r16\_e24} & Qwen3-1.7B-Base & A / A+ seed 0 & 11 \\
\texttt{exp1c\_s1} & Qwen3-1.7B-Base & A / A+ seed 1 & 14 \\
\texttt{exp1c\_s2} & Qwen3-1.7B-Base & A / A+ seed 2 & 13 \\
\texttt{exp20\_aplus\_0p6} & Qwen3-0.6B-Base & 0.6B A+ heal & 11 \\
\texttt{exp20\_qil\_0p6} & Qwen3-0.6B-Base & 0.6B CellFill r64 & 9 \\
\texttt{exp21\_qil\_4b} & Qwen3-4B-Base & 4B CellFill r64 & 29 \\
\texttt{exp23\_mistral\_a} & Mistral-7B-v0.3 & Mistral-7B clip-merge, diverged & 24 \\
\texttt{exp23\_mistral\_qil} & Mistral-7B-v0.3 & Mistral-7B CellFill r64 & 47 \\
\texttt{exp24\_qil\_8b} & Qwen3-8B-Base & 8B CellFill r64 & 48 \\
\texttt{exp25\_mistral\_a\_lr5e-5} & Mistral-7B-v0.3 & Mistral-7B clip-merge at lr 5e-5 & 24 \\
\texttt{exp26\_cellfill\_lr1e-3} & Qwen3-1.7B-Base & stability sweep: CellFill at lr 1e-3 & 10 \\
\texttt{exp26\_cellfill\_lr3e-3} & Qwen3-1.7B-Base & stability sweep: CellFill at lr 3e-3 & 10 \\
\texttt{exp26\_lora\_lr1e-3} & Qwen3-1.7B-Base & stability sweep: LoRA at lr 1e-3 & 10 \\
\texttt{exp26\_lora\_lr3e-3} & Qwen3-1.7B-Base & stability sweep: LoRA at lr 3e-3, diverged & 10 \\
\texttt{exp29\_seq\_etanh} & Qwen3-1.7B-Base & 4 tasks, E|tanh| logged at fold & 15 \\
\texttt{exp2\_heal4\_qwen3-1.7b\_r16\_e24} & Qwen3-1.7B-Base & A+ at rehearsal 30\% & 13 \\
\texttt{exp30\_qil\_r64\_s1} & Qwen3-1.7B-Base & CellFill r64 seed 1 & 10 \\
\texttt{exp30\_qil\_r64\_s2} & Qwen3-1.7B-Base & CellFill r64 seed 2 & 10 \\
\texttt{exp31\_qil\_r16\_s1} & Qwen3-1.7B-Base & CellFill r16 seed 1, rehearsal 0.1 & 10 \\
\texttt{exp31\_qil\_r16\_s2} & Qwen3-1.7B-Base & CellFill r16 seed 2, rehearsal 0.1 & 20 \\
\texttt{exp33\_real\_cellfill} & Qwen3-1.7B-Base & real corpus, CellFill r64, 1 phrasing & 5 \\
\texttt{exp34\_real\_aplus} & Qwen3-1.7B-Base & real corpus, clip-merge then heal & 5 \\
\texttt{exp35\_real\_free} & Qwen3-1.7B-Base & real corpus, unconstrained control & 8 \\
\texttt{exp36\_realaug\_cellfill} & Qwen3-1.7B-Base & real corpus, medical domain augmented & 11 \\
\texttt{exp37\_realaug\_all} & Qwen3-1.7B-Base & real corpus, all domains augmented & 9 \\
\texttt{exp38\_real8dom} & Qwen3-1.7B-Base & 8 domains, truncated generation budget & 10 \\
\texttt{exp3\_dense\_qwen3-1.7b\_e24} & Qwen3-1.7B-Base & B dense seed 0 & 20 \\
\texttt{exp3\_s1} & Qwen3-1.7B-Base & B dense seed 1 & 21 \\
\texttt{exp3\_s2} & Qwen3-1.7B-Base & B dense seed 2 & 17 \\
\texttt{exp40\_real8\_fixed} & Qwen3-1.7B-Base & 8 domains, budget from the data & 12 \\
\texttt{exp41\_real8\_final} & Qwen3-1.7B-Base & 8 domains, per-domain floors archived & 11 \\
\texttt{exp4\_free\_qwen3-1.7b\_e24} & Qwen3-1.7B-Base & unconstrained seed 0 & 20 \\
\texttt{exp4\_free\_s1} & Qwen3-1.7B-Base & unconstrained seed 1 & 20 \\
\texttt{exp4\_free\_s2} & Qwen3-1.7B-Base & unconstrained seed 2 & 20 \\
\texttt{exp5\_qil\_s4\_qwen3-1.7b\_r16\_e24} & Qwen3-1.7B-Base & CellFill, under-stepped gain & 15 \\
\texttt{exp5b\_qil\_s1} & Qwen3-1.7B-Base & CellFill r16 seed 1 & 18 \\
\texttt{exp5b\_qil\_s2} & Qwen3-1.7B-Base & CellFill r16 seed 2 & 14 \\
\texttt{exp5b\_qil\_s40\_qwen3-1.7b\_r16\_e24} & Qwen3-1.7B-Base & CellFill r16 at rehearsal 30\% & 15 \\
\texttt{exp6\_10k\_qwen3-1.7b\_r16\_e24} & Qwen3-1.7B-Base & capacity 10k, A and A+ & 122 \\
\texttt{exp6b\_dense10k\_qwen3-1.7b\_e24} & Qwen3-1.7B-Base & capacity 10k, B & 196 \\
\texttt{exp70\_seq\_anchor} & Qwen3-1.7B-Base & six tasks folded, the control the fixed cell is read against & 87 \\
\texttt{exp70\_seq\_preimage} & Qwen3-1.7B-Base & six tasks in a fixed cell: the room never decays & 97 \\
\texttt{exp7\_rho05\_qwen3-1.7b\_r16\_e24} & Qwen3-1.7B-Base & radius rho = 0.5 & 14 \\
\texttt{exp8\_e96\_qwen3-1.7b} & Qwen3-1.7B-Base & exposure 96 epochs & 30 \\
\texttt{exp\_codec} & Qwen3-1.7B-Base & nested checkpoints, k = 1..4 & 14 \\
\texttt{exp\_geom\_1p7b} & Qwen3-1.7B-Base & safe-radius geometry, 1.7B & 7 \\
\texttt{exp\_geom\_4b} & Qwen3-4B-Base & safe-radius geometry, 4B & 15 \\
\bottomrule
\end{tabular}
}\end{center}

\section{Storage precision of the residual}
The refinement code of Proposition~\ref{prop:codec} places reconstructions
at sub-cell centers, strictly interior to the cell. Checkpoint dtype
nonetheless matters: at LLM weight scales, bf16's 8-bit mantissa can round a
stored value across a decision boundary. With the 1\% cell margin used at 1.7B (2\% at 27B), fp32 and fp16
storage are safe and bf16 is not; a 5\% margin
makes bf16 safe at the cost of 8\% of the writable range. This is verified
by unit test over randomly perturbed blocks.

\bibliographystyle{plain}
\sloppy

\end{document}